\documentclass[10pt]{article} % For LaTeX2e
\usepackage[preprint]{tmlr}
\usepackage{amsmath,amsfonts,bm}

\def\eqref#1{equation~\ref{#1}}
\def\1{\bm{1}}

\DeclareMathAlphabet{\mathsfit}{\encodingdefault}{\sfdefault}{m}{sl}
\SetMathAlphabet{\mathsfit}{bold}{\encodingdefault}{\sfdefault}{bx}{n}

\usepackage{hyperref}
\usepackage{url}

\title{Provably Convergent Policy Optimization via Metric-aware Trust Region Methods}

\author{\name Jun Song \email       
      juns113@uw.edu \\
      \addr Department of Industrial and Systems Engineering\\
      University of Washington
      \AND
      \name Niao He \email niao.he@inf.ethz.ch \\
      \addr Department of Computer Science \\ ETH Zürich
      \AND
      \name Lijun Ding \email lding47@wisc.edu\\
      \addr Wisconsin Institute for Discovery \\
      University of Wisconsin - Madison 
      \AND 
      \name Chaoyue Zhao \email 
      cyzhao@uw.edu \\
      \addr 
      Department of Industrial and Systems Engineering\\
      University of Washington
      }

\def\month{07}  % Insert correct month for camera-ready version
\def\year{2023} % Insert correct year for camera-ready version
\def\openreview{\url{https://openreview.net/forum?id=jkTqJJOGMS&noteId=2Kp2cQcdbR}} % Insert correct link to OpenReview for camera-ready version

\usepackage{bibentry}
\usepackage[utf8]{inputenc} % allow utf-8 input
\usepackage[T1]{fontenc}    % use 8-bit T1 fonts
\usepackage{hyperref}       % hyperlinks
\usepackage{url}            % simple URL typesetting
\usepackage{booktabs}       % professional-quality tables
\usepackage{amsfonts}       % blackboard math symbols
\usepackage{nicefrac}       % compact symbols for 1/2, etc.
\usepackage{microtype}      % microtypography
\usepackage{xcolor}         % colors

\usepackage{amsmath}

\usepackage{mathtools}
\usepackage{amssymb}
\usepackage[shortlabels]{enumitem}
\usepackage{graphicx}
\usepackage{wrapfig}
\usepackage{color}
\usepackage{xcolor,soul}
\usepackage{subcaption}
\usepackage{schemata}
\usepackage[ruled,vlined,algo2e]{algorithm2e}
\makeatletter
\newcommand{\algorithmstyle}[1]{\renewcommand{\algocf@style}{#1}}
\newcommand{\Lim}[1]{\raisebox{0.5ex}{\scalebox{0.95}{$\displaystyle \lim_{#1}\;$}}}
\newcommand{\Sup}[1]{\raisebox{0.5ex}{\scalebox{0.95}{$\displaystyle \sup_{#1}\;$}}}
\makeatother
\usepackage{adjustbox}
\usepackage{float}
\usepackage{amsthm}
\usepackage{chngcntr}
\usepackage{apptools}
\usepackage{thmtools}
\usepackage{thm-restate}
\usepackage{cleveref}
\usepackage{wrapfig}
\usepackage{enumitem}
\usepackage{comment}
\usepackage{tablefootnote}
\usepackage{longtable}

\allowdisplaybreaks

\newtheorem{remark}{Remark}

\newcommand{\hlight}{\textcolor{black}}
\newcommand{\bblack}{\textcolor{black}}

\newcommand{\zcyan}{\textcolor{black}}
\newcommand{\zzred}{\textcolor{black}}
\newcommand{\xred}{\textcolor{black}}

\newcommand{\bblue}{\textcolor{black}}
\newcommand{\ccred}{\textcolor{black}}
\newcommand{\rred}{\textcolor{black}}
\newcommand{\zblue}{\textcolor{black}}

\begin{document}

\maketitle

\begin{abstract}

Trust-region methods based on Kullback-Leibler divergence are pervasively used to stabilize policy optimization in reinforcement learning. In this paper, we exploit more flexible metrics and examine two natural extensions of policy optimization with Wasserstein and Sinkhorn trust regions, namely \emph{Wasserstein policy optimization (WPO)} and \emph{Sinkhorn policy optimization (SPO)}. Instead of restricting the policy to a parametric distribution class, we directly optimize the policy distribution and derive their closed-form policy updates based on the Lagrangian duality. Theoretically, we show that WPO guarantees a monotonic performance improvement, and SPO provably converges to WPO as the entropic regularizer diminishes. \zblue{Moreover, we prove that with a decaying Lagrangian multiplier to the trust region constraint, both methods converge to global optimality.} Experiments across tabular domains, robotic locomotion, \bblue{and continuous control tasks} further demonstrate the performance improvement of both approaches, more robustness of WPO to sample insufficiency, and faster convergence of SPO, over state-of-art policy gradient methods. %\NH{made some change, still needs polish.} 

\end{abstract}

\section{Introduction}
%%
% Model-free Reinforcement learning (RL) is a prominent machine learning paradigm that solves for the sequential actions of RL agents performing in an uncertain interactive environment and learning from feedback to optimize a specified objective.
%%
Policy-based reinforcement learning (RL) approaches have received remarkable success in many domains, including video games \citep{wang2017acer, wu2017acktr, mnih2016_a3c}, robotics \citep{grudic2003_pg_robotics, levine2016visuomotorpolicyforrobot}, and continuous control tasks \citep{duan2016_continuous, schulmanl2016_gae_continuous, heess2015continuouscontrolpolicy}.  One prominent example is policy gradient method \citep{grudic2003_pg_robotics, peters2006_pg_robotics, lillicrap2015_ddpg, sutton1999_pg, williams1992_reinforce, mnih2016_a3c, silver2014_dpg}. 
The core idea is to represent the policy with a probability distribution $\pi_\theta (a|s) = P [a|s; \theta]$, such that the action $a$ in state $s$ is chosen stochastically following the policy $\pi_\theta$ controlled by parameter $\theta$. 
%%
% However, policy optimization often suffers from a large variance problem that degrades the performance, especially in complex domains with high dimensional {state and action} spaces. 
%%
Determining the right step size to update the policy is crucial for maintaining the stability of policy gradient methods: too conservative choice of stepsizes result in slow convergence, while too large stepsizes may lead to catastrophically bad updates. 

%
% It is especially true for complex domain applications where state and action spaces are in high dimensions. 
%

To control the size of policy updates, Kullback-Leibler (KL) divergence is commonly adopted to measure the difference between two policies. For example, the seminal work on trust region policy optimization (TRPO) by \citet{schulman2015_trpo} introduced KL divergence based constraints (trust region constraints) to restrict the size of the policy update; see also \citet{peng2019advantage, abdolmaleki2018maximum}.  \citet{kakade2001natural} and \citet{schulman2017_ppo} introduced a KL-based penalty term to the objective to prevent excessive policy shift.

Though KL-based policy optimization has achieved promising results, it remains interesting whether using other metrics to gauge the similarity between policies could bring additional benefits.  Recently, a few work~\citep{richemond2017_wrl,zhang2018_wgf,moskovitz2021wnpg,pacchiano2019_bgpg} has explored the Wasserstein metric to restrict the deviation between consecutive policies.  Compared with KL divergence, the Wasserstein metric has several desirable  properties. 
Firstly, it is a true symmetric distance measure. Secondly, it allows flexible user-defined costs between actions and is less sensitive to ill-posed likelihood ratios. Thirdly but most importantly, the Wasserstein metric takes into account the geometry of the metric space \citep{panaretos2019_wassersteingeometry} and allows distributions to have different or even non-overlapping supports. 

\textcolor{black}{
{\bf Motivating Example}: Below we provide an example of a grid world (see Figure \ref{fig:grid_world_example}) that illustrates the advantages of using the Wasserstein metric over the KL divergence to construct trust regions and policy updates. The grid world consists of $5$ regular grids and $2$ goal grids, and there are three possible actions: left, right, and pickup. The player always starts from the middle grid, and making a left or right move results in a reward of $-1$. Picking up yields a reward of $-3$ at regular grids, $+5$ at the blue goal grid, and $+10$ at the red goal grid. An episode terminates either at the maximum length of $10$ or immediately after picking up. We define the geometric distance between left and right actions to be $1$, and $4$ between other actions.}
\begin{figure}[H]
    \centering
    \includegraphics[width = 0.4\linewidth]{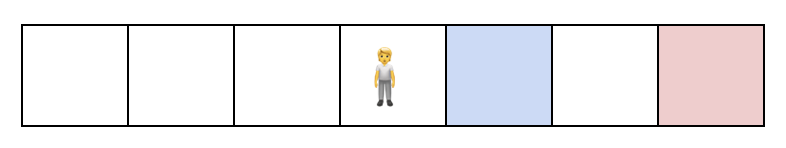}
    \caption{Motivating grid world example}
    \label{fig:grid_world_example}
\end{figure}

\textcolor{black}{
Figure \ref{fig:wd_geometric_advantage} shows the Wasserstein distance and KL divergence for different policy shifts of this grid world example. We can see that Wasserstein metric utilizes the geometric distance between actions to distinguish the shift of policy distribution to a close action (policy distribution 1 $\xrightarrow{}$ 2 in Figure \ref{fig:wd_geometric_advantage_close}) from the shift to a far action (policy distribution 1 $\xrightarrow{}$ 3 in Figure \ref{fig:wd_geometric_advantage_far}), while KL divergence does not. Figure \ref{fig:policy_update_example_kl_wass} demonstrates the constrained policy updates based on Wasserstein distance  and KL divergence respectively with a fixed trust region size $1$. We can see that Wasserstein-based policy update finds the optimal policy faster than KL-based policy update. This is because KL distance is larger than Waserstein when considering policy shifts of close actions (see Figure \ref{fig:wd_geometric_advantage_close}). Therefore, Wasserstein policy update is able to shift action (from left to right) in multiple states, while KL update is only allowed to shift action in a single state. Besides, KL policy update keeps using a suboptimal short-sighted solution between the 2nd and 4th iteration, which further slows down the convergence. }
\begin{figure}[H]
\centering
  \begin{subfigure}{0.48\columnwidth}
    \centering
    \includegraphics[width=0.95\linewidth]{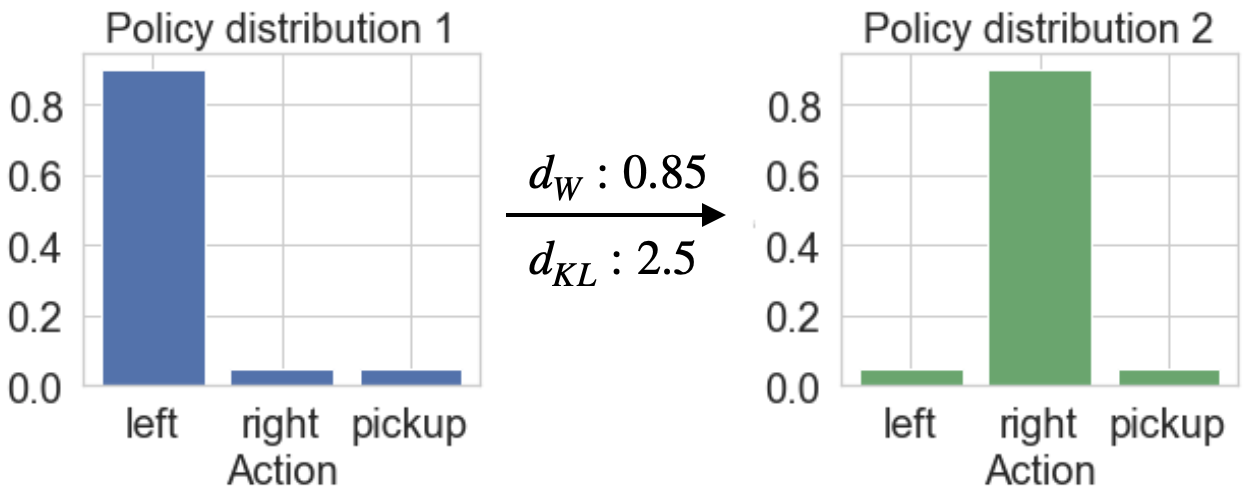}
    \caption{Policy shift of close action}
    \label{fig:wd_geometric_advantage_close}
  \end{subfigure}
   \begin{subfigure}{0.48\columnwidth}
    \centering
    \includegraphics[width=0.95\linewidth]{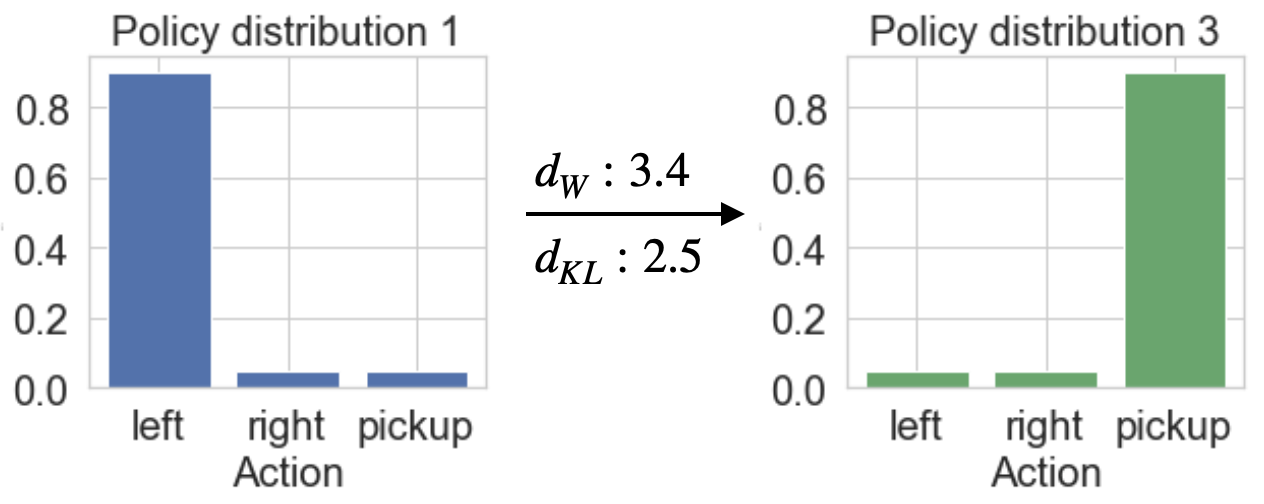}
    \caption{Policy shift of far action}
    \label{fig:wd_geometric_advantage_far}
  \end{subfigure}
\caption{Wasserstein utilizes geometric feature of action space}
\label{fig:wd_geometric_advantage}
\end{figure}

\begin{figure}[H]
\centering
\begin{minipage}{0.4\textwidth}
    \centering
    \includegraphics[width=0.8\linewidth]{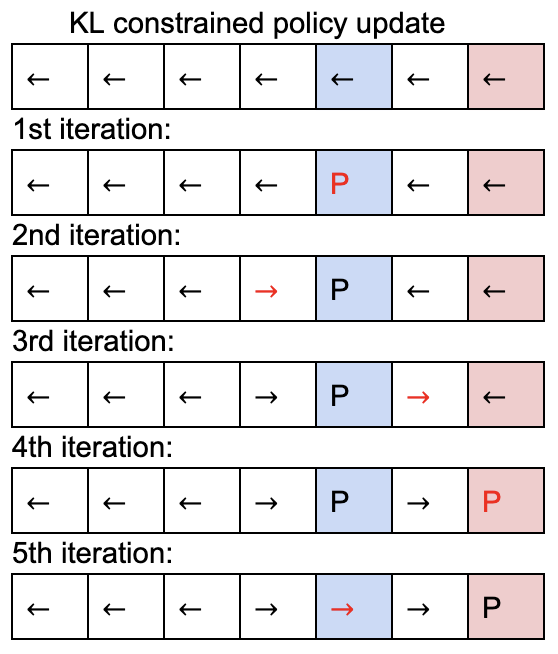}
\end{minipage}
\begin{minipage}{0.4\textwidth}
    \centering
    \includegraphics[width=0.8\linewidth]{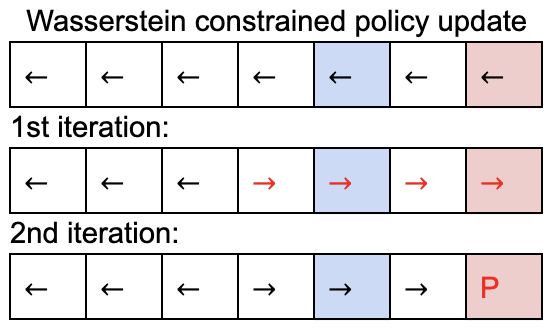}
\end{minipage}
\caption{Demonstration of policy updates under different trust regions \label{fig:policy_update_example_kl_wass}}
\end{figure}

%%
%%
% These desirable properties motivate us to apply Wasserstein metric to policy optimization. %
% effectively considers the geometric structure on the space of probability distributions, which can remarkably help exploitation. \hlight{Reference: One of the advantages of choosing a Wasserstein geometry is that non-surjective trajectory embedding maps are allowed. This is not possible with a KL induced
% one (in non-surjective cases, computing the likelihood ratios in the KL definition is in general
% intractable)} Thirdly but most importantly, Wasserstein metric has a weaker topology, thus possibly leading to a more robust policy \citep{arjovsky2017wgan}. These desirable properties motivate us to apply Wasserstein metric to policy optimization.  Apart from KL divergence, Wasserstein metric has also been used to restrict the deviation between the old and new policies to ensure learning stability.
%%
%%

However, the challenge of applying the Wasserstein metric for policy optimization is also evident: evaluating the Wasserstein distance requires solving an optimal transport problem, which could be computationally expensive. 
%
% Thus, the use of Wasserstein metric may incur slow policy convergence, instability and statistical inefficiency in large scale applications. That is also the main reason that, despite all its excellent properties, Wasserstein metric has not been widely applied in reinforcement learning domain.  
%
To avoid this computation hurdle, existing work resorts to different techniques to \emph{approximate the policy update} under Wasserstein regularization. For example, \citet{richemond2017_wrl} solved the resulting RL problem using Fokker-Planck equations; \citet{zhang2018_wgf} introduced particle approximation method to estimate the Wasserstein gradient flow. Recently, \citet{moskovitz2021wnpg} instead considered the second-order Taylor expansion of Wasserstein distance based on Wasserstein information matrix to characterize the local behavioral structure of policies. \citet{pacchiano2019_bgpg} tackled behavior-guided policy optimization with smooth Wasserstein regularization by solving an approximate dual reformulation defined on reproducing kernel Hilbert spaces. Aside from such approximation, some of these work also limits the policy representation to a particular parametric distribution class, As indicated in \citet{tessler2019_dpo}, since parametric distributions are not convex in the distribution space, optimizing over such distributions results in local movements in the action space and thus leads to convergence to a sub-optimal solution.  Until now, the theoretical performance of policy optimization under the Wasserstein metric remains elusive in light of these approximation errors.

% \NH{changed the motivation a bit, please check if this makes more sense. We need to give enough credit to these relevant work and emphasize the differences, otherwise, they might criticize us and misjudge the novelty.}
% Aside from the stability issue of policy based RL, another reason for its limited utilization in complex domain applications is the suboptimality. Policy based RL often limits the policy representation to a particular parametric distribution class, such as Gaussian \citep{schulman2015_trpo, schulman2017_ppo}, Beta \citep{chou2017_beta} and Delta \citep{lillicrap2015_ddpg, silver2014_dpg} distribution functions. As indicated in \cite{tessler2019_dpo}, since parametric distributions are not convex in the distribution space, optimizing over such distributions results in local movements in the action space and thus leads to convergence to a sub-optimal solution.  Also in practice, it is very difficult, if not impossible at all, for machine learners to correctly predetermine the underlying distribution of the optimal policy. 

%% Todo: need to update reference

In this paper, we study policy optimization with trust regions based on Wasserstein distance and Sinkhorn divergence. The latter is a smooth variant of Waserstein distance by imposing an entropic regularization to the optimal transport problem~\citep{cuturi2013_sinkhorn}. We call them, \emph{Wasserstein Policy Optimization (WPO)} and \emph{Sinkhorn Policy Optimization (SPO)}, respectively. Instead of confining the distribution of policy to a particular distribution class, we work on the space of policy distribution directly, and consider all admissible policies that are within the trust regions with the goal of avoiding approximation errors. Unlike existing work, we focus on \emph{exact characterization} of the policy updates. We would like to emphasize that our methodology and theoretical analysis in Section \ref{wasserstein}, \ref{sinkhorn}, and \ref{analysis} primarily concentrate on a discrete action space. However, we also present an extension of our method to accommodate a continuous action space, detailed in Section \ref{continuous_control}. We highlight our contributions as follows:

\begin{enumerate}[leftmargin=0.5cm,itemsep = 0.1pt]
    \item \textbf{Algorithms:} We develop closed-form expressions of the policy updates for both WPO and SPO based on the corresponding optimal Lagrangian multipliers of the trust region constraints. \bblue{To the best of our knowledge, this is the first explicit closed-form updates for policy optimization based on Wasserstein and Sinkhorn trust regions}. In particular, the optimal Lagrangian multiplier of SPO admits a simple form and can be computed efficiently.  A practical on-policy actor-critic algorithm is proposed based on the derived expressions of policy updates and advantage value function estimation.
    %%
     % compute the Lagrangian multlilplier. The reformulations of both policy optimization problems are derived and the closed-form optimal policies are obtained. 
    %%
    
    \item \textbf{Theory:} We theoretically show that WPO guarantees a \emph{monotonic performance improvement} through the iterations, \emph{even with non-optimal Lagrangian multipliers}. We also prove that SPO converges to WPO as the entropic regularizer diminishes. \zblue{Moreover, we prove that with a decaying schedule of the multiplier, SPO and WPO converge to \emph{global optimality}, and with a constant multiplier, both methods converge \emph{linearly} up to 
   a neighborhood of the optimal value.} {\color{black} To our best knowledge, this appears to be the first convergence rate analysis of policy optimization based on Wasserstein-type metrics.}
   %%
    % We also show that the resulting RL algorithm using Wasserstein metric is more robust to value function inaccuracies and yields to a better final performance with limited samples compared to that using KL divergence, both theoretically and numerically.
   %%

    \item \textbf{Experiments:} We provide comprehensive evaluation on the efficiency of WPO and SPO under several types of testing environments including tabular domains, robotic locomotion tasks, {\color{black} and further extend it to continuous control tasks}.  Compared to state-of-art policy gradients approaches that use KL divergence such as TRPO and PPO and those use Wasserstein metric such as Wasserstein Natural Policy Gradient (WNPG) \citep{moskovitz2021wnpg} and Behavior Guided Policy Gradients (BGPG) \citep{pacchiano2019_bgpg}, our methods achieve better sample efficiency, faster convergence, and improved final performance.  Numerical study indicates that by properly choosing the weight of entropic regularizer, {SPO} achieves a better trade-off between convergence and final performance than {WPO}. 
     
\end{enumerate}

%%
% Instead of solving the original optimal transport problem, the Sinkhorn divergence, a recently introduced variant of the Wasserstein metric in \cite{cuturi2013_sinkhorn}, regularizes the original optimal transport problem using the entropy of the probabilistic coupling and thus successfully reduce the computational burden for metric evaluation. The efficiency improvement that Sinkhorn divergence and the related algorithms brought paves the way to utilize Wasserstein-like metrics in machine learning domains, including online learning \citep{}, model selection \citep{}, generative modeling \citep{}, dimensionality reduction \citep{}, etc.
% In addition to Wasserstein metric, we also apply Sinkhorn divergence to our policy optimization model to achieve computational efficiency.
%%

\emph{Related work:} 
Wasserstein-like metrics have been explored in a number of works in the context reinforcement learning. \citet{ferns2004metrics} first introduced bisimulation metrics based on Wasserstein distance to quantify behavioral similarity between states for the purpose of state aggregation. Such bisimulation metrics were recently utilized for representation learning of RL; see e.g.,~\citet{castro2020scalable,agarwal2020contrastive}. In addition, a few recent work has also exploited Wasserstein distance for imitation learning (see e.g.,~\citet{xiao2019wasserstein,dadashi2021primal}) and unsupervised RL (see e.g., ~\citet{he2022_wurl}). \hlight{Our work is closely related to several previous studies, including }\citet{richemond2017_wrl,zhang2018_wgf,moskovitz2021wnpg,pacchiano2019_bgpg}\hlight{, which also utilize Wasserstein distance to measure the proximity of policies. However, unlike the aforementioned studies that solely employ Wasserstein distance as an explicit penalty function, we additionally utilize it as a trust region constraint. Moreover, we consider nonparametric policies and derive explicit policy update forms, whereas these studies update parametric policies using policy gradients. Furthermore, we demonstrate monotonic performance improvement and global convergence with our policy update, which is not provided in these previous works. Regarding the use of Sinkhorn divergence in RL, }\citet{pacchiano2019_bgpg}\hlight{ is the only related work to our best knowledge, where the entropy regularization is used to mitigate the computational burden of computing Wasserstein metric. However, no explicit form of policy update is provided in this work, while we derive an explicit Sinkhorn policy update and demonstrate its advantage in convergence speed. Additionally, we use Wasserstein distance to directly measure the proximity of nonparametric policies in the distribution space, while }\citet{pacchiano2019_bgpg, moskovitz2021wnpg}\hlight{ measure the similarity of parametric policies in the behavioral space.} 

Wasserstein-like  metrics are also pervasively studied in distributionally robust optimization (DRO); see e.g., ~\citet{esfahani2018data,gao2016distributionally,zhao2018_dro_wasserstein,blanchet_wasserstein_dro}. We also point out that a recent concurrent work by~\citet{wang2021sinkhorn}   studied DRO using the Sinkhorn distance. 
Our duality formulations are largely inspired from existing work in DRO. However, we note that constrained policy optimization is conceptually different from DRO.
% Despite the similarity shared in the duality formulations, the DRO problems are fundamentally different from constrained policy optimization. 
Constrained policy optimization focuses on finding the optimistic policy that falls in a trust region, whereas DRO (e.g., the KL DRO) aims to optimize some worst-case loss given by the adversarial distribution of unknown parameters within some ambiguity set.

%%
% \hlight{Reference: While the OT problem (to calculate Wasserstein metric) is originally cast as a linear program, with a $O(n^3log(n))$ cost, many of these works rely on solving instead a penalized OT problem using Sinkhorn’s algorithm [35, 13]. In its most naive implementation, the Sinkhorn has quadratic complexity [2]. Recent works achieve O(n) complexity by targeting the matrix-vector updates in Sinkhorn’s algorithm using low-rank approximations of the data kernel matrix [4, 3, 32]. This idea can be further improved by imposing the low-rank constraint on the optimization variables of the original OT problem [19], to modify Sinkorn’s steps by enforcing a low rank factorization of the coupling variable [33].}
%%
%%
% To our knowledge, the study of Sinkhorn divergence for policy optimization is relatively new in the literature. The most related work is \citep{pacchiano2019_bgpg}: Behavior Guided Policy Gradient is proposed with Wasserstein distance between old and new policies penalized to prevent large policy updates in the behavior space. 
% To mitigate the computational burden of computing Wasserstein distance, the entropy-regularized Wasserstein distance is utilized. 
%%
%%
% \section{ODRPO: A Framework to Learn Distributional Policies}
% \label{section_odrpo_framework}
%\vspace{-0.2cm}
%%
\section{Background and Notations}
%%
% In this section, we first develop the framework of the Optimistic Distributionally Robust Policy Optimization (\texttt{ODRPO}) under a general probability metric. Then we discuss the solution methodology for \texttt{ODRPO} with Wasserstein metric based trust region. 
%%

\vspace{-0.1cm}

\noindent {\bf Markov Decision Process (MDP):} We consider an infinite-horizon discounted MDP, defined by the tuple $(\mathcal{S},\mathcal{A},P,r,\rred{\upsilon},\gamma)$, where $\mathcal{S}$ is the state space, $\mathcal{A}$ is the action space, $P: \mathcal{S} \times \mathcal{A} \times \mathcal{S} \xrightarrow{} \mathbb{R}$ is the transition probability, $r: \mathcal{S} \times \mathcal{A} \xrightarrow{} \mathbb{R}$ is the reward function, $\rred{\upsilon}: \mathcal{S} \xrightarrow{} \mathbb{R}$ is the distribution of the initial state $s_0$, and $\gamma\in (0,1)$ is the discount factor. We define the return of timestep $t$ as the accumulated discounted reward from $t$, $R_t = \sum_{k=0}^{\infty} \gamma^k r(s_{t+k},a_{t+k})$, \rred{and the value function as $V^\pi(s) = \mathbb{E}[R_t|s_t = s; \pi]$}. The performance of a stochastic policy $\pi$ is defined as ${J}(\pi) = \mathbb{E}_{s_0,a_0,s_1 \dots} [\sum_{t=0}^{\infty} \gamma^t r(s_t,a_t)]$ where $a_t \sim \pi(a_t|s_t)$, $s_{t+1} \sim P(s_{t+1}|s_t,a_t)$. As shown in \citet{kakade2002_approximatelyoptimal}, the expected return of a new policy $\pi'$ can be expressed in terms of the advantage over the old policy $\pi$: ${J}(\pi') = {J}(\pi) + \mathbb{E}_{s\sim \rho^{\pi'}_\upsilon, a \sim \pi'} [A^{\pi}(s,a)]$, where $A^{\pi}(s,a) = \mathbb{E}[R_t|s_t = s, a_t = a; \pi] - \mathbb{E}[R_t|s_t = s; \pi]$ represents the advantage function and $\rho^{\pi}_\upsilon$ represents the unnormalized discounted visitation frequencies with initial state distribution $\upsilon$, i.e., $\rho^{\pi}_\upsilon(s) = \mathbb{E}_{s_0 \sim \upsilon} [\sum_{t=0}^{\infty}\gamma^t P(s_t = s| s_0)]$.

\vspace{-0.1cm}

\paragraph{Trust Region Policy Optimization (TRPO):} In TRPO~{\citep{schulman2015_trpo}}, the policy $\pi$ is parameterized as $\pi_{\theta}$ with parameter vector $\theta$. For notation brevity, we use $\theta$ to represent the policy $\pi_{\theta}$. Then, the new policy $\theta'$ is found in each iteration to maximize the expected improvement ${J}(\pi') - {J}(\pi)$, or equivalently, the expected value of the advantage function:
\begin{equation}
\begin{split}
& \max_{\theta'}\ \  \mathbb{E}_{s\sim \rho^{\theta}_\upsilon, a \sim \theta'} [ A^{\theta} (s,a)]  \\
& \text{s.t.} \ \ \mathbb{E}_{s\sim \rho^{\theta}_\upsilon} [d_{\text{KL}} (\theta', \theta)] \le \delta, 
\end{split}
\label{trpo_problem}
\end{equation}
where $d_{\text{KL}}$ represents the KL divergence and $\delta$ is the threshold of the distance between new and old policies. 
% Note that here the expected value of the advantage function is an estimation as the visitation frequency $\rho^{\theta}_\upsilon$ is used rather than
% $\rho^{\theta'}_\upsilon$, which means the changes in state visitation frequencies caused by the changes in policy are ignored. 
%\NH{Moved this part here.}

\vspace{-0.1cm}

\paragraph{Wasserstein Distance: }  Given two probability distributions of policies $\pi$ and $\pi'$ on the discrete action space $\mathcal{A} = \{a_1, a_2, \dots, a_N\}$, the Wasserstein distance between the policies is defined as: 
\begin{equation}
d_\text{W}(\pi',\pi) = \inf_{Q \in \Pi(\pi', \pi)} \langle Q, D\rangle, \label{def_wassertein}
\end{equation}
where $\langle\cdot, \cdot\rangle$ denotes the Frobenius inner product. The infimum is taken over all joint distributions $Q$ with marginals $\pi'$ and $\pi$, and $D$ is the cost matrix with $D_{ij} = d(a_i, a_j)$, where $d(a_i, a_j)$ is defined as the distance between actions $a_i$ and $a_j$. Its largest entry in magnitude is denoted  by $\|D\|_{\infty}$. \hlight{In implementation, our choice of distance $d$ is task-dependent and is reported in Table {\ref{appendix_hyperparameters}} in Appendix {\ref{appendix:implementation}}.}

% In implementation, for discrete action with dimension $p$, we use $d(a_i, a_j) = \sum_{k=1}^p I(a_i^k, a_j^k)$, where $I(x,y) = 0$ if $x = y$ and $1$ otherwise, and for continuous action, we use $\ell_2$ norm $d(a,a') =\lVert a-a'\rVert_2$. 
%
%\NH{use a different notation for joint distribution $P$, as it is already used for transition prob.}
%
% Apart from KL divergence, Wasserstein metric has also been used to restrict the deviation between the old and new policies to ensure learning stability. For example, \cite{richemond2017_wrl} adopts a Wasserstein metric based trust region constraint and solve the resulting RL problem using Fokker-Planck equations; \cite{zhang2018_wgf} formulates the policy optimization problem as a gradient flow in the manifold of policy distributions, where the geodesic length between policy distributions is defined by the Wasserstein distance. \cite{moskovitz2020efficient} develop Wasserstein natural policy gradients and Wasserstein natural evolution strategies by using Wasserstein metric to characterize the local behavioral structure of policies. However, most of the existing literature on Wasserstein policy optimization are policy gradient based, which means that the policy is represented as parameters, which are updated with gradient ascent. In our paper, we work on the space of policy distribution directly, and consider all admissible policies that are within the trust region constructed by the Wasserstein metric. 

\vspace{-0.2cm}

\paragraph{Sinkhorn Divergence: } Sinkhorn divergence \citep{cuturi2013_sinkhorn} provides a smooth approximation of the Wasserstein distance by adding an entropic regularizer. The Sinkhorn divergence is defined as: 
\begin{equation}
d_\text{S}(\pi',\pi|\lambda) = \inf_{Q \in \Pi(\pi', \pi)} \left\{ \langle Q, D\rangle - \frac{1}{\lambda} h(Q)\right\}, \label{def_sinkhorn}
\end{equation}
where $h(Q) = -\sum_{i=1}^N \sum_{j=1}^N Q_{ij} \log Q_{ij}$ represents the entropy term, and $\lambda > 0$ is a regularization parameter. The intuition of adding the entropic regularization is: since most elements of the optimal joint distribution $Q$ will be $0$ with a high probability, by trading the sparsity with entropy, a smoother and denser coupling between distributions can be achieved \citep{courty2014domain,courty2016optimal}. Therefore, when the weight of the entropic regularization decreases (i.e., $\lambda$ increases), the sparsity of the divergence increases, and the Sinkhorn divergence converges to the Wasserstein metric, i.e.,
$\lim_{\lambda \rightarrow \infty} d_\text{S}(\pi',\pi|\lambda) = d_\text{W}(\pi',\pi)$. 
More critically, Sinkhorn divergence is useful to mitigate the computational burden of computing Wasserstein distance. In fact, the efficiency improvement that Sinkhorn divergence and the related algorithms brought paves the way to utilize Wasserstein-like metrics in many machine learning domains, including online learning \citep{cesa2006prediction}, model selection \citep{juditsky2008learning,rigollet2011exponential}, generative modeling \citep{genevay2018learning,petzka2017regularization,patrini2020sinkhorn}, dimensionality reduction \citep{huang2021riemannian,lin2020projection,wang2021two}.
%\NH{todo: add missing references}

% To our knowledge, the study of Sinkhorn divergence for policy optimization is relatively new in the literature. The most related work is \citep{pacchiano2019_bgpg}: Behavior Guided Policy Gradient is proposed with Wasserstein distance between old and new policies penalized to prevent large policy updates in the behavior space. To mitigate the computational burden of computing Wasserstein distance, the entropy-regularized Wasserstein distance is utilized. 

\section{Wasserstein Policy Optimization}
\label{wasserstein}
% \noindent{\bf Wasserstein Policy Optimization Formulation:} Our approach is motivated by TRPO approach to restrict the size of the policy update, in order to maintain learning stability. 
Motivated by TRPO, here we consider a trust region based on the Wasserstein metric. Moreover, we lift the restrictive assumption that a policy has to follow a parametric distribution class by allowing all admissible policies. Then, the new policy $\pi'$ is found in each iteration to maximize the estimated expected value of the advantage function. Therefore, the \emph{Wasserstein Policy Optimization} (WPO) framework is shown as follows:
\begin{equation}
\begin{split}
& \max_{\pi' \in \mathcal{D}} \hspace{3mm} \mathbb{E}_{s\sim \rho^{\pi}_\upsilon, a \sim \pi'(\cdot|s)} [ A^{\pi} (s,a)] \\
& \text{where} \hspace{3mm} \mathcal{D} = \{\pi' | \mathbb{E}_{s\sim \rho^{\pi}_\upsilon} [d_{\text{W}} (\pi'(\cdot|s), \pi(\cdot|s))] \le \delta \},
\end{split}
\label{odrpo_problem}
\end{equation}
where the Wasserstein distance $d_\text{W}(\cdot,\cdot)$ is defined in (\ref{def_wassertein}).
%This framework is closely related to the literature on practicing optimism when facing an uncertain environment \citep{bi2005support,srinivas2009gaussian, hanasusanto2017ambiguous, nguyen2019optimistic}. 

In most practical cases, the reward $r$ is bounded and correspondingly, the accumulated discounted reward $R_t$ is bounded. So without loss of generality, we make the following assumption: 

% Before describing our main result, we adopt an assumption that generally holds true in most practical cases:

\begin{restatable}{assume}{assumptionboundA}
Assume $A^{\pi} (s,a)$ is bounded, i.e.,   $\sup_{a \in \mathcal{A},s\in\mathcal{S}}{|A^{\pi} (s,a)|} \leq A^{\mbox{\tiny max}}$ for some $A^{\mbox{\tiny max}}>0$.
\label{bounded_A}
\end{restatable}

With Wasserstein metric based trust region constraint, we are able to derive the closed-form of the policy update shown in Theorem \ref{thm_opt_policy_wass}. The main idea is to form the Lagrangian dual of the
constrained optimization problem presented above, \bblue{which is inspired by the way to obtain the extremal distribution in Wasserstein DRO literature, see e.g., \citet{ kuhn2019wasserstein, blanchet_wasserstein_dro,zhao2018_dro_wasserstein}}. The detailed proof can be found in Appendix \ref{appendix:wass}. 

\begin{restatable}{thm}{thmoptpolicywass}
\textbf{(Closed-form policy update)} Let $\kappa^\pi_s(\beta,j) =  \text{argmax}_{k = 1 \dots N} \{A^{\pi} (s,a_k) -  \beta D_{kj}\}$, \bblack{where $D$ denotes the cost matrix}. If Assumption \ref{bounded_A} holds, then an optimal solution to (\ref{odrpo_problem}) is: 
\begin{equation}
    \pi^*(a_i|s) = {\sum}_{j=1}^N \pi(a_j|s) f_s^*(i, j),
\label{wass_policy_update}
\end{equation}
where $f_s^*(i, j) = 1$ if \zcyan{$i = \kappa_s^\pi(\beta^*,j)$} and $f_s^*(i, j) = 0$ otherwise, and $\beta^*$ is an optimal Lagrangian multiplier corresponds to the following dual formulation: 
\begin{equation}
\min_{\beta \ge 0} \hspace{1mm} F(\beta) =  \min_{\beta \ge 0} \hspace{1mm} \{\beta\delta + \mathbb{E}_{s \sim \rho^{\pi}_\upsilon} {\sum}_{j=1}^N \pi(a_j|s)  \max_{i = 1 \dots N} (A^{\pi} (s,a_i) -  \beta D_{ij})\}.
\label{wass_dual_formulation}
\end{equation}
Moreover, we have $\beta^* \leq \bar{\beta}$, where $\bar{\beta}:=\max_{s \in \mathcal{S}, k, j = 1 \dots N, k \ne j} {(D_{kj})^{-1}}{(A^\pi (s,a_k) - A^\pi (s,a_j))} $.
\label{thm_opt_policy_wass}
\end{restatable}

\begin{remark} For ease of notation and simplicity, we assume the uniqueness of $\kappa_s^\pi(\beta,j)$ in order to form the simple expression of $f_s^*$ in Theorem \ref{thm_opt_policy_wass}. When it is not unique, a necessary condition for the  optimality of $\pi^*$ in (\ref{wass_policy_update}) is $\sum_{i \in \mathcal{K}^\pi_s(\beta,j)} f_s^*(i,j) = 1$, and $f_s^*(i, j) = 0$ for $i \notin \mathcal{K}^\pi_s(\beta,j)$, where $\mathcal{K}^\pi_s(\beta,j) =  \text{argmax}_{k = 1 \dots N} A^{\pi} (s,a_k) -  \beta D_{kj}$. The weight $f_s^*(i, j)$ for $i \in \mathcal{K}^\pi_s(\beta,j)$ could be determined through linear programming (see details in (\ref{wass_odrpo_problem_simple}) in Appendix \ref{appendix:wass}).
\end{remark}

The exact policy update for WPO  in (\ref{wass_policy_update}) requires computing the optimal Lagrangian multiplier $\beta^*$ by solving the one-dimensional subproblem (\ref{wass_dual_formulation}). A closed-form of $\beta^*$ is not easy to obtain in general, except for special cases of the distance $d(x,y)$ or cost matrix $D$. In Appendix \ref{appendix:specialdist},  we provide the closed-form of $\beta^*$ for the case when $d(x,y) = 0$ if $x = y$ and $1$ otherwise.

\textbf{WPO Policy Update:} Based on Theorem \ref{thm_opt_policy_wass}, we introduce the following WPO policy updating rule:
\begin{equation}
    \pi_{k+1}(a_i|s) =\mathbb{F}^{\textrm{WPO}}(\pi_k):= \sum_{j=1}^N \pi_k(a_j|s) f_s^k(i, j), \tag{WPO} \label{wass_policy_update_exact}
\end{equation}

\vspace{-0.2cm}

where $f_s^k(i,j) = 1$ if $i = \kappa_s^{\pi_k}(\beta_k,j)$ and $0$ otherwise. 
% \hlight{(LD: I found $f_s^k$ is not defined. I think it means we take $\beta$ to be $\beta_k$ in $k_s^\pi$. Note the current definition of $k_s^\pi (\beta,j)$ has an argmax with a dummy variable $k$ as well.)} \textcolor{blue}{Jun: Renamed $k_s^\pi$ to $\kappa_s^\pi$, since $k$ is already used to represent the iteration number. Added definition for $f_s^k$. }
Note that different from (\ref{wass_policy_update}), we allow $\beta_k$ to be chosen arbitrarily and time dependently. We show that such policy update always leads to a monotonic improvement of the performance even when $\beta_k$ is not the optimal Lagrangian multiplier.  In particular, we propose two strategies to update multiplier $\beta_k$: 

\vspace{-0.2cm}

% since the corresponding objective function is not convex in $\beta$
% We propose three ways to obtain the optimal $\beta$ value for the Wasserstein metric based policy update:

\begin{itemize}
    \item[(i)] Approximation of optimal $\beta_k$: To improve the convergence, we can approximately solve the optimal Lagrangian multiplier based on Sinkhorn divergence. More details in Section \ref{sinkhorn}. 
    \item[(ii)] Time-dependent $\beta_k$:  To improve the computational efficiency, we can simply treat $\beta_k$ as a time-dependent parameter, e.g., we can set $\beta_k$ as a diminishing sequence. In this setting, (\ref{wass_policy_update_exact}) produces the solution to the following penalty version of problem (\ref{odrpo_problem}) (with $d=d_{\text{W}}$):
\begin{equation}
\max_{\pi_{k+1}} \;\mathbb{E}_{s\sim \rho^{\pi_k}_\upsilon, a \sim {\pi_{k+1}}(\cdot|s)} [ A^{\pi_k} (s,a)] - \beta_k \mathbb{E}_{s\sim \rho^{\pi_k}_\upsilon} [d (\pi_{k+1}(\cdot|s), \pi_k(\cdot|s))].
\label{odrpo_WA}
\end{equation} 
    % We will derive the optimal $\beta$ value for the Sinkhorn divergence based policy update (will be shown in Section \ref{sinkhorn}). Then, by taking advantage of the convergence property of Sinkhorn divergence to Wasserstein metric, we will be able to approximate the optimal $\beta$ value for the Wasserstein metric based policy update by using the $\beta$ for Sinkhorn divergence.
    % \item Optimal $\beta$ of a special case. We will derive the closed form of the optimal $\beta$ for the Wasserstein metric based policy update when a special discrete distance $d(x,y)$ is considered, that is, i.e., $d(x,y) = 0$ if $x = y$ and $1$ otherwise. By taking advantage of the special structure of this distance metric and considering different initialization scenarios, we can find the local optimal $\beta$ values to \eqref{wass_dual_formulation}.
\end{itemize}

\section{Sinkhorn Policy Optimization} \label{sinkhorn}
In this section, we introduce Sinkhorn policy optimization (SPO) by constructing trust region with Sinkhorn divergence.
In the following theorem, we derive the optimal policy update in each step when using Sinkhorn divergence based trust region. Detailed proofs are provided in Appendix \ref{appendix:sinkhorn}.

\begin{restatable}{thm}{thmoptpolicysinkhorn}
If Assumption \ref{bounded_A} holds, then the optimal solution to (\ref{odrpo_problem}) with Sinkhorn divergence is: 
\begin{equation}
    \pi^*_\lambda(a_i|s) = \sum_{j=1}^N \pi(a_j|s) f_{s,\lambda}^*(i,j),
\label{sinkhorn_policy_update}
\end{equation}
where $D$ denotes the cost matrix, $f_{s,\lambda}^*(i,j) = \frac{\exp{(\frac{\lambda}{\beta_\lambda^*}A^\pi(s,a_i) - \lambda D_{ij})}}{\sum_{k=1}^N \exp{(\frac{\lambda}{\beta_\lambda^*}A^\pi(s,a_k) - \lambda D_{kj})}}$ and $\beta_\lambda^*$ is an optimal solution to the following dual formulation:
\small
\begin{equation}
\begin{split}
& \min_{\beta \ge 0} \hspace{1mm}  F_\lambda(\beta) = \min_{\beta \ge 0} 
\Big\{ \beta \delta - \mathbb{E}_{s \sim \rho^{\pi}_\upsilon} \sum_{j=1}^N  \pi(a_j|s)(\frac{\beta}{\lambda}  + \frac{\beta}{\lambda} \ln (\pi(a_j|s)) - \frac{\beta}{\lambda} \ln [\sum_{i=1}^N \exp{(\frac{\lambda}{\beta}A^\pi(s,a_i) - \lambda D_{ij})}]) \\&\qquad\qquad \mathbb{E}_{s \sim \rho^{\pi}_\upsilon} \sum_{i=1}^N \sum_{j=1}^N  \frac{\beta}{\lambda} \frac{\exp{(\frac{\lambda}{\beta}A^\pi(s,a_i) - \lambda D_{ij})} \cdot \pi(a_j|s)}{\sum_{k=1}^N \exp{(\frac{\lambda}{\beta}A^\pi(s,a_k) - \lambda D_{kj})}}\Big \}.
\end{split}
\label{sinkhorn_dual_formulation}
\normalsize
\end{equation}

\noindent Moreover, we have $\beta_{\lambda}^* \leq \frac{2A^{\mbox {\tiny max}}}{\delta} $.
\label{thm_opt_policy_sinkhorn}
\end{restatable}

In contrast to the Wasserstein dual formulation (\ref{wass_dual_formulation}), the objective in the Sinkhorn dual formulation (\ref{sinkhorn_dual_formulation}) is differentiable in $\beta$ and admits closed-form gradients (shown in Appendix \ref{appendix:gradient}). With this gradient information, we can use gradient-based global optimization algorithms \citep{wales1998_basin_hopping, zhan2006_monte_carlo_basin, leary2000_global_optimization} to find a global optimal solution $\beta_\lambda^*$ to (\ref{sinkhorn_dual_formulation}). 
% \NH{Is the objective convex at all? Can we find global optimal solution?}  

Next, we show that if the entropic regularization parameter $\lambda$ is large enough, then the optimal solution $\beta_\lambda^*$ is a close approximation to the $\beta^*$ of Wasserstein dual formulation. Proof is provided in Appendix \ref{appendix:optbetaconvergence}.

% Let $G_{\lambda}(\beta)$ represent the objective function \eqref{objective_sinkhorn_primal}. By substituting the optimal $P_{ij}^{s*}$ in \eqref{sinkhorn_pij_optimum} into \eqref{objective_sinkhorn_primal}, we have: 
% \begin{align}
% G_{\lambda}(\beta) &= \mathbb{E}_{s \sim \rho_{\upsilon}^{\pi}} [\sum_{i=1}^N A^{\pi} (s,a_i) \sum_{j=1}^N \frac{\exp{(\frac{\lambda}{\beta}A^\pi(s,a_i) - \lambda D_{ij})}}{\sum_{k=1}^N \exp{(\frac{\lambda}{\beta}A^\pi(s,a_k) - \lambda D_{kj})}} \pi(a_j|s)] \nonumber \\
% % & = \mathbb{E}_{s \sim \rho_{\upsilon}^{\pi}} [\sum_{i=1}^N \sum_{j=1}^N \pi(a_j|s) A^{\pi} (s,a_i)  \frac{\exp{(\frac{\lambda}{\beta}A^\pi(s,a_i) - \lambda D_{ij})}}{\sum_{k=1}^N \exp{(\frac{\lambda}{\beta}A^\pi(s,a_k) - \lambda D_{kj})}}] \\
% % & = \mathbb{E}_{s \sim \rho_{\upsilon}^{\pi}} [\sum_{j=1}^N \sum_{i=1}^N \pi(a_j|s) A^{\pi} (s,a_i)  \frac{\exp{(\frac{\lambda}{\beta}A^\pi(s,a_i) - \lambda D_{ij})}}{\sum_{k=1}^N \exp{(\frac{\lambda}{\beta}A^\pi(s,a_k) - \lambda D_{kj})}}] \\
% & = \mathbb{E}_{s \sim \rho_{\upsilon}^{\pi}} [\sum_{j=1}^N \pi(a_j|s) \sum_{i=1}^N  A^{\pi} (s,a_i)  \frac{\exp{(\frac{\lambda}{\beta}A^\pi(s,a_i) - \lambda D_{ij})}}{\sum_{k=1}^N \exp{(\frac{\lambda}{\beta}A^\pi(s,a_k) - \lambda D_{kj})}}]. 
% \end{align}

\begin{restatable}{thm}{thmoptbetasinkhornwassrelationship}
Define $\beta_{\text{UB}} = \max\{ \frac{2A^{\mbox {\tiny max}}}{\delta}, \bar{\beta}\}$. We have: 
\begin{enumerate}
\item 
 $F_\lambda(\beta)$ converges to $F(\beta)$ uniformly on $[0, \beta_{\text{UB}}]$: 
\hlight{$\Lim{\lambda \xrightarrow{} \infty} \Sup{0 \le \beta \le \beta_{\text{UB}}} \Big| F_\lambda(\beta) - F(\beta) \Big| \le \Lim{\lambda \xrightarrow{} \infty} \frac{\beta_{\text{UB}}}{\lambda} N \ln N = 0$}.
\item
$\begin{aligned} \lim_{\lambda \xrightarrow{} \infty} \text{argmin}_{0 \le \beta \le \beta_{\text{UB}}} F_{\lambda} (\beta) \subseteq \text{argmin}_{0 \le \beta \le \beta_{\text{UB}}} F(\beta) \end{aligned}$.
\end{enumerate}
\label{thm_opt_beta_sinkhorn_wass_relationship}
\end{restatable}

\vspace{-0.2cm}
Although it is difficult to obtain the exact value of the optimal solution $\beta^*$ to the  Wasserstein dual formulation (\ref{wass_dual_formulation}), the above theorem suggests that we can approximate $\beta^*$ via $\beta_\lambda^*$ by setting up a relative large $\lambda$. In practice, we can also adopt a smooth homotopy approach by setting an increasing sequence $\lambda_k$ for each {iteration} and letting $\lambda_k \to\infty$. 
% \ccred{Note that upper bounds of $\beta^*$ and $\beta_\lambda^*$ are only used to establish the uniform convergence. They are not used to approximate the optimizers, thus are not intended to be tight. } 

\paragraph{SPO Policy Update:} Based on Theorem \ref{thm_opt_policy_sinkhorn}, we introduce the following SPO policy updating rule: 
\begin{equation}
\pi_{k+1}(a_i|s) = \mathbb{F}^{\textrm{SPO}}(\pi_k)   =  \sum_{j=1}^N \pi_k(a_j|s) f_{s,\lambda_k}^k (i,j). \tag{SPO} \label{sinkhorn_policy_update_exact}
\end{equation}

\vspace{-0.3cm}

Here $f_{s,\lambda_k}^k (i,j) = \frac{\exp{(\frac{\lambda_k}{\beta_k}A^{\pi_k}(s,a_i) - \lambda_k D_{ij})}}{{\sum}_{l=1}^N \exp{(\frac{\lambda_k}{\beta_k}A^{\pi_k}(s,a_{l}) - \lambda_k D_{lj})}}$, $\lambda_k\geq0$ and $\beta_k\geq0$ are some control parameters. The parameter $\beta_k$  can be either  computed via solving the one-dimensional subproblem (\ref{sinkhorn_dual_formulation}) or simply set as a diminishing sequence. The proper setup of $\lambda_k$ can effectively adjust the trade-off between convergence speed and final performance. More details are provided in the ablation study in Section \ref{section_experiments}. 

% \rred{To be moved to analysis section?  ============}

% \rred{Also combine with the performance improvement of WPO since it is just a special case of it.}

% Next, we provide the theoretical justification of the performance improvement of SPO policy update when $\lambda \xrightarrow{} \infty$. The detailed proof can be found in Appendix \ref{appendix:monotonicsink}. 

% \begin{restatable}{thm}{thmmonotonicimprovementinexactsink}
% \textbf{(SPO performance improvement)} For any initial state distribution $\mu$ and any $\beta_t\geq 0$, if $||\hat{A}^{\pi} - A^{\pi} ||_\infty \le \epsilon$ for inaccurate advantage function $\hat{A}^{\pi}$ and some $\epsilon > 0$, the SPO policy update with the inaccurate advantage function $\hat{A}^{\pi}$ and $\lambda \xrightarrow{} \infty$, guarantees the following performance improvement bound,
%         \begin{equation}
%         \begin{split}
%             & {J}(\pi_{t+1}) \ge {J}(\pi_{t}) + \\ 
%             & \beta_t \mathbb{E}_{s \sim \rho_\mu^{\pi_{t+1}}} {\sum}_{j=1}^N \pi_{t}(a_j|s) \sum_{i \in \mathcal{\hat{K}}^\pi_s(\beta_t,j)} \frac{D_{ij}}{|\hat{\mathcal{K}}^\pi_s(\beta_t,j)|} - \frac{2\epsilon}{1-\gamma}, \label{monotonic_improvement_sinkhorn}
%         \end{split}
%     \end{equation}

% where $\hat{\mathcal{K}}^\pi_s(\beta_t,j) = \text{argmax}_{k = 1, \dots, N} \{\hat{A}^{\pi_t} (s, a_k) - \beta_t D_{kj}\}$. 
% \label{thm_monotonic_improvement_inexact_sink}
% \end{restatable}

% Note that this performance improvement bound is not applied to SPO with an arbitrarily fixed $\lambda$. 

% \rred{============================================}

\section{Theoretical Analysis}
\label{analysis}

We first show that SPO policy update converges to WPO policy update as the regularization parameter increases (i.e., $\lambda \xrightarrow{} \infty$). The detailed proof is provided in Appendix \ref{appendix:spo_converge_to_wpo}. 

% \hlight{LD: Is $\mathbb{F}^{\textrm{WPO}}(\pi_k)$ necessarily unique? Did we assume $k_s^\pi$ is unique? If $k_s^\pi$ is not unique, what does $\mathbb{F}^{\textrm{SPO}}(\pi_k)$ converge to?} \textcolor{blue}{Jun: We only assume the uniqueness of $\kappa_s^\pi$ in Theorem 1 but in general we don't have this assumption. So $\mathbb{F}^{\textrm{WPO}}(\pi_k)$ is not necessarily unique. $\mathbb{F}^{\textrm{SPO}}(\pi_k)$ converges to a specific item of the set $\mathbb{F}^{\textrm{WPO}}(\pi_k)$, maybe we should change to $\lim_{\lambda_k \xrightarrow{} \infty}\mathbb{F}^{\textrm{SPO}}(\pi_k) \in \mathbb{F}^{\textrm{WPO}}(\pi_k)$ ? }
% \hlight{LD: $\in$ makes more sense I think.} \textcolor{blue}{Jun: Changed to belong to in Lemma 1. }

\begin{restatable}{lem}{lemmaspoconvergetowpo}
As $\lambda_k \xrightarrow{} \infty$, SPO update converges to WPO update: $\lim_{\lambda_k \xrightarrow{} \infty}\mathbb{F}^{\textrm{SPO}}(\pi_k) \in \mathbb{F}^{\textrm{WPO}}(\pi_k)$. 
\label{lemma_spo_converge_to_wpo}
\end{restatable}
We then provide a theoretical justification that WPO policy update  (and SPO with $\lambda \xrightarrow{} \infty$) are always guaranteed to improve the true performance ${J}$ monotonically if we have access to the true advantage function. If the advantage function can only be evaluated inexactly with limited samples, then an extra estimation error (measured by the largest absolute entry $\|\cdot \|_\infty$) will occur. Proof can be found in Appendix \ref{appendix:monotonic}.

\vspace{0.1cm}

\begin{restatable}{thm}{thmmonotonicimprovementwpoandspo}
    \textbf{(Performance improvement)} For any initial state distribution $\upsilon$ and any $\beta_k \geq 0$, if $||\hat{A}^{\pi} - A^{\pi} ||_\infty \le \epsilon$ for some $\epsilon > 0$,  let $\hat{\mathcal{K}}^{\pi_k}_s(\beta_k,j) = \text{argmax}_{i = 1, \dots, N} \{\hat{A}^{\pi_k} (s, a_i) - \beta_k D_{ij}\}$, WPO policy update (and SPO with $\lambda \xrightarrow{} \infty$) guarantee the following performance improvement bound when the inaccurate advantage function $\hat{A}^{\pi}$ is used,

\vspace{-0.2cm}

\begin{equation}
{J}(\pi_{k+1}) \ge {J}(\pi_{k}) + \beta_k \mathbb{E}_{s \sim \rho_\upsilon^{\pi_{k+1}}} \sum_{j=1}^N \pi_{k}(a_j|s) \sum_{i \in \hat{\mathcal{K}}^{\pi_k}_s(\beta_k, j)} f_s^k(i, j) D_{ij}  - \frac{2\epsilon}{1-\gamma}. \label{monotonic_improvement_wpo_spo}
\end{equation}
\label{thm_monotonic_improvement_wpo_and_spo}
\end{restatable}

\vspace{-0.2cm}

\hlight{The value of $\epsilon$, which quantifies the approximation error of the advantage function, is dependent on various factors such as the advantage estimation algorithm used and the number of samples }\citep{schulmanl2016_gae_continuous}\hlight{. It is worth noting that the improvement bound of NPG/TRPO }\citep{cen2021fast}\hlight{ includes the same additional term $- \frac{2\epsilon}{1-\gamma}$, which indicates that our methods offer comparable theoretical performance guarantees to KL based updates.} In the following, we show that with a decreasing schedule of the multiplier $\beta_k$, both WPO and SPO policy updates have their values $J(\pi_k)$ converging to the optimal $J^\star = \max_{\pi} J(\pi)$ on the tabular domain. To start, for $k$-th iteration, we consider   (\ref{wass_policy_update_exact}) and (\ref{sinkhorn_policy_update_exact}) (with arbitrary $\lambda>0$) whose updates $\pi_{k+1}$ are optimal solutions to (\ref{odrpo_WA}) with $d$ being $d_{\text{W}}$ and $d_{\text{S}}$ respectively.

% \begin{equation}
% \begin{split}
% & \max_{\pi_{k+1}} \hspace{3mm} \mathbb{E}_{s\sim \rho^{\pi_k}, a \sim \pi_{k+1}(\cdot|s)} [ A^{\pi_k} (s,a)] - \beta_k \mathbb{E}_{s\sim \rho^{\pi_k}} [d (\pi_{k+1}(\cdot|s), \pi_k(\cdot|s))] .
% \end{split}
% \label{odrpo_WA}
% \end{equation}
% Here $d$ represents either the Wasserstein distance $d_\text{W}$ or the Sinkhorn divergence $d_{\text{S}}$. Aligned with the tabular domain setting, we have the following assumption:
% We denote $\pi_+$ to be a solution of the problem \eqref{odrpo_WA} and $\pi_t$ to be the $t$-th iteration of the above problem. \textcolor{blue}{(LD: I want to say by definition 
% $\pi_+ = \mathbb{F}^{\textrm{WPO}}(\pi)$ if $d = d_\text{W}$
% and $\pi_+ = \mathbb{F}^{\textrm{SPO}}(\pi)$ if $d = d_\text{S}$. Is this true?)}
% %
% \rred{Jun: I think it is equivalent to the WPO/SPO policy update with time-dependent $\beta_t$. See the top of page 5. }

% \rred{Just added the definition of $V$ in Section 2. This can be removed. ((}
% For any policy $\pi'$, we define its value function $  V^{\pi'}$ as 
% \begin{equation}\label{eq: valueFunction}
%     V^{\pi}(s) = \mathbb{E}(R_t|s_t = s,\pi),\; \forall s\in \mathcal{S}.
% \end{equation}
% By definition $J(\pi) = \mathbb{E}_{s_0\sim \nu}[V^{\pi}(s)]$.
% \rred{))}

% We assume the following finiteness of the MDP to ensure optimal value $J^*$ and optimal policy $\pi_\star$ are well-defined. To remove the finiteness of state space, extra technical conditions need to be imposed ensuring the existence of optimizers.

\begin{restatable}{assume}{assumptionfiniteness}
The state space and the action space are both finite, the reward function $r$ is non-negative, and the initial distribution covers all state.
% , i.e., the cardinality of the state space $|\mathcal{S}|$, and the cardinality of the action space $|\mathcal{A}|$, are both finite, the reward $r(s,a)\geq 0$ for any $(s,a)\in \mathcal{S}\times \mathcal{A}$, and $\rho_0(s)>0$ for all $s\in \mathcal{S}$.
\label{finiteness_B}
\end{restatable}

Note that once state and action spaces are both finite, the reward can be assumed non-negative without loss of generality, as we can always add $\max_{s,a}|r(s,a)|$ to the reward function without changing the optimal policy and the order of the policies. 
%
%Also note the finiteness of the action space also implies that the diameter $D:= \max_{a,a'\in \mathcal{A}}d(a,a')$ of the action space is finite so long as the distance $d$ takes value in real numbers. %
Defining the optimal value function $V^\star(s) = \max_{\pi} \mathbb{E}[R_t|s_t=s]$, we have the following theorem, whose proof is in Appendix \ref{appendix:globalconvergenceproof} 
and is inspired by \cite{bhandari2021linear}. 

% and $J^\star = \max_{\pi} J(\pi)$, 

% $d_{\text{W}}(\mu, \nu) \leq \max_{a,a'\in \mathcal{A}}d(a,a') = :\,D$. The diameter $D$ is finite so long as the distance $d$ takes value in real numbers.

% i.e., for any two polices $\pi$ and $\pi'$, if  $J(\pi)\geq J(\pi')$, then the order of the performance is preserved after the reward changed to be the one added by the constant $\max_{s,a}|r(s,a)|$. 

%\end{remark}
\begin{restatable}{thm}{thmglobalconvergence}
\label{thm: GlobalConvergence} \textbf{(Global convergence)}
Under Assumption \ref{finiteness_B}, 
 we have for any $\beta_k \geq 0$, (\ref{wass_policy_update_exact})  satisfies that 
\begin{equation} \label{eq: iterativeWsRL}
\|V^\star - V^{\pi_{k+1}}\|_{\infty} \leq 
\gamma    \|V^\star - V^{\pi_k}\|_{\infty} +\beta_k \|D\|_\infty, 
\end{equation}
and  (\ref{sinkhorn_policy_update_exact}) satisfies that 
\begin{equation} \label{eq: iterativeWsRL2}
\|V^\star - V^{\pi_{k+1}}\|_{\infty} \leq 
\gamma    \|V^\star - V^{\pi_k}\|_{\infty} +2\frac{\beta_k}{1-\gamma} \left(\|D\|_\infty+2\frac{\log N}{\lambda}\right).
\end{equation}
If $\lim_{k\rightarrow \infty}\beta_k= 0$, we further have
$\lim_{k\rightarrow \infty} J(\pi_{k}) = J^\star.$
\end{restatable}
%

% because 
% $0\leq  J^\star - J(\pi^t) \leq  \gamma
%     \|V^\star - V^{\pi_0}\|_{\infty} +\beta_t D$,
%\textcolor{blue}{(LD: the number $D$ in the bound \eqref{eq: iterativeWsRL} might not be correct for \eqref{sinkhorn_policy_update_exact}. It needs to be an upper bound of $d_{\text{S}}$.)}

% \begin{equation}\label{eq: VVt}
%     \|V^\star - V^{\pi^T}\|_{\infty} \leq  \gamma^T
%     \|V^\star - V^{\pi_0}\|_{\infty} +\frac{\beta D}{1-\gamma}.
% \end{equation}
% and 
\begin{remark}
Note the convergence is \emph{geometric}.  If we keep $\beta_k$ as a constant, then  $0\leq  J^\star - J(\pi^T) \leq \|V^\star - V^{\pi^T}\|_{\infty} \leq  \gamma^T
    \|V^\star - V^{\pi_0}\|_{\infty} +\frac{\beta B}{1-\gamma}$, where  $B =\|D\|_\infty $ for (\ref{wass_policy_update_exact}) and $B = 2\frac{\|D\|_\infty+2\frac{\log N}{\lambda}}{1-\gamma}$ for (\ref{sinkhorn_policy_update_exact}).
% \small
% \begin{equation}\label{eq: Jt}
%     0\leq  J^\star - J(\pi^T) \leq \|V^\star - V^{\pi^T}\|_{\infty} \leq  \gamma^T
%     \|V^\star - V^{\pi_0}\|_{\infty} +\frac{\beta \|M\|_{\infty}}{1-\gamma}.
% \end{equation}
% \normalsize
To achieve an $\epsilon$ optimality gap, we only need to take $\beta = \frac{(1-\gamma)\epsilon}{2B}$ and let $T\geq \frac{\log (\epsilon/2)}{\gamma}$.
\end{remark}

% The inequality \eqref{eq: iterativeWsRL} also reveals that if we solve the hard constrained version \eqref{odrpo_problem}, and the dual variable $\beta^*$ corresponding to the constraint shrinks to zeros, then the policy $\pi_t$ from both WPO and SPO achieves \emph{global optimality}, i.e., $\lim_{t\rightarrow \infty} J(\pi_t) = J^\star$.

\begin{remark}
The study of global non-asymptotic convergence of nonconvex  policy optimization algorithms has been an  active research topic. Recent theoretical work has mostly centered on PG and natural policy gradient (NPG) \cite{kakade2001natural} - a close relative of TRPO; see e.g., \cite{agarwal2021theory, cen2021fast, lan2022policy}. To our best knowledge, a few work has discussed the global convergence of TRPO. \cite{neu2017unified} and \cite{geist2019theory} established the connection of TRPO to Mirror Descent, but did not provide any non-asymptotic rate; \cite{shani2020adaptive} showed that adaptive TRPO with decaying stepsize achieved $O(1/\sqrt{T})$ convergence rate for unregularized MDPs in the tabular setting (finite state and finite action).   
Our result seems to be the first non-asymptotic analysis of policy optimization based on Wasserstein and Sinkhorn divergence. It remains interesting to extend the convergence theory of TRPO/WPO/SPO to function approximation regime following recent advance~\cite{agarwal2021theory}. However, this is beyond the scope of our current work, as we focus on explicit closed-form update of WPO/SPO, which can be a viable alternative to TRPO in practice.  
\end{remark}

\section{A Practical Algorithm}
\label{section_odrpo_algorithm}
% Now we focus on the development of efficient solution algorithms for WPO and SPO frameworks. We first provide a theoretical policy iteration algorithm - Algorithm \ref{policy_iteration_algorithm}. Let $\pi_{k+1} = \mathbb{F} (\pi_k)$ denote the policy update, i.e., Wasserstein metric based policy update in \eqref{wass_policy_update_exact} and Sinkhorn divergence based policy update in \eqref{sinkhorn_policy_update}, the theoretical policy iteration algorithm is shown as follows: 

In practice, the advantage value functions are often estimated from sampled trajectories. In this section, we provide a practical on-policy actor-critic algorithm, described in Algorithm \ref{odrpo_algorithm}, that combines WPO/SPO with advantage function estimation. 

% However, the general approach described in Algorithm \ref{policy_iteration_algorithm} requires explicit computations of the advantage functions, which involves evaluating the expectations of value functions over the action space. This costly step prohibits the general approach from being applied to real-world size problems, especially those with high-dimensional and continuous state space. To address the challenge, we also propose Algorithm \ref{odrpo_algorithm}: a practical on-policy algorithm. 

At each iteration, the first step is to collect trajectories, which can be either complete or partial. If the trajectory is complete, the total return can be directly expressed as the accumulated discounted rewards $R_t  = \sum_{k=0}^{T-t-1} \gamma^k r_{t+k}$. 
If the trajectory is partial, it can be estimated by applying the multi-step temporal difference (TD) methods \citep{asis2017_multistep}: $\hat{R}_{t:t+n} = \sum_{k=0}^{n-1} \gamma^k r_{t+k} + \gamma^n V(s_{t+n})$. 
Then for the advantage estimation, we can use Monte Carlo advantage estimation, i.e., $\hat{A}_t^{\pi_k}=R_t - V_{\psi_k}(s_t)$ or Generalized Advantage Estimation (GAE) \citep{schulmanl2016_gae_continuous}, which provides a more explicit control over the bias-variance trade-off. \zcyan{In the value update step, we use a neural net to represent the value function, where $\psi$ is the parameter that specifies the value net $s \rightarrow V(s)$. Then, we can update $\psi$ by using gradient descent, which significantly reduces the computational burden of computing advantage directly.} The computational complexity of the algorithm is discussed in Appendix \ref{appendix:complexity}.

\begin{algorithm2e}[H]
\SetAlgoLined
Input: number of iterations $K$, learning rate $\alpha$ \\
Initialize policy $\pi_0$ and value network $V_{\psi_0}$ with random parameter $\psi_0$\\
 \For{$k = 0,1,2 \dots K$}{
  Collect trajectory set $\mathcal{D}_k$ on policy $\pi_k$
  \\
  For each timestep $t$ in each trajectory, compute total returns $G_t$  and estimate advantages $\hat{A}_t^{\pi_k}$
  \\
  Update value: \\ $\begin{aligned} \psi_{k+1} \xleftarrow[]{} \psi_{k} - \alpha \nabla_{\psi_{k}} \sum (G_t - V_{\psi_k}(s_t))^2 \end{aligned}$
  \\
  Update policy: \\
  $\pi_{k+1} \xleftarrow[]{} \mathbb{F} (\pi_k)$ via WPO/ SPO with $\hat{A}_t^{\pi_k}$ \\
 }
 \caption{On-policy WPO/SPO algorithm}
 \label{odrpo_algorithm}
\end{algorithm2e}

\section{Experiments}
\label{section_experiments}

In this section, we evaluate the proposed WPO and SPO approaches presented in Algorithm \ref{odrpo_algorithm}.  We compare the performance of our methods with benchmarks including TRPO \citep{schulman2015_trpo}, PPO \citep{schulman2017_ppo}, A2C \citep{mnih2016_a3c}; and with BGPG \citep{pacchiano2019_bgpg}, WNPG \citep{moskovitz2021wnpg} for continuous control. 
% We compare with A2C because it is similar to our framework in the sense that both of them are simple on-policy actor-critic methods that utilize the advantage information to perform policy updates. 
The code of our WPO/SPO can be found  here\footnote{\url{https://github.com/efficientwpo/EfficientWPO}}. We adopt the implementations of TRPO, PPO and A2C from OpenAI Baselines \citep{openaibaseline} for MuJuCo tasks and Stable Baselines \citep{stable-baselines} for other tasks. For BGPG, we adopt the same implementation\footnote{\url{https://github.com/behaviorguidedRL/BGRL}} as \citep{pacchiano2019_bgpg}.

 Our experiments include (1) ablation study that focuses on sensitivity analysis of WPO and SPO; (2) tabular domain tasks with discrete state and action including the Taxi, Chain, and Cliff Walking environments;  (3) locomotion tasks with continuous state and discrete action including the CartPole, Acrobot environments; (4) comparison of KL and Wasserstein trust regions under tabular domain and locomotion tasks; and (5) extension to continuous control tasks with continuous action including HalfCheetah, Hopper, Walker, and Ant environments from MuJuCo.  See Table \ref{summary_performance_table} in Appendix \ref{appendix:implementation} for a summary of performance. The setting of hyperparameters and network sizes of our algorithms and additional results are provided in Appendix \ref{appendix:implementation}. 
 
% % summarizes the numerical values of the final performance comparison among our algorithms and the baseline methods  for these experiments.  We further elaborate our results in separate subsections below.  In all figures, the shaded area depicts the mean $\pm$ the standard deviation.

\subsection{Ablation Study}
\label{section_ablation}

In this experiment, we first examine the sensitivity of WPO in terms of different strategies of $\beta_k$.  We test four settings of $\beta$ value for WPO policy update: (1) Setting 1: Computing optimal $\beta$ value for all policy update; (2) Setting 2: Computing optimal $\beta$ value for first $20\%$ of policy updates and decaying $\beta$ for the remaining; (3) Setting 3: Computing optimal $\beta$ value for first $20\%$ of policy updates and fix $\beta$ as its last updated value for the remaining; (4) Setting 4: Decaying $\beta$ for all policy updates (e.g., $\beta_k = \Theta(1/\log{k})$). In particular, Setting 2 is rooted in the observation that $\beta^*$ decays slowly in the later stage of the experiments carried out in the paper. Small perturbations are added to the approximate values to avoid any stagnation in updating. Taxi task \citep{dietterich1998_taxi} from tabular domain is selected for this experiment. 

\begin{table}[H]
\centering
\caption{Runtime for different $\beta$ settings, average across $5$ runs with random initialization \label{runtime_beta}}
\renewcommand{\arraystretch}{1.0}
\begin{adjustbox}{width=0.5\columnwidth,center}
\begin{tabular}{lll}
 \hline
 Runtime & Taxi (s) & CartPole (s) \\
 \hline
 Setting 1 (optimal $\beta$) & 1224.3 \hlight{$\pm$ 105.7} & 129.7 \hlight{$\pm$ 15.2}  \\
 \hline
 Setting 2 (optimal-then-decay) & 648.4 \hlight{$\pm$ 55.7} & 63.2 \hlight{$\pm$ 8.3} \\
 \hline
 Setting 3 (optimal-then-fix) & 630.2 \hlight{$\pm$ 67.4} & 67.1 \hlight{$\pm$ 9.7} \\
 \hline
 Setting 4 (decaying $\beta$) & 522.7 \hlight{$\pm$ 49.5} & 44.3 \hlight{$\pm$ 6.2}\\
 \hline
\end{tabular}
\end{adjustbox}
\end{table}

The performance comparisons and average run times are shown in Figure \ref{fig:tabular_different_betas_lambdas} and Table \ref{runtime_beta} respectively. Figure \ref{fig:tabular_different_betas} and Table \ref{runtime_beta} clearly indicate a tradeoff between computation efficiency and accuracy in terms of different choices of $\beta$ value. 
%  the c, Setting 1, though yields the best performance, is most computational costly. Setting 2 is comparable to the setting 1 in terms of final performance and convergence, but more computational efficiency. Setting 3 is dominated by setting 2 in final performance, though leads to a similar run time as setting 2. Setting 4, which has the least computational burden, has the worst performance as well. 
Setting 2 is the most effective way to balance the tradeoff between performance and run time. For the rest of experiments,  we adopt this setting for both WPO and SPO \hlight{(see Appendix {\ref{appendix:implementation_setting_2}} for how Setting 2 is tuned for each task)}. Figure \ref{fig:tabular_different_lambdas} shows that 
% the performance of SPO under different choices of regularization parameter $\lambda$. 
as $\lambda$ increases, the convergence of SPO becomes slower but the final performance of SPO improves and becomes closer to that of WPO, which verifies the convergence property of Sinkhorn to Wasserstein distance shown in Theorem \ref{thm_opt_beta_sinkhorn_wass_relationship}. Therefore, the choice of $\lambda$ can effectively adjust the trade-off between convergence and final performance. 
% With a proper $\lambda$ choice, SPO is able to attain a faster convergence speed with an optimum that is close to WPO. 
Similar results are observed when using time-varying $\lambda$ on Taxi, Chain and CartPole tasks, presented in Figure \ref{fig:additional_ablation} in Appendix \ref{appendix:implementation}.

\begin{figure}[H]
\centering
  \begin{subfigure}{0.35\columnwidth}
    \centering
    \includegraphics[width=0.99\linewidth]{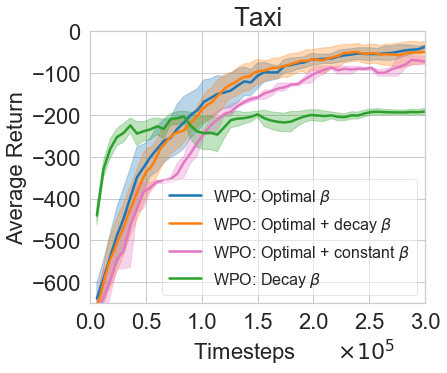}
    \caption{Different settings of $\beta$}
    \label{fig:tabular_different_betas}
  \end{subfigure}
  \hspace{0.5cm}
  \begin{subfigure}{0.35\columnwidth}
    \centering
    \includegraphics[width=0.99\linewidth]{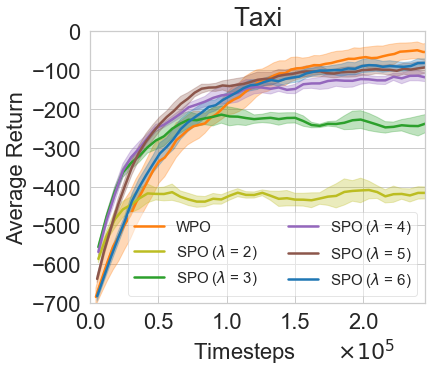}
    \caption{Different choices of $\lambda$}
    \label{fig:tabular_different_lambdas}
  \end{subfigure}
\caption{Episode rewards for Taxi with different $\beta$ and $\lambda$ settings, averaged across $5$ runs with random initialization. The shaded area depicts the mean $\pm$ the standard deviation.}
\label{fig:tabular_different_betas_lambdas}
\end{figure}

% In order to avoid the computationally cumbersome step of obtaining $\beta^*$ for WPO, in each iteration of the above experiments, we approximate that with the $\beta$ value by using the one obtained from Sinkhorn divergence with a large $\lambda$ in first $20\%$ policy updates. In the rest of policy updates, we choose $\beta$ as a time dependent control parameter, i.e., $\beta_t = \frac{1}{t^2}$ for $t$th policy update. This is based on our observation that WPO converges quickly so that in the later stage of policy update, trust region constraint is not violated in most cases. Therefore, we choose a decaying $\beta$ as it is not necessary to penalize the violation of trust region constraint. This trick generally reduces the computational time and speeds up the convergence without noticeable performance degradation. We also justify this setting of $\beta$ by the following experiments.
% 

\subsection{Tabular Domains}

We evaluate WPO and SPO on tabular domain tasks and test the exploration ability of the algorithms on several environments  including Taxi, Chain, and Cliff Walking. We use a table of size $|\mathcal{S}| \times |\mathcal{A}|$ to represent the policy $\pi(a|s)$. For the value function, we use a neural net to smoothly update the values. The performance of WPO \zzred{and SPO are} compared to the performance of TRPO, PPO and A2C under the same neural net structure. Results on Taxi, Cliff and Chain are reported in  Figure \ref{fig:tabular}.

As shown in Figure \ref{fig:tabular}, the performances of WPO, SPO and TRPO are manifestly better than  A2C and PPO. Among the trust region based methods, WPO and SPO outperform TRPO in Taxi and Cliff Walking, whereas in Chain, the performances of these three methods are comparable. In all of the test cases, SPO converges faster than WPO but to a lower optimum. As further shown in Table \ref{taxi_table}, for the Taxi environment,  WPO has a higher successful drop-off rate and a lower task completion time while the original TRPO reaches the time limit with a drop-off rate $0$, suggesting that WPO finds a better policy than the original TRPO. In Figure \ref{fig:tabular_different_Ns}, we also compare the performance of WPO under Wasserstein and KL divergences given different  number of samples $N_A$ used to estimate the advantage function, and the result suggests that using Wasserstein metric is more robust than KL divergence under inaccurate advantage values.

\begin{figure}[H]
\centering
  \includegraphics[width=0.3\linewidth]{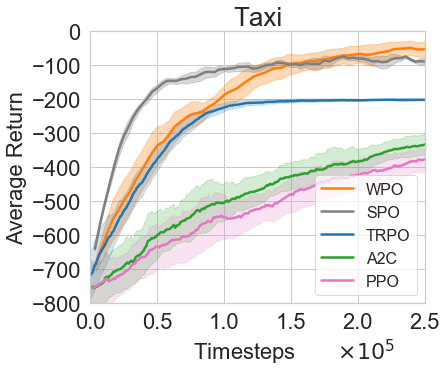} 
  \hspace{0.25cm}
  \includegraphics[width=0.3\linewidth]{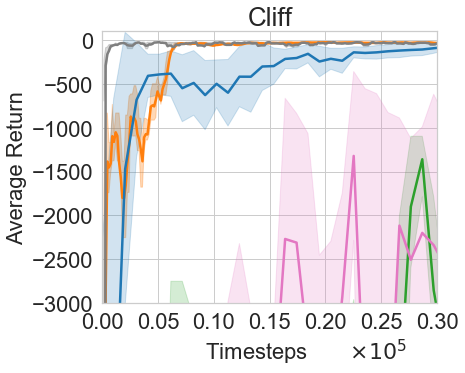}  
  \hspace{0.25cm}
  \includegraphics[width=0.3\linewidth]{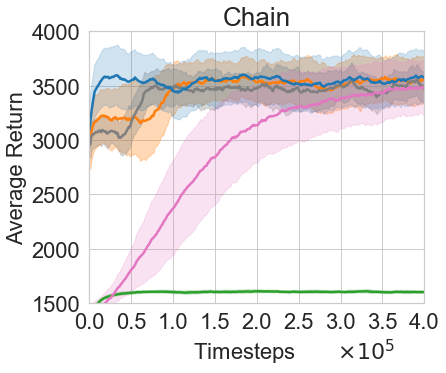}
\caption{Episode rewards during training for tabular domain tasks, averaged across $5$ runs with random initialization. The shaded area depicts the mean $\pm$ the standard deviation.  \label{fig:tabular}}
\end{figure}

\begin{table}[H]
\caption{Trained agents performance on Taxi \label{taxi_table}}
\renewcommand{\arraystretch}{1.0}
\begin{adjustbox}{width=0.5\columnwidth,center}
\begin{tabular}{lllll}
 \hline
 & Success (+20) & Fail (-10) &  Steps (-1) & Return  \\
 \hline 
 WPO & 0.753 & 0.232 & 70.891 &  -58.151 \\
 \hline 
 TRPO & 0 & 0 & 200 & -200 \\
 \hline 
\end{tabular}
\end{adjustbox}
\end{table}

\subsection{Robotic Locomotion Tasks}
We now integrate deep neural network architecture into WPO and SPO  and evaluate their performance on several locomotion tasks (with continuous state and discrete action), including CartPole \citep{barto1983_carpole} and Acrobot \citep{geramifard2015_acrobot}. We use two separate neural nets to represent the policy and the value. The policy neural net receives state $s$ as an input and outputs the categorical distribution of $\pi(a|s)$. A random subset of states $\mathcal{S}_k \in \mathcal{S}$ is sampled at each iteration to perform policy updates. 

Figure \ref{fig:locomotion_control} shows that WPO \zzred{and SPO outperform} TRPO, PPO and A2C in most tasks in terms of final performance, except in Acrobot where PPO performs the best. \zzred{In most cases, SPO converges faster but WPO has a better final performance.}  To train $10^5$ timesteps in the discrete locomotion tasks, the training wall-clock time is $63.2$ \hlight{$\pm 8.2s$} for WPO, $64.7$ \hlight{$\pm 7.8s$} for SPO, $59.4$ \hlight{$\pm 10.2s$} for TRPO and $69.9$ \hlight{$\pm 10.5s$} for PPO. Therefore, WPO has a similar computational efficiency as TRPO and PPO.

\begin{figure}[H]
\centering
  \includegraphics[width=.24\linewidth]{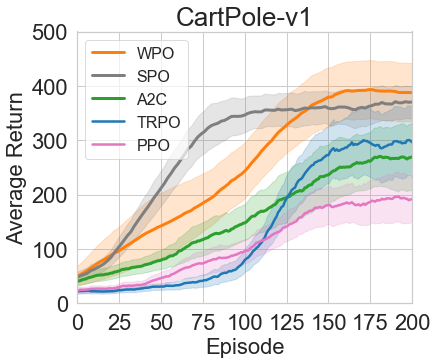} 
  \includegraphics[width=.24\linewidth]{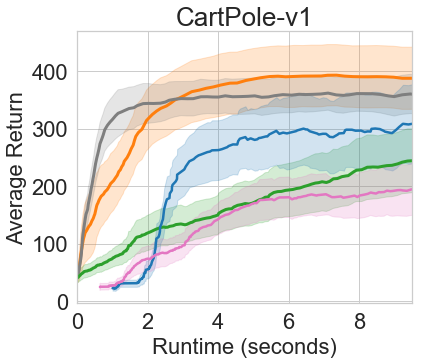}
  \includegraphics[width=.24\linewidth]{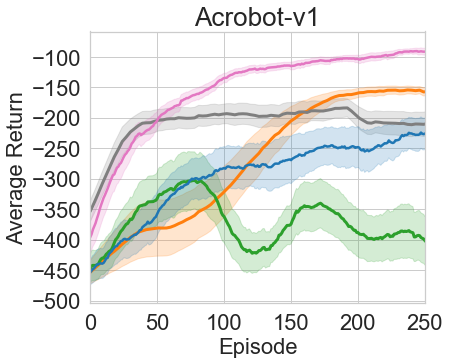}
  \includegraphics[width=.24 \linewidth]{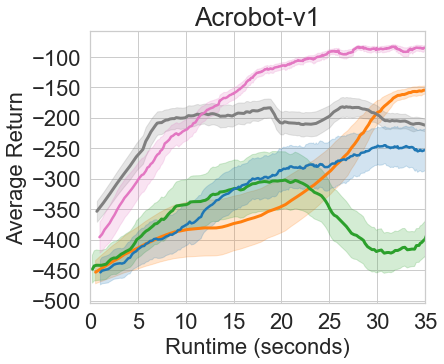}
\caption{Episode rewards during the training process for the locomotion tasks, averaged across $5$ runs with  random initialization. The shaded area depicts the mean $\pm$ the standard deviation.  \label{fig:locomotion_control}}
\end{figure}

% The performances of WPO and KL are also compared for the discrete locomotion tasks under different $N_A$. \zblue{The results are reported in Figure \ref{fig:locomotion_control_different_Ns} in Appendix \ref{appendix:implementation}.} We observe a similar results as the tabular domain: when $N_A$ is small, WPO significantly outperforms KL. These results indicate that for discrete locomotion tasks, WPO is more robust than KL when advantage values are inaccurate.

% \begin{figure}[H]
% \centering
% \subfloat[$N_A = 100$]{
%   \includegraphics[width=.22\linewidth]{cartpole_N_100.png}} 
% \subfloat[$N_A = 500$]{
%   \includegraphics[width=.22\linewidth]{cartpole_N_500.png}} 
% \subfloat[$N_A = 100$]{
%   \includegraphics[width=.22\linewidth]{acrobot_N_100.png}}
% \subfloat[$N_A = 500$]{
%   \includegraphics[width=.22\linewidth]{acrobot_N_500.png}}
% \caption{Episode rewards during the training process for the locomotion tasks, averaged across $5$ runs with a random initialization. The shaded area depicts the mean $\pm$ the standard deviation. \vspace{-.5em} \label{fig:locomotion_control_different_Ns}}
% \end{figure}

\subsection{Comparison of Wasserstein and KL Trust Regions}

We show that compared with the KL divergence, the utilization of Wasserstein metric can cope with the inaccurate advantage estimations caused by the lack of samples. Let $N_A$ denote the number of samples used to estimate the advantage function. 
We evaluate the performance of WPO framework (\ref{odrpo_problem}) with Wasserstein and KL constraints (as derived in \citet{peng2019advantage}). We consider the Chain task and different $N_A$. As shown in Figure \ref{fig:tabular_different_Ns}, when $N_A$ is $1000$, KL performs slightly better than WPO. However, when $N_A$ decreases to $100$ or $250$, WPO outperforms KL. These results indicate that WPO is more robust than KL under inaccurate advantage values. This finding is consistent with our observations on the policy update formulations of Wasserstein and KL. For the Wasserstein update in (\ref{wass_policy_update}), policy will be updated only when the advantage difference between two actions is significant, i.e., $A^\pi(s,a_j) - \beta D_{ij} \ge A^\pi(s,a_i)$. However, for the KL update in  \citet{peng2019advantage}, policy will be updated as long as the current advantage function has a single non-zero value. Therefore, KL update is more sensitive; while Wasserstein update is more robust and more tolerant to advantage inaccuracies. Similar results are obtained for the locomotion tasks (Figure \ref{fig:locomotion_control_different_Ns} in Appendix \ref{appendix:implementation}). \hlight{The runtime of Wasserstein and KL updates are reported in Table {\ref{runtime_wpo_kl}} in Appendix {\ref{appendix:implementation}}.}

\begin{figure}[H]
\centering
    \subfloat[$N_A = 100$]{\includegraphics[width=0.25\linewidth]{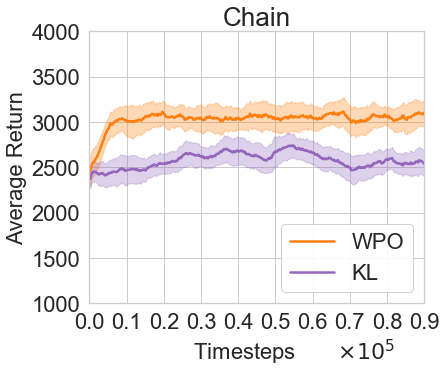}}  
    \hspace{0.35cm}
    \subfloat[$N_A = 250$]{\includegraphics[width=0.25\linewidth]{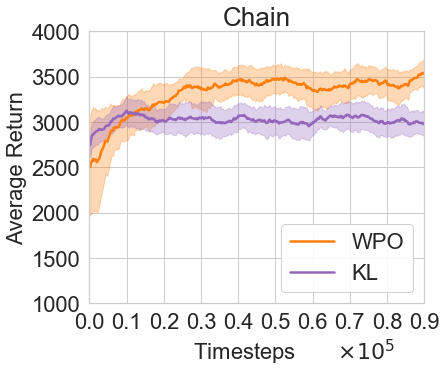}} 
    \hspace{0.35cm}
    \subfloat[$N_A = 1000$]{\includegraphics[width=0.25\linewidth]{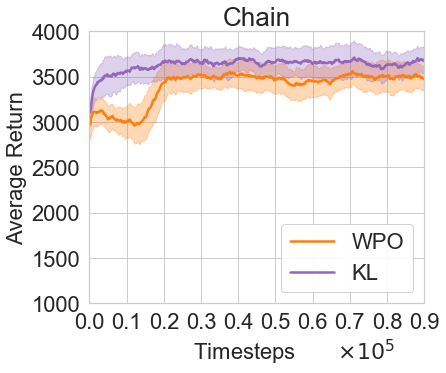}}
\caption{Episode rewards during training for the Chain task, where advantage value function is estimated under different number of samples, averaged across 5 runs with random initialization. The shaded area depicts the mean $\pm$ the standard deviation.} 
\label{fig:tabular_different_Ns}
\end{figure}

\subsection{Extension to Continuous Control}
\label{continuous_control}

To extend to environments with continuous action, we use Implicit Quantile Networks (IQN) \citep{dabney2018iqn} actor that can represent an arbitrary complex non-parametric policy. Let $F_s^{-1}(p)$ represent the quantile function associated with policy $\pi(\cdot|s)$. The IQN actor takes state $s$ and probability $p \in [0,1]$ as input, and outputs the corresponding quantile value $a = F_s^{-1}(p)$. IQN actor can be trained to approach pre-defined target policy distributions through quantile regression \citep{dabney2018iqn, tessler2019_dpo}. 

Define the action support for state $s$ in $k$-th iteration as $I^{\pi_k}(s) = \{a': A^{\pi_k}(s,a') > \min_{a \in I^{\pi_{k-1}}(s)} A^{\pi_k}(s,a)\}$. Then, the WPO/SPO target policy distribution to guide IQN update in the $k$-th iteration is: 
\begin{equation}
    P_{I^{\pi_k}(s)} (a'|s) = {\sum}_{a \in I^{\pi_{\tiny {k-1}}}(s)} \pi_k(a|s) f_s(a',a),
\end{equation}
where for WPO update $f_s(a',a) = 1$ if $a'= \text{argmax}_{a' \in I^{\pi_k}(s)} \{A^{\pi_k}(s,a') - \beta_k d(a',a)\}$ and $f_s(a',a) = 0$ otherwise; for SPO update, $f_s(a',a) =  \frac{\exp(\frac{\lambda_k}{\beta_k}A^{\pi_k}(s,a')-\lambda_k d(a',a))}{\sum_{a' \in I^{\pi_{\tiny {k}}}(s)} \exp(\frac{\lambda_k}{\beta_k}A^{\pi_k}(s,a')-\lambda_k d(a',a))}$. 
In implementation, we sample a batch of states $\mathcal{S}_k \in \mathcal{S}$ at each iteration to perform policy updates, and for each $s \in \mathcal{S}_k$, we sample $|\mathcal{A}_k|$ actions to approximate the support $I^{\pi_k}(s)$ and the target policy distribution $P_{I^{\pi_k}(s)} (\cdot|s)$. 

We additionally compare WPO and SPO with BGPG \citep{pacchiano2019_bgpg} and WNPG \citep{moskovitz2021wnpg} that are specially designed to address the continuous control with Wasserstein metric, for several MuJuCo tasks including HalfCheetah, Hopper, Walker, and Ant. Figure \ref{fig:mujuco} shows that WPO and SPO have consistently better performances than other benchmarks. \hlight{Similar results are obtained for the challenging Humanoid task, presented in Figure {\ref{fig:humanoid}} in Appendix {\ref{appendix:implementation}}. We also provide the runtime of each algorithm in Table {\ref{runtime_continuous_control}} in Appendix {\ref{appendix:implementation}}.}

\begin{figure}[H]
\centering
  \includegraphics[width=.24\linewidth]{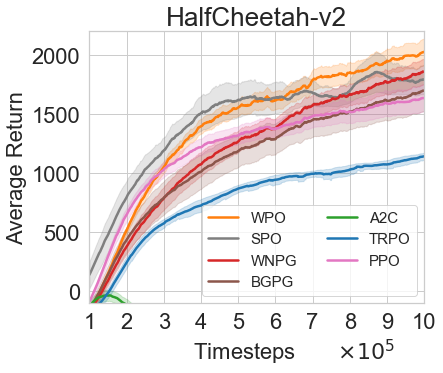}
  \includegraphics[width=.24\linewidth]{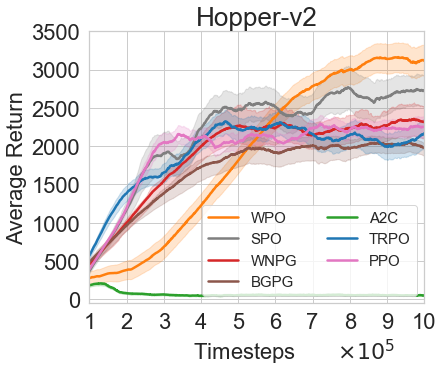}
  \includegraphics[width=.24\linewidth]{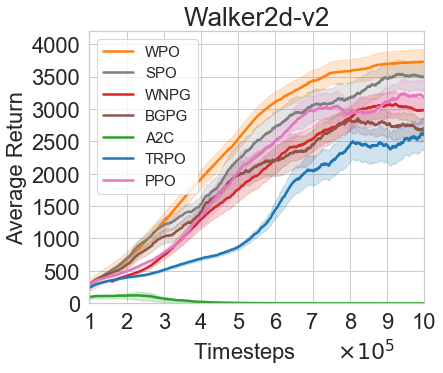}
  \includegraphics[width=.24 \linewidth]{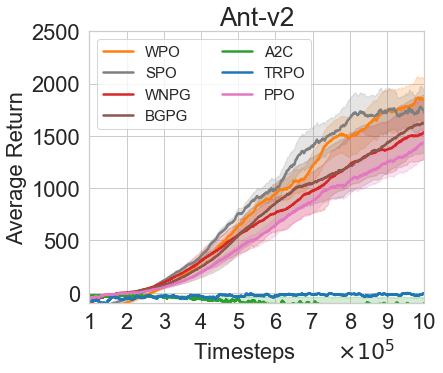}
\caption{\hlight{Episode rewards during training for MuJuCo continuous control tasks, averaged across $10$ runs with random initialization. The shaded area depicts the mean $\pm$ the standard deviation.}}
\label{fig:mujuco}
\end{figure}

\section{Conclusion}\label{conclusion}
In this paper, we present two policy optimization frameworks, WPO and SPO, which can exactly characterize the policy updates instead of confining their distributions to particular distribution class or requiring any approximation. Our methods outperform TRPO and PPO with better sample efficiency, faster convergence, and improved final performance. Our numerical results show that the Wasserstein metric is more robust to the ambiguity of advantage functions, compared with the KL divergence. Our strategy for adjusting $\beta$ value for WPO can reduce the computational time and boost the convergence without noticeable performance degradation. SPO improves the convergence speed of WPO by properly choosing the weight of the entropic regularizer. Performance improvement and global convergence for WPO are discussed. For future work, it remains interesting to extend the idea to PPO and natural policy gradients, which penalize the policy update instead of imposing trust region constraint, and extend it to off-policy frameworks. 

% There are many interesting directions we plan to explore in the future. First, the idea of exact characterization of policies can be applied to PPO and natural policy gradient, which penalize the policy update instead of imposing trust region constraint. Second, both WPO and SPO are theoretically grounded methods, yet at the same time, on-policy actor-critic algorithms. One focus for the future work will be on
% integrating WPO and SPO with other well-known off-policy frameworks, such as Soft Actor Critic \citep{haarnoja2018_sac} and Twin Delayed DDPG \citep{fujimoto2018_td3}.

% Last, extending to continuous control space and exploring the convergence rate, are directions that will also be considered in future research.

%  \section{Broader impact}
% By employing the probabilistic model-free conception into the reinforcement learning domain, we open the opportunities for researchers and practitioners to achieve better learning outcomes. The reason is three-fold.
% First of all, by replacing untrustworthy probabilistic assumptions with a confidence set that covers all permissible distributions, a potentially better result is theoretically guaranteed. Second, by avoiding having to resort to heuristics, each step can be solved to its optimality using a rigorous optimization approach. Third, it greatly reduces the effort of tackling nonconvexity constraints introduced by parametric policy assumptions, which is oftentimes the major computational obstacles in solving the subroutine of the model-free reinforcement learning framework. 

\newpage
\bibliography{ref}
\bibliographystyle{tmlr}

%%%%%%%%%%%%%%%%%%%%%%%%%%%%%%%%%%%%%%%%%%%%%%%%%%%%%%%%%%%%%%%%%%%%%%%%%%%%%%%
%%%%%%%%%%%%%%%%%%%%%%%%%%%%%%%%%%%%%%%%%%%%%%%%%%%%%%%%%%%%%%%%%%%%%%%%%%%%%%%
% APPENDIX
%%%%%%%%%%%%%%%%%%%%%%%%%%%%%%%%%%%%%%%%%%%%%%%%%%%%%%%%%%%%%%%%%%%%%%%%%%%%%%%
%%%%%%%%%%%%%%%%%%%%%%%%%%%%%%%%%%%%%%%%%%%%%%%%%%%%%%%%%%%%%%%%%%%%%%%%%%%%%%%
\newpage
\appendix
\onecolumn

%%%%%%%%%%%%%%%%%%%%%%%%%%%%%%%%%%%%%%%%%%%%%%%%%%%%%%%%%%%%%%%%%%%%%%%%%%%%%%%
%%%%%%%%%%%%%%%%%%%%%%%%%%%%%%%%%%%%%%%%%%%%%%%%%%%%%%%%%%%%%%%%%%%%%%%%%%%%%%%

\section{Implementation Details and Additional Results}
\label{appendix:implementation}

\bblue{The implementation of our WPO/SPO can be found in https://github.com/efficientwpo/EfficientWPO. We use the implementations of TRPO, PPO and A2C from OpenAI Baselines \citep{openaibaseline} for MuJuCo tasks and Stable Baselines \citep{stable-baselines} for other tasks. For BGPG, we adopt the same implementation as Pacchinao et al., (2020) based on the released code https://github.com/behaviorguidedRL/BGRL.}

\subsection{Visitation Frequencies Estimation: } The unnormalized discounted visitation frequencies are needed to compute the global optimal $\beta^*$. At the $k$-th iteration, the visitation frequencies $\rho^\pi_k$ are estimated using samples of the trajectory set $\mathcal{D}_k$. Specifically, we first initialize $\rho^\pi_k (s) = 0, \; \forall s \in S$. Then for each timestep $t$ in each trajectory from $\mathcal{D}_k$, we update $\rho^\pi_k$ as $\rho^\pi_k(s_t) \xleftarrow[]{} \rho^\pi_k (s_t) + \gamma^t / |\mathcal{D}_k|$.

\subsection{Optimal-then-decay Beta Strategy: }
\label{appendix:implementation_setting_2}

\hlight{During the training of multiple tasks, including Taxi, Chain and CartPole, we observe a consistent trend in the behavior of the optimal $\beta$ value during the policy updates: It initially fluctuates, then stabilizes and decays slowly towards 0. In the Taxi task, the optimal $\beta$ stabilizes after approximately $18\%$ of the total training iterations. If we decay $\beta$ before this stabilization point (e.g, using optimal beta for only first $5\%$ or $10\%$ updates), we observe a drop in performance. However, we do not observe any notable performance difference when we decay $\beta$ after this stabilization point (e.g., using optimal $\beta$ for first $20\%$ or $30\%$ updates). We also observe that the optimal $\beta$ decays at a very slow rate, and $\Theta(1/\log(k))$ matches this trend best. If we employ a faster decaying function, such as $\Theta(1/k)$ or $\Theta(1/k^2)$, we observe a drop in performance.}

\hlight{Based on these findings, when implementing the optimal-then-decay $\beta$ strategy on other tasks, we compute the optimal $\beta$ for each policy update until we observe that its value stabilizes across updates. At this point, we stop calculating the optimal $\beta$ and decay it using $\Theta(1/\log(k))$ for the remaining policy updates. The specific iteration at which the optimal $\beta$ value stabilizes varies across tasks, and we denote this point as $k_\beta$, which is reported in Table {\ref{appendix_hyperparameters}}. }

% \textbf{Policy Representation: } The general approach depicted in Algorithm \ref{odrpo_algorithm} allows various policy representations including tables and neural networks. In tabular domains where $|\mathcal{S}|$ and $|\mathcal{A}|$ are finite, we adopt a table of size $|\mathcal{S}| \times |\mathcal{A}|$ to represent the policy. In a high-dimensional environment (i.e., continuous state and/or continuous action), we adopt an Implicit Quantile Network (IQN) actor, which is able to represent arbitrary complex policy.  

\subsection{Hyperparameters and Performance Summary}

Our main experimental results are reported in section \ref{section_experiments}. In addition, we provide the setting of hyperparameters and network sizes of our WPO/SPO algorithms in Table \ref{appendix_hyperparameters}, and a summary of performance in Table \ref{summary_performance_table}.

% And we present the numerical results of the final performance comparison among our algorithms and the baseline methods (i.e., TRPO, PPO, A2C) in Table \ref{appendix_performance_table}. 

\renewcommand*{\arraystretch}{1.3}
\begin{longtable}{llllll}
\caption{Hyperparameters and network sizes} \\
 \hline
 & Taxi-v3 & NChain-v0 & CartPole-v1 & Acrobot-v1  &  MuJuCo tasks \\
 & & CliffWalking-v0 & & &  \\
 \hline
 $\gamma$ & $0.9$ & 0.9 & $0.95$ & $0.95$ & $0.99$ \\
 \hline
 $lr_{\pi}$ & \textbackslash & \textbackslash & $10^{-2}$ & $5 \times 10^{-3}$ & $10^{-4}$ \\
 \hline
 $lr_{\text{value}}$ & $10^{-2}$ & $10^{-2}$ & $10^{-2}$ & $5 \times 10^{-3}$ & $10^{-3}$ \\
 \hline
 $|\mathcal{D}_k|$ & $60$ (Taxi) & $1$ (Chain) & 2 & 3 & partial \\
 & & $3$ (CliffWalking) &  &  & \\
 \hline
  $\pi$ size & 2D array & 2D array & $[64,64]$ & $[64,64]$  & $[400,300]$ \\
 \hline
 \text{Q/v} size & $[10,7,5]$ & $[10,7,5]$ & $[64,64]$ & $[64,64]$ & $[400,300]$ \\
  \hline
 $|\mathcal{S}_k|$ & all states, $|\mathcal{S}|$ & all states, $|\mathcal{S}|$ & $128$ & $128$ & $64$ \\
 \hline
  $|\mathcal{A}_k|$ & all actions, $|\mathcal{A}|$ & all actions, $|\mathcal{A}|$ & all actions, $|\mathcal{A}|$ & all actions, $|\mathcal{A}|$ & $32$ \\
 \hline
 \hlight{$d(a,a')$} & \hlight{0-1 distance} \footnote{We note that specifying distance based on control relevance leads to higher performance in this test case: i.e., $d=1$ to distinct actions from set $A$ = $\{$ move north, move south, move west, move east$\}$, $d=1$ to distinct actions from set $B$ = $\{$ pickup, dropoff$\}$, and $d=4$ to actions from different sets. } & \hlight{0-1 distance} & \hlight{0-1 distance} & \hlight{0-1 distance} &  \hlight{L1 distance} \\
 \hline
 \hlight{$k_\beta$} & \hlight{$250$} & \hlight{$100$ (Chain)} & \hlight{$150$} & \hlight{$150$} & \hlight{$1000$} \\ 
 & & \hlight{$50$ (CliffWalking)} &  &  & \\
 \hline
\label{appendix_hyperparameters}
\end{longtable}

\begin{table}[H]
\caption{Averaged rewards over last 10\% episodes during the training process}
\renewcommand{\arraystretch}{1.5}
\centering
\begin{adjustbox}{width=\linewidth,center}
\begin{tabular}{llllllll}
 \hline\hline
Environment  & WPO & SPO & TRPO & PPO & A2C & BGPG & WNPG \\
\hline
Taxi-v3 & $-45 \pm 27$ & $-87 \pm 11$ & $-202 \pm 3$ & $-381 \pm 34$ & $-338 \pm 30$ & - & -\\ 
\hline
NChain-v0 & $3549 \pm 197$ & $3432 \pm 131$ & $3522 \pm 258$ & $3506 \pm 237$ & $1606 \pm 10$ & -  & -\\
\hline
CliffWalking-v0 & $-35 \pm 15$ & $-25 \pm 1$ & $-159 \pm 94$ & $-3290 \pm 2106$ & $-5587 \pm 1942$  & - & - \\ 
\hline\hline
CartPole-v1 & $388 \pm 54$ & $370 \pm 30$ & $297 \pm 65$ & $193 \pm 45$ & $267 \pm 61$ & - & -\\
\hline
Acrobot-v1 & $-162 \pm 8$ & $-185 \pm 15$ & $-248 \pm 33$ & $-103 \pm 5$ & $-379 \pm 39$ & -  & -\\
\hline\hline
HalfCheetah-v2 & $2050 \pm 108$ & $1750 \pm 172$ & $1158 \pm 35$ & $1628 \pm 136$ & $-645 \pm 31$ & $1697 \pm 195$ & $1832 \pm 125$ \\
\hline
Hopper-v2 & $3208 \pm 259$ & $2834 \pm 305$ & $2035 \pm 248$ & $2321 \pm 233$ & $43 \pm 21$ & $1982 \pm 218$ & $2361 \pm 272$\\
\hline
Walker2d-v2 & $3739 \pm 298$ & $3489 \pm 257$ & $2535 \pm 369$ & $3290 \pm 354$ & $28 \pm 1$ & $2775 \pm 301$ & $3059 \pm 209$ \\
\hline
Ant-v2 & $1863 \pm 271$ & $1780 \pm 257$ & $21 \pm 10$ & $1487 \pm 206$ & $-39 \pm 8$ & $1622 \pm 235$ & $1587 \pm 221$ \\
\hline
Humanoid-v2 & $965 \pm 76$ & $914 \pm 93$ & $725 \pm 112$ & $632 \pm 73$ & $107 \pm 15$ & $797 \pm 85$ & $820 \pm 91 $  \\
\hline
\end{tabular}
\label{summary_performance_table}
\end{adjustbox}
\end{table}

\subsection{Additional Results for Ablation Studies}
\begin{figure}[H]
\centering
    \includegraphics[width=0.32\linewidth]{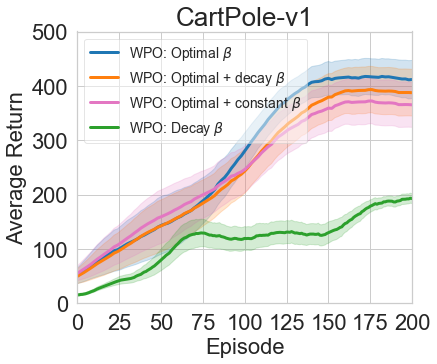}
    \includegraphics[width=0.32\linewidth]{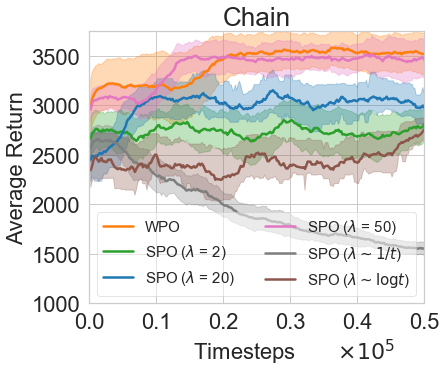}
    \includegraphics[width=0.32\linewidth]{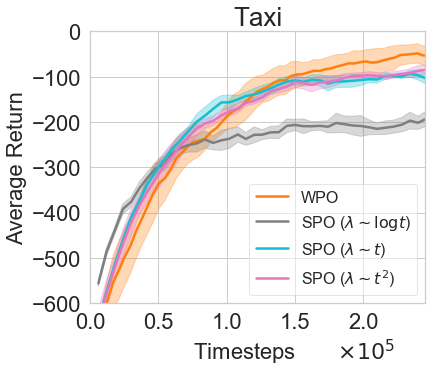}
\caption{Episode rewards during the training process for different $\beta$ and $\lambda$ settings, averaged across $5$ runs with a random initialization. The shaded area depicts the mean $\pm$ the standard deviation.}
\label{fig:additional_ablation}
\end{figure}

\subsection{Additional Comparison of Wasserstein and KL Trust Regions}

\begin{figure}[H]
\centering
\subfloat[$N_A = 100$]{
  \includegraphics[width=.22\linewidth]{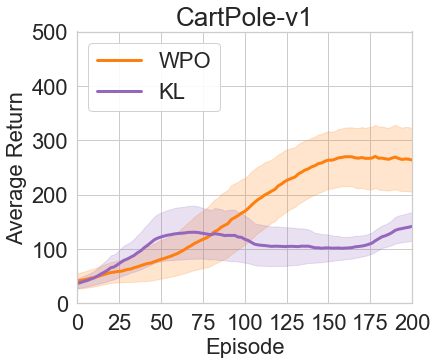}} 
\subfloat[$N_A = 500$]{
  \includegraphics[width=.22\linewidth]{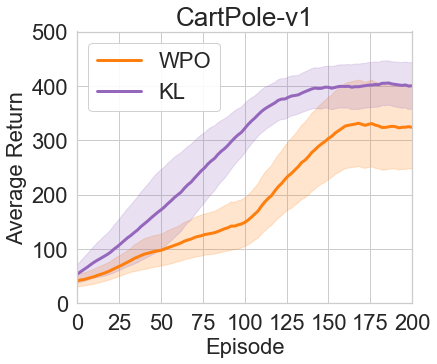}} 
\subfloat[$N_A = 100$]{
  \includegraphics[width=.22\linewidth]{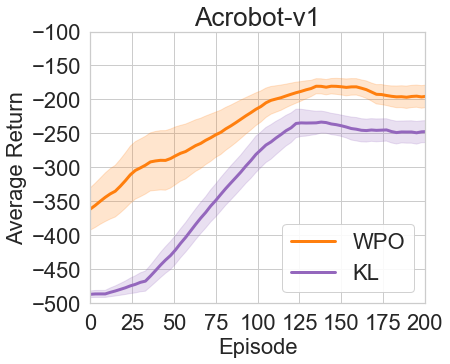}}
\subfloat[$N_A = 500$]{
  \includegraphics[width=.22\linewidth]{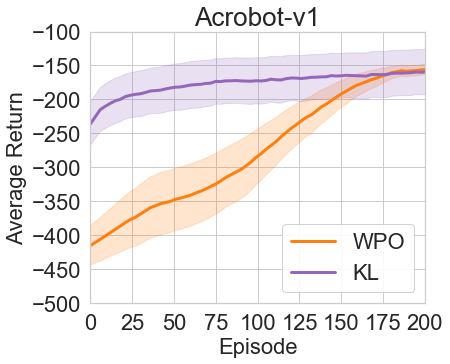}}
\caption{Episode rewards during the training process for the locomotion tasks, averaged across $5$ runs with a random initialization. The shaded area depicts the mean $\pm$ the standard deviation. \label{fig:locomotion_control_different_Ns}}
\end{figure}

\begin{table}[H]
\caption{\hlight{Average runtime (seconds) of WPO, SPO and KL}}
\renewcommand{\arraystretch}{1.3}
\begin{adjustbox}{width=0.6\columnwidth,center}
\begin{tabular}{llll}
 \hline
 & \hlight{WPO} & \hlight{SPO} & \hlight{KL} \\
 \hline 
 \hlight{Taxi-v3 (per $10^3$ steps)} & \hlight{$71.0 \pm 7.3$}  & \hlight{$69.5 \pm 8.7$} & \hlight{$74.3 \pm 9.5$}  \\
 \hline 
 \hlight{NChain-v0 (per $10^3$ steps)} & \hlight{$58.4 \pm 9.1$} & \hlight{$63.1 \pm 7.4$} & \hlight{$59.9 \pm 8.7$}  \\
 \hline 
 \hlight{CartPole-v1 (per $10^6$ steps)} & \hlight{$11.4 \pm 1.8$} & \hlight{$10.2 \pm 2.3$} & \hlight{$9.7 \pm 1.9$}  \\
 \hline 
 \hlight{Acrobot-v1 (per $10^5$ steps)} & \hlight{$10.4 \pm 1.9$} & \hlight{$9.7 \pm 2.5$} & \hlight{$10.9 \pm 2.3$} \\
 \hline 
 \hlight{Humanoid-v2 (per $10^5$ steps)} & \hlight{$422.7 \pm 65.4$} & \hlight{$409.1 \pm 46.5$} & \hlight{$438.5 \pm 61.2$} \\
 \hline
\end{tabular}
\end{adjustbox}
\label{runtime_wpo_kl}
\end{table}

\subsection{Additional Results for Large-scale Continuous Control}

\begin{figure}[H]
    \centering
    \includegraphics[width = 0.4\linewidth]{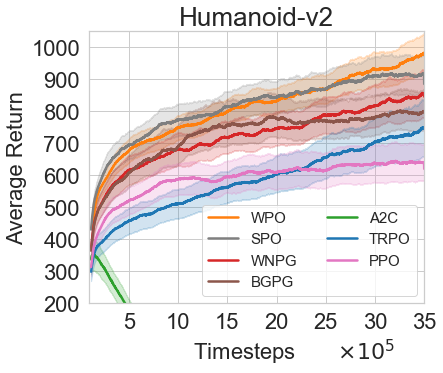}
    \caption{\hlight{Episode rewards during training for MuJuCo Humanoid task, averaged across 10 runs with random initialization. The shaded area depicts the mean $\pm$ the standard deviation.}}
    \label{fig:humanoid}
\end{figure}

\begin{table}[H]
\caption{\hlight{Average runtime (seconds per $10^5$ timesteps) for the MuJuCo continuous control tasks}}
\renewcommand{\arraystretch}{1.3}
\centering
\begin{adjustbox}{width=0.7\linewidth,center}
\begin{tabular}{llllllll}
\hline
\hlight{Environment}  & \hlight{WPO} & \hlight{SPO} & \hlight{TRPO} & \hlight{PPO} & \hlight{A2C} & \hlight{BGPG} & \hlight{WNPG} \\
\hline
\hlight{HalfCheetah-v2} & \hlight{$297 \pm 31$} & \hlight{$289 \pm 25$} & \hlight{$290 \pm 28$} & \hlight{$292 \pm 36$} & \hlight{$293 \pm 27$} & \hlight{$306 \pm 33$} & \hlight{$298 \pm 22$} \\
\hline
\hlight{Hopper-v2} & \hlight{$233 \pm 38$} & \hlight{$226 \pm 42$}  & \hlight{$242 \pm 56$} & \hlight{$167 \pm 36$} & \hlight{$254 \pm 49$} & \hlight{$201 \pm 32$} & \hlight{$197 \pm 31$} \\
\hline
\hlight{Walker2d-v2} & \hlight{$289 \pm 55$} & \hlight{$312 \pm 61$} & \hlight{$253 \pm 39$} &  \hlight{$307 \pm 52$} &  \hlight{$259 \pm 46$} & \hlight{$322 \pm 62$} & \hlight{$214 \pm 45$} \\
\hline
\hlight{Ant-v2} & \hlight{$307 \pm 51$} & \hlight{$290 \pm 57$} &  \hlight{$296 \pm 63$} & \hlight{$251 \pm 47$} & \hlight{$291 \pm 41$} & \hlight{$286 \pm 63$} & \hlight{$269 \pm 54$} \\
\hline
\hlight{Humanoid-v2} & \hlight{$423 \pm 65$} & \hlight{$401 \pm 47$} & \hlight{$446 \pm 52$} &  \hlight{$395 \pm 57$} & \hlight{$230 \pm 31$} & \hlight{$425 \pm 58$} & 
\hlight{$398 \pm 49$}  \\
\hline
\end{tabular}
\label{runtime_continuous_control}
\end{adjustbox}
\end{table}

\section{Proof of Theorem \ref{thm_opt_policy_wass}}
\label{appendix:wass}
\thmoptpolicywass*
\begin{proof}[Proof of Theorem \ref{thm_opt_policy_wass}]
First, we denote $Q^s$ as the joint distribution of $\pi(\cdot|s)$ and $\pi'(\cdot|s)$ with $\sum_{i=1}^N Q^s_{ij} = \pi(a_j|s)$ and $\sum_{j=1}^N Q^s_{ij} = \pi'(a_i|s)$. Also, let $f_s(i,j)$ represent the conditional distribution of $\pi'(a_i|s)$ under $\pi(a_j|s)$. Then $Q^s_{ij} = \pi(a_j|s)f_s(i,j)$, $\pi'(a_i|s) = \sum_{j=1}^N Q^s_{ij} = \sum_{j=1}^N \pi(a_j|s)f_s(i,j)$. In addition:  \begin{eqnarray*}
&d_\text{W} (\pi'(\cdot|s), \pi(\cdot|s)) & = \min_{Q^s_{ij}} \sum_{i=1}^N \sum_{j=1}^N D_{ij}Q^s_{ij} =  \min_{f_s(i,j)} \sum_{i=1}^N \sum_{j=1}^N D_{ij} \pi(a_j|s)f_s(i,j), \mbox{ and }\label{Was1}\\
&\mathbb{E}_{a \sim \pi'(\cdot|s)} [ A^{\pi} (s,a)] & = \sum_{i=1}^N   A^{\pi} (s,a_i) \pi'(a_i|s) = \sum_{i=1}^N \sum_{j=1}^N  A^{\pi} (s,a_i) \pi(a_j|s)f_s(i,j).\label{Was2}
\end{eqnarray*}
Thus, the WPO problem in (\ref{odrpo_problem}) can be reformulated as: 
\begin{subequations}
\begin{align}
 \max_{f_s(i,j) \ge 0} \hspace{3mm} \mathbb{E}_{s\sim \rho^{\pi}_\upsilon}  &\sum_{i=1}^N \sum_{j=1}^N  A^{\pi} (s,a_i) \pi(a_j|s)f_s(i,j) \label{objective_wass_primal} \\
 s.t.  \hspace{3mm} \mathbb{E}_{s\sim \rho^{\pi}_\upsilon}  &\sum_{i=1}^N \sum_{j=1}^N D_{ij} \pi(a_j|s)f_s(i,j)  \le \delta, \label{constraint1_wass_primal} \\
 &  \sum_{i=1}^N  f_s(i,j) = 1, \hspace{1cm} \forall s \in \mathcal{S}, j = 1 \dots N \label{constraint2_wass_primal}.
\end{align}
\label{wass_odrpo_problem}
\end{subequations}
Note here that (\ref{constraint1_wass_primal}) is equivalent to $\mathbb{E}_{s\sim \rho^{\pi}_\upsilon} \min_{f_s(i,j)} \sum_{i=1}^N \sum_{j=1}^N D_{ij} \pi(a_j|s)f_s(i,j)  \le \delta$ because if we have a feasible $f_s(i,j)$ to make (\ref{constraint1_wass_primal}) hold, we must have  $\mathbb{E}_{s\sim \rho^{\pi}_\upsilon} \min_{f_s(i,j)} \sum_{i=1}^N \sum_{j=1}^N D_{ij} \pi(a_j|s)f_s(i,j)  \le \delta$. 

Since both the objective function and the constraint are linear in $f_s(i,j)$, (\ref{wass_odrpo_problem}) is a convex optimization problem. Also, Slater's condition holds for (\ref{wass_odrpo_problem}) as the feasible region has an interior point, which is $f_s(i,i) = 1$  $\forall i$, and $f_s(i,j) = 0$  $\forall i \ne j$. Meanwhile, since $A^{\pi} (s,a)$ is bounded based on Assumption \ref{bounded_A}, the objective is bounded above. Therefore, strong duality holds for (\ref{wass_odrpo_problem}). At this point we can derive the dual problem of (\ref{wass_odrpo_problem}) as its equivalent reformulation:
\begin{equation}
\begin{split}
    & \min_{\beta \ge 0, \zeta^s_{j}} \hspace{3mm} \beta\delta + \int_{s \in \mathcal{S}} \sum_{j=1}^N \zeta^s_{j} ds\\
    & s.t. \hspace{3mm} A^{\pi} (s,a_i) \pi(a_j|s) - \beta D_{ij} \pi(a_j|s) - \frac{\zeta^s_{j}}{\rho^{\pi}_\upsilon(s)}  \le 0, \hspace{1cm} \forall s \in \mathcal{S}, i, j = 1 \dots N.
\label{wass_dual_formulation1_in_proof}
\end{split}
\end{equation}
We observe that with a fixed $\beta$, the optimal $\zeta^s_{j}$ will be achieved at: \begin{equation}\zeta^{s*}_{j}(\beta) = \max_{i = 1 \dots N} \rho^{\pi}_\upsilon (s)\pi(a_j|s)(A^{\pi} (s,a_i) -  \beta D_{ij}).\label{optimaleta}
\end{equation}
Denote $\beta^*$ as an optimal solution to (\ref{wass_dual_formulation1_in_proof}) \zcyan{and $f_s^*(i,j)$ as an optimal solution to (\ref{wass_odrpo_problem}}). 
Due to the complimentary slackness, the following equations hold:
\begin{equation*}
    (A^{\pi} (s,a_i) \pi(a_j|s) - \beta^* D_{ij} \pi(a_j|s) - \frac{\zeta^{s*}_{j}(\beta^*)}{\rho^{\pi}_\upsilon(s)}) f_s^*(i,j) = 0, \hspace{1cm} \forall s,i,j.
\label{wass_complimentary_slackness}
\end{equation*}
In this case, $f_s^*(i,j)$ can have non-zero values only when $A^{\pi} (s,a_i) \pi(a_j|s) - \beta^* D_{ij} \pi(a_j|s) - \frac{\zeta^{s*}_{j}(\beta^*)}{\rho^{\pi}_\upsilon(s)}  = 0$, which means $\zeta^{s*}_{j}(\beta^*) = \rho^{\pi}_\upsilon(s)\pi(a_j|s)(A^{\pi}(s,a_i) - \beta^* D_{ij})$. Given the expression of the optimal $\zeta^{s*}_{j}$ in (\ref{optimaleta}), $f_s^*(i,j)$ can have non-zero values only when \zcyan{$i \in \mathcal{K}^\pi_s(\beta^*,j)$}, where $\zcyan{\mathcal{K}^\pi_s(\beta,j)} =  \text{argmax}_{k = 1 \dots N} A^{\pi} (s,a_k) -  \beta D_{kj}$.   
                 
\ccred{
When there exists a unique optimizer, i.e., $|\mathcal{K}^\pi_s(\beta^*,j)| = 1$, let $\kappa^\pi_s(\beta^*,j)$ denote the optimizer. Since $\sum_{i=1}^N  f_s^*(i,j) = 1$ as indicated in (\ref{constraint2_wass_primal}), the only optimal solution is: 
\[ f_s^*(i,j) = \begin{cases*}
                    1 & \hspace{3mm} if  $i = \kappa^\pi_s(\beta^*,j)$,  \\
                    0 & \hspace{3mm} otherwise.
                 \end{cases*} \]% 
}

\ccred{
When there exists multiple optimizers, i.e., $|\mathcal{K}^\pi_s(\beta^*,j)| > 1$, the optimal weights $f_s^*(i,j)$ for $i \in \mathcal{K}^\pi_s(\beta^*,j)$ could be determined by solving the following linear programming: 
\begin{equation}
\begin{split}
 \max_{f_s^*(i,j) \ge 0, i \in \mathcal{K}^\pi_s(\beta^*,j)} \hspace{3mm} \mathbb{E}_{s\sim \rho^{\pi}_\upsilon}  & \sum_{j=1}^N  \pi(a_j|s) \sum_{i \in \mathcal{K}^\pi_s(\beta^*,j)} A^{\pi} (s,a_i) f_s^*(i,j)  \\
 s.t.  \hspace{3mm} \mathbb{E}_{s\sim \rho^{\pi}_\upsilon}  & \sum_{j=1}^N \pi(a_j|s) \sum_{i \in \mathcal{K}^\pi_s(\beta^*,j)} D_{ij}  f_s^*(i,j) \le \delta, \\
 &  \sum_{i \in \mathcal{K}^\pi_s(\beta^*,j)}  f_s^*(i,j) = 1, \hspace{1cm} \forall s \in \mathcal{S}, j = 1 \dots N .
\end{split}
\label{wass_odrpo_problem_simple}
\end{equation}
}

And then the corresponding optimal solution is, $\pi^*(a_i|s) = \sum_{j=1}^N \pi(a_j|s)f_s^*(i,j)$. 

Last, by substituting $\zeta^{s*}_{j}(\beta) =  \rho^{\pi}_\upsilon (s)\pi(a_j|s) \max_{i = 1 \dots N} (A^{\pi} (s,a_i) -  \beta D_{ij})$ into the dual problem (\ref{wass_dual_formulation1_in_proof}), we can reformulate (\ref{wass_dual_formulation1_in_proof}) into:
\begin{equation}
    \min_{\beta \ge 0} \{\beta\delta + \int_{s \in \mathcal{S}} \sum_{j=1}^N \zeta^{s*}_{j}(\beta) ds \} = \min_{\beta \ge 0} \{\beta\delta + \mathbb{E}_{s \sim \rho^{\pi}_\upsilon} \sum_{j=1}^N \pi(a_j|s) \max_{i = 1 \dots N} (A^{\pi} (s,a_i) -  \beta D_{ij}) \} \label{wass_dual_formulation2_in_proof}. 
\end{equation}
The optimal $\beta$ can then be obtained by solving (\ref{wass_dual_formulation2_in_proof}). 

\xred{
We will further show that $\beta^* \leq \bar{\beta}:=\max_{s \in \mathcal{S}, k, j = 1 \dots N, k \ne j} {(D_{kj})^{-1}}{(A^\pi (s,a_k) - A^\pi (s,a_j))}$. }

\xred{In the general case, i.e., $\beta \ge 0$, (\ref{objective_wass_primal}) is non-negative because: }

\xred{
\begin{subequations}
\begin{align}
\mathbb{E}_{s\sim \rho^{\pi}_\upsilon} & \sum_{i=1}^N \sum_{j=1}^N  A^{\pi} (s,a_i) \pi(a_j|s)f_s^*(i,j) \\
&= \mathbb{E}_{s\sim \rho^{\pi}_\upsilon} \sum_{j=1}^N \pi(a_j|s) \sum_{i=1}^N A^{\pi} (s,a_i) f_s^*(i,j) \\
&= \mathbb{E}_{s\sim \rho^{\pi}_\upsilon} \sum_{j=1}^N \pi(a_j|s) \sum_{i \in \mathcal{K}^\pi_s(\beta^*,j)} f_s^*(i,j) A^{\pi} (s,a_i) \\  
&\ge \mathbb{E}_{s\sim \rho^{\pi}_\upsilon} \sum_{j=1}^N \pi(a_j|s) \sum_{i \in \mathcal{K}^\pi_s(\beta^*,j)} f_s^*(i,j) [A^{\pi} (s,a_j) + \beta^* D_{ij}] 
\label{primal_objective_non_negative_proof} \\
& = \mathbb{E}_{s\sim \rho^{\pi}_\upsilon} \sum_{j=1}^N \pi(a_j|s) A^{\pi} (s,a_j) + \mathbb{E}_{s\sim \rho^{\pi}_\upsilon} \sum_{j=1}^N \pi(a_j|s) \sum_{i \in \mathcal{K}^\pi_s(\beta^*,j)} f_s^*(i,j) \beta^* D_{ij}
\\
& = \mathbb{E}_{s\sim \rho^{\pi}_\upsilon} \sum_{j=1}^N \pi(a_j|s) \beta^* \sum_{i \in \mathcal{K}^\pi_s(\beta^*,j)} f_s^*(i,j) D_{ij}
\\
&\ge 0,
\end{align}
\end{subequations}
where (\ref{primal_objective_non_negative_proof}) holds since for $i \in \mathcal{K}^\pi_s(\beta^*,j)$,
$A^{\pi} (s,a_i) - \beta^* D_{ij} \ge A^\pi(s,a_j) - \beta^* D_{jj} = A^\pi(s,a_{j})$. }
\xred{
When $\beta^* > \max_{s \in \mathcal{S}, k, j = 1 \dots N, k \ne j} \{\frac{A^\pi (s,a_k) - A^\pi (s,a_j)}{D_{kj}} \}$, we have that for all $s \in \mathcal{S}$, $\kappa^\pi_s(\beta^*,j) = j$. Thus, $f_s^*(i,i) = 1$, $\forall i$ and $f_s^*(i,j) = 0$, $\forall i \ne j$. The objective value (\ref{objective_wass_primal}) will be $0$ because $\mathbb{E}_{s\sim \rho^{\pi}_\upsilon} \sum_{i=1}^N \sum_{j=1}^N  A^{\pi} (s,a_i) \pi(a_j|s)f_s^*(i,j) = \mathbb{E}_{s\sim \rho^{\pi}_\upsilon} \sum_{i=1}^N  A^{\pi} (s,a_i) \pi(a_i|s) = 0$. The left hand side of (\ref{constraint1_wass_primal}) equals to $\mathbb{E}_{s\sim \rho^{\pi}_\upsilon} \sum_{i=1}^N \sum_{j=1}^N D_{ij} \pi(a_j|s)f_s^*(i,j) = \mathbb{E}_{s\sim \rho^{\pi}_\upsilon} \sum_{i=1}^N D_{ii} \pi(a_i|s) = 0$. Thus, for any $\delta > 0$, (\ref{constraint1_wass_primal}) is always satisfied.} 

\xred{
Since the objective of the primal Wasserstein trust-region constrained problem in (\ref{wass_dual_formulation}) constantly evaluates to $0$ when $\beta^* > \max_{s \in \mathcal{S}, k, j = 1 \dots N, k \ne j} \{\frac{A^\pi (s,a_k) - A^\pi (s,a_j)}{D_{kj}} \}$, and is non-negative when $\beta^* \leq \max_{s \in \mathcal{S}, k, j = 1 \dots N, k \ne j} \{\frac{A^\pi (s,a_k) - A^\pi (s,a_j)}{D_{kj}} \}$, we can use $\max_{s \in \mathcal{S}, k, j = 1 \dots N, k \ne j} \{\frac{A^\pi (s,a_k) - A^\pi (s,a_j)}{D_{kj}} \}$ as an upper bound for the optimal dual variable $\beta^*$.}

\end{proof}

\section{Optimal Beta for a Special Distance}
\label{appendix:specialdist}
%%%%%% Proposition 3 Proof %%%%%%

\begin{restatable}{prop}{propoptimalbeta}
\label{prop_optimal_beta_wass}
Let $k_s = \text{argmax}_{i=1,\dots,N}A^\pi(s,a_i)$, we have:

(1). If the initial point $\beta_0$ is in $[\max_{s, j} \{A^{\pi}(s,a_{k_s}) - A^{\pi} (s,a_j) \}, +\infty)$, the local optimal $\beta$ solution is $\max_{s, j} \{A^{\pi}(s,a_{k_s}) - A^{\pi} (s,a_j) \}$. 
\\
(2). If the initial point $\beta_0$ is in $[0, \min_{s, j \ne k_s} \{A^{\pi}(s,a_{k_s}) - A^{\pi} (s,a_j) \}]$: if $\delta - \int_{s\in \mathcal{S}} \rho^{\pi} (s)  (1 -  \pi(a_{k_s}|s)) ds < 0$, the local optimal $\beta$ is $\min_{s, j \ne k_s} \{A^{\pi}(s,a_{k_s}) - A^{\pi} (s,a_j) \}$; otherwise, the local optimal $\beta$ solution is $0$. 

(3). If the initial point $\beta_0$ is in $(\min_{s, j \ne k_s} \{A^{\pi}(s,a_{k_s}) - A^{\pi} (s,a_j) \}, \max_{s, j} \{A^{\pi}(s,a_{k_s}) - A^{\pi} (s,a_j) \})$, we construct sets $I_s^1$ and $I_s^2$ as: 

$\textbf{\textup{for}} \hspace{1mm} s \in \mathcal{S}, j \in \{1,2\dots N\} \hspace{1mm}: \hspace{1mm} \textbf{\textup{if}} \hspace{1mm} \beta_0 \ge A^{\pi}(s,a_{k_s}) - A^{\pi} (s,a_j) \hspace{1mm} \textbf{\textup{then}} \hspace{1mm} \textup{\text{Add}}\hspace{1mm} j \hspace{1mm} \textup{\text{to}} \hspace{1mm} I_s^1 \hspace{1mm} \textbf{\textup{else}} \hspace{1mm} \textup{\text{Add}} \hspace{1mm} j \hspace{1mm} \textup{\text{to}} \hspace{1mm} I_s^2$. Then, if $\delta - \mathbb{E}_{s \sim \rho^{\pi}}\sum_{j \in I_s^2} \pi(a_j|s) < 0$, the local optimal $\beta$ is $\min_{s \in \mathcal{S}, j \in I_s^2} \{ A^{\pi}(s,a_{k_s}) - A^{\pi} (s,a_j) \}$; otherwise, the local optimal $\beta$ is $\max_{s \in \mathcal{S}, j \in I_s^1} \{A^{\pi}(s,a_{k_s}) - A^{\pi} (s,a_j)\}$. \end{restatable}

\begin{proof}[Proof of Proposition \ref{prop_optimal_beta_wass}]

(1). When $\beta \in [\max_{s, j} \{A^{\pi}(s,a_{k_s}) - A^{\pi} (s,a_j) \}, +\infty)$, we have $A^{\pi} (s,a_j) \ge A^{\pi}(s,a_{k_s}) - \beta$ for all $s \in \mathcal{S}$, $j = 1 \dots N$. Since $A^{\pi}(s,a_{k_s}) - \beta \ge A^{\pi}(s,a_k) - \beta$ for all $k = 1 \dots N$, we have $A^{\pi} (s,a_j) \ge A^{\pi}(s,a_k) - \beta$  for all  $s \in \mathcal{S}$, $j = 1 \dots N$, $k = 1 \dots N$. Thus, $j \in \text{argmax}_{k = 1 \dots N} \{A^{\pi} (s,a_k) -  \beta D_{kj}\}$, for all $s \in \mathcal{S}$,  $j = 1 \dots N$. Therefore, (\ref{wass_dual_formulation}) can be reformulated as: 
\begin{equation*}
    \min_{\beta \ge 0} \{\beta\delta + \mathbb{E}_{s \sim \rho^{\pi}_\upsilon} \sum_{j=1}^N \pi(a_j|s) A^{\pi} (s,a_j) \}.
\end{equation*}
Since $\delta \ge 0$, we have the local optimal $\beta = \max_{s, j} \{A^{\pi}(s,a_{k_s}) - A^{\pi} (s,a_j) \}$.

(2). When $\beta \in [0, \min_{s, j \ne k_s} \{A^{\pi}(s,a_{k_s}) - A^{\pi} (s,a_j) \}]$, we have $A^{\pi} (s,a_j) \le A^{\pi}(s,a_{k_s}) - \beta$ for all $s \in \mathcal{S}$,  $j = 1 \dots N$, $j \ne k_s$. Thus $k_s \in \text{argmax}_{k = 1 \dots N} \{A^{\pi} (s,a_k) -  \beta D_{kj}\}$ for all $s \in \mathcal{S}$,  $j = 1 \dots N$. The inner part of (\ref{wass_dual_formulation}) then is: 
\begin{equation*}
\begin{gathered}
    \beta\delta + \mathbb{E}_{s \sim \rho^{\pi}_\upsilon} \{ \sum_{j=1, j \ne k_s}^N \pi(a_j|s) ( A^{\pi} (s,a_{k_s}) - \beta ) + \pi(a_{k_s}|s) A^{\pi} (s,a_{k_s}) \} \\
    = \beta (\delta - \mathbb{E}_{s \sim \rho^{\pi}_\upsilon}  \sum_{j=1, j \ne k_s}^N \pi(a_j|s))+ \mathbb{E}_{s \sim \rho^{\pi}_\upsilon}  \sum_{j=1}^N \pi(a_j|s)  A^{\pi} (s,a_{k_s}) \\
    = \beta (\delta - \int_{s\in \mathcal{S}} \rho^{\pi}_\upsilon (s)  (1 -  \pi(a_{k_s}|s)) ds ) + \mathbb{E}_{s \sim \rho^{\pi}_\upsilon}  \sum_{j=1}^N \pi(a_j|s) A^{\pi} (s,a_{k_s}).
\end{gathered} 
\end{equation*}
If $\delta - \int_{s\in \mathcal{S}} \rho^{\pi}_\upsilon (s)  (1 -  \pi(a_{k_s}|s)) ds < 0$, we have the local optimal $\beta = \min_{s, j \ne k_s} \{A^{\pi}(s,a_{k_s}) - A^{\pi} (s,a_j) \}$.
If $\delta - \int_{s\in \mathcal{S}} \rho^{\pi}_\upsilon (s)  (1 -  \pi(a_{k_s}|s)) ds  \ge 0$, we have the local optimal $\beta = 0$.

(3). For an initial point $\beta_0$ in $(\min_{s, j \ne k_s} \{A^{\pi}(s,a_{k_s}) - A^{\pi} (s,a_j) \}, \max_{s, j} \{A^{\pi}(s,a_{k_s}) - A^{\pi} (s,a_j) \})$, we construct partitions $I_s^1$ and $I_s^2$ of the set $\{1,2 \dots N\}$ in the way described in Proposition \ref{prop_optimal_beta_wass} for all $s \in \mathcal{S}$. Consider $\beta$ in the neighborhood of $\beta_0$, i.e., $\beta \ge A^{\pi}(s,a_{k_s}) - A^{\pi} (s,a_j)$ for $s \in \mathcal{S}, j \in I_s^1$ and $\beta \le A^{\pi}(s,a_{k_s}) - A^{\pi} (s,a_j)$ for $s \in \mathcal{S}, j \in I_s^2$. Then the inner part of (\ref{wass_dual_formulation}) can be reformulated as: 
\begin{equation*}
\begin{gathered}
    \beta\delta + \mathbb{E}_{s \sim \rho^{\pi}_\upsilon} \{ \sum_{j \in I_s^1} \pi(a_j|s)  A^{\pi} (s,a_{j})  + \sum_{j \in I_s^2} \pi(a_j|s) ( A^{\pi} (s,a_{k_s}) - \beta ) \} \\
    =  \beta (\delta - \mathbb{E}_{s \sim \rho^{\pi}_\upsilon}  \sum_{j \in I_s^2} \pi(a_j|s))+ \mathbb{E}_{s \sim \rho^{\pi}_\upsilon}  \{ \sum_{j \in I_s^1} \pi(a_j|s)  A^{\pi} (s,a_{j}) + \sum_{j \in I_s^2} \pi(a_j|s)  A^{\pi} (s,a_{k_s}) \}. \\
\end{gathered} 
\end{equation*}
If $\delta - \mathbb{E}_{s \sim \rho^{\pi}_\upsilon}\sum_{j \in I_s^2} \pi(a_j|s) < 0$, we have the local optimal $\beta = \min_{s \in \mathcal{S}, j \in I_s^2} \{ A^{\pi}(s,a_{k_s}) - A^{\pi} (s,a_j) \}$. If $\delta - \mathbb{E}_{s \sim \rho^{\pi}_\upsilon}\sum_{j \in I_s^2} \pi(a_j|s) \ge 0$, we have the local optimal $\beta = \max_{s \in \mathcal{S}, j \in I_s^1} \{A^{\pi}(s,a_{k_s}) - A^{\pi} (s,a_j)\}$. 
\end{proof}

\section{Proof of Theorem \ref{thm_opt_policy_sinkhorn}}
\label{appendix:sinkhorn}
\thmoptpolicysinkhorn*

\begin{proof}[Proof of Theorem \ref{thm_opt_policy_sinkhorn}]

Invoking the definition of Sinkhorn divergence in (\ref{def_sinkhorn}), the trust region constrained problem with Sinkhorn divergence can be reformulated as:
\begin{subequations}
\begin{align}
    \max_{Q} \hspace{5mm} & \mathbb{E}_{s \sim \rho_{\upsilon}^{\pi}} [{\sum}_{i=1}^N A^{\pi} (s,a_i) {\sum}_{j=1}^N Q^s_{ij}] \label{objective_sinkhorn_primal}\\
    s.t. \hspace{3mm} & \mathbb{E}_{s \sim \rho_{\upsilon}^{\pi}} [{\sum}_{i=1}^N {\sum}_{j=1}^N D_{ij}Q^s_{ij} + \frac{1}{\lambda} Q^s_{ij}\log Q^s_{ij} ] \le \delta \label{constraint1_sinkhorn_primal}\\
    & {\sum}_{i=1}^N Q^s_{ij} = \pi(a_j|s), \hspace{5mm} \forall j = 1,\dots, N, s \in \mathcal{S} \label{constraint2_sinkhorn_primal}.
\end{align}
\label{sinkhorn_odrpo_problem}
\end{subequations}

 Let $\beta$ and $\omega$ represent the dual variables of constraints (\ref{constraint1_sinkhorn_primal}) and (\ref{constraint2_sinkhorn_primal}) respectively, then the Lagrangian duality of (\ref{sinkhorn_odrpo_problem}) can be derived as:
\begin{subequations}
\begin{align}
 &\hspace{-0.2cm}   \max_{Q} \min_{\beta \ge 0, \omega} L(Q, \beta, \omega) =  \max_{Q}\min_{\beta \ge 0, \omega} \mathbb{E}_{s \sim \rho_{\upsilon}^{\pi}} [\sum_{i=1}^N A^{\pi} (s,a_i) \sum_{j=1}^N Q^s_{ij}]   \nonumber\\
        & + \int_{s \in \mathcal{S}} \sum_{j=1}^N \omega^s_j (\sum_{i=1}^N Q^s_{ij} - \pi(a_j|s)) ds  + \beta(\delta - \mathbb{E}_{s \sim \rho_{\upsilon}^{\pi}} [\sum_{i=1}^N \sum_{j=1}^N D_{ij}Q^s_{ij} + \frac{1}{\lambda} Q^s_{ij}\log Q^s_{ij} ])  \\
       \hspace{1cm}  =  &\ \ \max_{Q} \min_{\beta \ge 0, \omega}\mathbb{E}_{s \sim \rho_{\upsilon}^{\pi}} [\sum_{i=1}^N A^{\pi} (s,a_i) \sum_{j=1}^N Q^s_{ij}] + \int_{s \in \mathcal{S}} \sum_{j=1}^N \sum_{i=1}^N \frac{\omega^s_j}{\rho_{\upsilon}^{\pi}(s)}Q^s_{ij} \rho_{\upsilon}^{\pi}(s) ds  \nonumber\\
        & - \int_{s \in \mathcal{S}} \sum_{j=1}^N \omega^s_j \pi(a_j|s) ds + \beta\delta - \beta \mathbb{E}_{s \sim \rho_{\upsilon}^{\pi}} [\sum_{i=1}^N \sum_{j=1}^N D_{ij}Q^s_{ij} + \frac{1}{\lambda} Q^s_{ij}\log Q^s_{ij} ])  \\
      \hspace{1cm}  =  &\ \ \max_{Q} \min_{\beta \ge 0, \omega}\beta\delta  - \int_{s \in \mathcal{S}} \sum_{j=1}^N \omega^s_j \pi(a_j|s) ds \nonumber\\
      & + \mathbb{E}_{s \sim \rho_{\upsilon}^{\pi}} [\sum_{i=1}^N \sum_{j=1}^N (A^\pi(s,a_i) - \beta D_{ij} + \frac{\omega^s_j}{\rho_{\upsilon}^{\pi}(s)})Q^s_{ij} - \frac{\beta}{\lambda}Q^s_{ij}\log Q^s_{ij}]\\
     \hspace{1cm}  =  &\ \ \min_{\beta \ge 0, \omega}\max_{Q} \beta\delta  - \int_{s \in \mathcal{S}} \sum_{j=1}^N \omega^s_j \pi(a_j|s) ds \nonumber\\
      & + \mathbb{E}_{s \sim \rho_{\upsilon}^{\pi}} [\sum_{i=1}^N \sum_{j=1}^N (A^\pi(s,a_i) - \beta D_{ij} + \frac{\omega^s_j}{\rho_{\upsilon}^{\pi}(s)})Q^s_{ij} - \frac{\beta}{\lambda}Q^s_{ij}\log Q^s_{ij}] \label{minmax},
\end{align}
\label{sinkhorn_lagrangian}
\end{subequations}
where (\ref{minmax}) holds since the Lagrangian function $L(Q, \beta, \omega)$ is concave in $Q$ and linear in $\beta$ and $\omega$, and we can exchange the $\max$ and the $\min$ following the Minimax theorem \citep{sion1958general}. 

Note that the inner max problem of (\ref{minmax}) is an unconstrained concave problem, and we can obtain the optimal $Q$ by taking the derivatives and setting them to $0$. That is,
\begin{equation}
    \frac{\partial L}{\partial Q^s_{ij}} = A^\pi(s,a_i) - \beta D_{ij} + \frac{\omega^s_j}{\rho_{\upsilon}^{\pi}(s)} - \frac{\beta}{\lambda} (\log Q^s_{ij} + 1) = 0, \ \ \forall i, j = 1,\cdots,N, s \in \mathcal{S}. 
\label{sinkhorn_lagrangian_derivative}
\end{equation}
Therefore, we have the optimal $Q^{s*}_{ij}$ as:
\begin{align}
    Q^{s*}_{ij} = & \exp{(\frac{\lambda}{\beta} A^\pi (s,a_i) - \lambda D_{ij})} \exp{(\frac{\lambda \omega^s_j}{\beta \rho_{\upsilon}^{\pi}(s)} - 1)}, \ \ \forall i, j = 1,\cdots,N, s \in \mathcal{S}. \label{optP1}
\end{align}
In addition, since $\sum_{i=1}^N Q^{s*}_{ij} = \pi(a_j|s)$, we have the following hold: 
\begin{equation}
    \exp{(\frac{\lambda \omega^s_j}{\beta \rho_{\upsilon}^{\pi}(s)} - 1)} = \frac{\pi(a_j|s)}{\sum_{i=1}^N \exp{(\frac{\lambda}{\beta}A^\pi(s,a_i) - \lambda D_{ij})}}.
\label{sinkhorn_dual_omega}
\end{equation}
By substituting the left hand side of (\ref{sinkhorn_dual_omega}) into (\ref{optP1}), we can further reformulate the optimal $Q^{s*}_{ij}$ as:
\begin{equation}
    Q^{s*}_{ij} = \frac{\exp{(\frac{\lambda}{\beta}A^\pi(s,a_i) - \lambda D_{ij})}}{\sum_{k=1}^N \exp{(\frac{\lambda}{\beta}A^\pi(s,a_k) - \lambda D_{kj})}} \pi(a_j|s), \ \ \forall i, j = 1,\cdots,N, s \in \mathcal{S}.
\label{sinkhorn_pij_optimum}
\end{equation}
To obtain the optimal dual variables, based on (\ref{sinkhorn_dual_omega}), we have the optimal $\omega^*$ as:
\begin{equation}
   \omega^{s*}_j = \rho_{\upsilon}^{\pi}(s) \{\frac{\beta}{\lambda} + \frac{\beta}{\lambda} \ln (\pi(a_j|s)) - \frac{\beta}{\lambda} \ln [\sum_{i=1}^N \exp{(\frac{\lambda}{\beta}A^\pi(s,a_i) - \lambda D_{ij})}] \}, \ \ \forall j = 1,\cdots,N, s \in \mathcal{S}
\label{sinkhorn_dual_omega2}
\end{equation}
By substituting (\ref{sinkhorn_pij_optimum}) and (\ref{sinkhorn_dual_omega2}) into (\ref{minmax}), we can obtain the optimal $\beta^*$ via:
\begin{equation}
\begin{split}
    \min_{\beta \ge 0} \hspace{3mm} & \beta \delta - \mathbb{E}_{s \sim \rho^{\pi}_\upsilon} \sum_{j=1}^N \pi(a_j|s) \{\frac{\beta}{\lambda} + \frac{\beta}{\lambda} \ln (\pi(a_j|s)) - \frac{\beta}{\lambda} \ln [\sum_{i=1}^N \exp{(\frac{\lambda}{\beta}A^\pi(s,a_i) - \lambda D_{ij})}]\} \\
    & + \mathbb{E}_{s \sim \rho^{\pi}_\upsilon} \sum_{i=1}^N \sum_{j=1}^N \frac{\beta}{\lambda} \frac{\exp{(\frac{\lambda}{\beta}A^\pi(s,a_i) - \lambda D_{ij})} \cdot \pi(a_j|s)}{\sum_{k=1}^N \exp{(\frac{\lambda}{\beta}A^\pi(s,a_k) - \lambda D_{kj})}}.
\nonumber
\end{split}
\end{equation}
The proof for the upper bound of sinkhorn optimal $\beta$ can be found in Appendix \ref{UB_sinkhorn}.
\end{proof}

\section{Upper bound of Sinkhorn Optimal Beta}\label{UB_sinkhorn}
\xred{In this section, we will derive the upper bound of Sinkhorn optimal $\beta$. First, for a given $\beta$, the optimal $Q_{ij}^{s*}(\beta)$ to the Lagrangian dual $L(Q,\beta,\omega)$ can be expressed in (\ref{sinkhorn_pij_optimum}). With this, we will present the following two lemmas:}

\begin{restatable}{lemma}{lemmasinkhornobjectivedecreasewithbeta}
The objective function (\ref{objective_sinkhorn_primal}) with respect to $Q_{ij}^{s*}(\beta)$ decreases as the dual variable $\beta$ increases.
\label{lemma_sinkhorn_objective_decrease_with_beta}
\end{restatable}

\begin{restatable}{lemma}{lemmasinkhornhasboundedfeasiblebeta}
If Assumption \ref{bounded_A} holds, then for every $\delta > 0$, $Q_{ij}^{s*}(\frac{2A^{\mbox {\tiny max}}}{\delta})$ is feasible to (\ref{constraint1_sinkhorn_primal}) for any $\lambda$. 
\label{lemma_sinkhorn_has_bounded_feasible_beta}
\end{restatable}

We provide proofs for Lemma \ref{lemma_sinkhorn_objective_decrease_with_beta} and Lemma \ref{lemma_sinkhorn_has_bounded_feasible_beta} in Appendix \ref{appendix:lemma1} and Appendix \ref{appendix:lemma2} respectively. Given the above two lemmas, we are able to prove the following proposition \xred{on the upper bound of Sinkhorn optimal $\beta$:}

\begin{restatable}{prop}{propoptsinkhornbetabounded}
If $\beta_\lambda^*$ is the optimal dual solution to the Sinkhorn dual formulation (\ref{sinkhorn_dual_formulation}), then $\beta_\lambda^* \leq \frac{2A^{\mbox {\tiny max}}}{\delta}$ for any $\lambda$.
\label{prop_opt_sinkhorn_beta_bounded}
\end{restatable}

\begin{proof}[Proof of Proposition \ref{prop_opt_sinkhorn_beta_bounded}]
We will prove it by contradiction. According to Lemma \ref{lemma_sinkhorn_has_bounded_feasible_beta}, $Q_{ij}^{s*}(\frac{2A^{\mbox {\tiny max}}}{\delta})$ is feasible to (\ref{constraint1_sinkhorn_primal}). Since $\beta_{\lambda}^*$ is the optimal dual solution, $Q_{ij}^{s*}(\beta_{\lambda}^*)$ is optimal to (\ref{sinkhorn_odrpo_problem}). If $\beta_\lambda^* > \frac{2A^{\mbox {\tiny max}}}{\delta}$, according to Lemma \ref{lemma_sinkhorn_objective_decrease_with_beta}, the objective value in (\ref{objective_sinkhorn_primal}) with respect to $\frac{2A^{\mbox {\tiny max}}}{\delta}$ is smaller than the objective value in (\ref{objective_sinkhorn_primal}) with respect to $\beta_\lambda^*$, which contradicts the fact that $Q_{ij}^{s*}(\beta_{\lambda}^*)$ is the optimal solution to (\ref{sinkhorn_odrpo_problem}).
\end{proof}

% Similarly, we prove that the optimal $\beta$ value for the Wasserstein metric based policy update is bounded:

% \begin{restatable}{prop}{propoptwassbetabounded}
% If $\beta^*$ is the optimal dual solution to the Wasserstein dual formulation (\ref{wass_dual_formulation}), then $\beta^* \leq \bar{\beta} := \max_{s \in \mathcal{S}, k, j = 1 \dots N, k \ne j} \{\frac{A^\pi (s,a_k) - A^\pi (s,a_j)}{D_{kj}} \}$.
% \label{prop_opt_wass_beta_bounded}
% \end{restatable}

\subsection{Proof of Lemma \ref{lemma_sinkhorn_objective_decrease_with_beta}}\label{appendix:lemma1}
%%%%%% Lemma 1 Proof %%%%%%
\lemmasinkhornobjectivedecreasewithbeta*

\begin{proof}[Proof of Lemma \ref{lemma_sinkhorn_objective_decrease_with_beta}]
Let $G_{\lambda}(\beta)$ represent the objective function (\ref{objective_sinkhorn_primal}). By substituting the optimal $Q_{ij}^{s*}$ in (\ref{sinkhorn_pij_optimum}) into (\ref{objective_sinkhorn_primal}), we have: 
\begin{subequations}
\begin{align}
G_{\lambda}(\beta) &= \mathbb{E}_{s \sim \rho_{\upsilon}^{\pi}} [\sum_{i=1}^N A^{\pi} (s,a_i) \sum_{j=1}^N \frac{\exp{(\frac{\lambda}{\beta}A^\pi(s,a_i) - \lambda D_{ij})}}{\sum_{k=1}^N \exp{(\frac{\lambda}{\beta}A^\pi(s,a_k) - \lambda D_{kj})}} \pi(a_j|s)] \\
& = \mathbb{E}_{s \sim \rho_{\upsilon}^{\pi}} [\sum_{j=1}^N \pi(a_j|s) \sum_{i=1}^N  A^{\pi} (s,a_i)  \frac{\exp{(\frac{\lambda}{\beta}A^\pi(s,a_i) - \lambda D_{ij})}}{\sum_{k=1}^N \exp{(\frac{\lambda}{\beta}A^\pi(s,a_k) - \lambda D_{kj})}}]. 
\end{align}
\end{subequations}
For any $\beta_2 > \beta_1 > 0$, we have:
\begin{subequations}
\begin{align}
& G_{\lambda}(\beta_1) - G_{\lambda}(\beta_2) \nonumber\\
& = \mathbb{E}_{s \sim \rho_{\upsilon}^{\pi}} \sum_{j=1}^N \pi(a_j|s) \sum_{i=1}^N  A^{\pi} (s,a_i)  \{\frac{\exp{(\frac{\lambda}{\beta_1}A^\pi(s,a_i) - \lambda D_{ij})}}{\sum_{k=1}^N \exp{(\frac{\lambda}{\beta_1}A^\pi(s,a_k) - \lambda D_{kj})}} \nonumber\\
& - \frac{\exp{(\frac{\lambda}{\beta_2}A^\pi(s,a_i) - \lambda D_{ij})}}{\sum_{k=1}^N \exp{(\frac{\lambda}{\beta_2}A^\pi(s,a_k) - \lambda D_{kj})}}\} \\
& = \mathbb{E}_{s \sim \rho_{\upsilon}^{\pi}} \sum_{j=1}^N \pi(a_j|s) \sum_{i=1}^N  A^{\pi} (s,a_{[i]})  \{\frac{\exp{(\frac{\lambda}{\beta_1}A^\pi(s,a_{[i]}) - \lambda D_{[i]j})}}{\sum_{k=1}^N \exp{(\frac{\lambda}{\beta_1}A^\pi(s,a_{[k]}) - \lambda D_{[k]j})}} \nonumber\\
&- \frac{\exp{(\frac{\lambda}{\beta_2}A^\pi(s,a_{[i]}) - \lambda D_{[i]j})}}{\sum_{k=1}^N \exp{(\frac{\lambda}{\beta_2}A^\pi(s,a_{[k]}) - \lambda D_{[k]j})}}\},
\end{align}
\end{subequations}
where $[i]$ denotes sorted indices that satisfy $A^\pi(s,a_{[1]}) \ge A^\pi(s,a_{[2]}) \ge \dots \ge A^\pi(s,a_{[N]})$. Let 
\begin{subequations}
\begin{align}
f_s(i) &= \frac{\exp{(\frac{\lambda}{\beta_1}A^\pi(s,a_{[i]}) - \lambda D_{[i]j})}}{\sum_{k=1}^N \exp{(\frac{\lambda}{\beta_1}A^\pi(s,a_{[k]}) - \lambda D_{[k]j})}} - \frac{\exp{(\frac{\lambda}{\beta_2}A^\pi(s,a_{[i]}) - \lambda D_{[i]j})}}{\sum_{k=1}^N \exp{(\frac{\lambda}{\beta_2}A^\pi(s,a_{[k]}) - \lambda D_{[k]j})}} \\
&= \frac{\exp{((\frac{\lambda}{\beta_1}-\frac{\lambda}{\beta_2})A^\pi(s,a_{[i]}))}\exp{(\frac{\lambda}{\beta_2}A^\pi(s,a_{[i]}) - \lambda D_{[i]j})}}{\sum_{k=1}^N \exp{((\frac{\lambda}{\beta_1}-\frac{\lambda}{\beta_2})A^\pi(s,a_{[k]}))}  \exp{(\frac{\lambda}{\beta_2}A^\pi(s,a_{[k]}) - \lambda D_{[k]j})}} \nonumber\\
&- \frac{\exp{(\frac{\lambda}{\beta_2}A^\pi(s,a_{[i]}) - \lambda D_{[i]j})}}{\sum_{k=1}^N \exp{(\frac{\lambda}{\beta_2}A^\pi(s,a_{[k]}) - \lambda D_{[k]j})}} \label{diff}.
\end{align}
\end{subequations}
For notation brevity, we let $m_s(i) = \exp{((\frac{\lambda}{\beta_1}-\frac{\lambda}{\beta_2})A^\pi(s,a_{[i]}))} > 0$, $w_s(i) = \exp{(\frac{\lambda}{\beta_2}A^\pi(s,a_{[i]}) - \lambda D_{[i]j})} > 0$ and $q_s(i) = \frac{1}{\sum_{k=1}^N m_s(k) w_s(k)} - \frac{1}{\sum_{k=1}^N m_s(i) w_s(k)}$. Then we have 
\begin{subequations}
\begin{align}
(\ref{diff}) & = \frac{m_s(i) w_s(i)}{\sum_{k=1}^N m_s(k) w_s(k)} - \frac{w_s(i)}{\sum_{k=1}^N w_s(k)} \\
& = m_s(i) w_s(i) ( \frac{1}{\sum_{k=1}^N m_s(k) w_s(k)} - \frac{1}{\sum_{k=1}^N m_s(i) w_s(k)} ) \\
& = m_s(i) w_s(i) q_s(i).
\end{align}
\end{subequations}
Since $\frac{\lambda}{\beta_1}-\frac{\lambda}{\beta_2} > 0$, $m_s(i)$ decreases as $i$ increases. Thus, $q_s(i)$ decreases as $i$ increases. Since $m_s(1) \ge m_s(k)$ and $m_s(N) \le m_s(k)$ for all $k = 1, \dots, N$, we have $q_s(1) = \frac{1}{\sum_{k=1}^N m_s(k) w_s(k)} - \frac{1}{\sum_{k=1}^N m_s(1) w_s(k)} \ge \frac{1}{\sum_{k=1}^N m_s(k) w_s(k)} - \frac{1}{\sum_{k=1}^N m_s(k) w_s(k)} = 0$, and $q_s(N) = \frac{1}{\sum_{k=1}^N m_s(k) w_s(k)} - \frac{1}{\sum_{k=1}^N m_s(N) w_s(k)} \le \frac{1}{\sum_{k=1}^N m_s(k) w_s(k)} - \frac{1}{\sum_{k=1}^N m_s(k) w_s(k)} = 0$. Since $q_s(1) \ge 0$, $q_s(N) \le 0$ and $q_s(i)$ decreases as $i$ increases, there exists an index $1 \le k_s \le N$ such that $q_s(i) \ge 0$ for $i \le k_s$ and $q_s(i) < 0$ for $i > k_s$. Since $m_s(i), w_s(i) > 0$, we have $f_s(i) \ge 0$ for $i \le k_s$ and $f_s(i) < 0$ for $i > k_s$. In addition, we have $\sum_{i=1}^N f_s(i) = 0$ directly follows from the definition. Thus, $\sum_{i=1}^N f_s(i) = \sum_{i=1}^{k_s}|f_s(i)| - \sum_{i=k_s+1}^{N}|f_s(i)| = 0$. Therefore, 
\begin{subequations}
\begin{align}
G_{\lambda}(\beta_1) - G_{\lambda}(\beta_2) &= \mathbb{E}_{s \sim \rho_{\upsilon}^{\pi}} \sum_{j=1}^N \pi(a_j|s) \sum_{i=1}^N  A^{\pi} (s,a_{[i]}) f_s(i) \\
& = \mathbb{E}_{s \sim \rho_{\upsilon}^{\pi}} \sum_{j=1}^N \pi(a_j|s) \{\sum_{i=1}^{k_s}  A^{\pi} (s,a_{[i]}) |f_s(i)| -  \sum_{i=k_s+1}^N  A^{\pi} (s,a_{[i]}) |f_s(i)|\} \\
& \ge \mathbb{E}_{s \sim \rho_{\upsilon}^{\pi}} \sum_{j=1}^N \pi(a_j|s) \{\sum_{i=1}^{k_s}  A^{\pi} (s,a_{[k_s]}) |f_s(i)| -  \sum_{i=k_s+1}^N  A^{\pi} (s,a_{[k_s+1]}) |f_s(i)|\} \label{Adecrease} \\
& = \mathbb{E}_{s \sim \rho_{\upsilon}^{\pi}} \sum_{j=1}^N \pi(a_j|s) \{A^{\pi} (s,a_{[k_s]}) \sum_{i=1}^{k_s} |f_s(i)| -  A^{\pi} (s,a_{[k_s+1]}) \sum_{i=k_s+1}^N |f_s(i)|\} \\
& = \mathbb{E}_{s \sim \rho_{\upsilon}^{\pi}} \sum_{j=1}^N \pi(a_j|s) \{A^{\pi} (s,a_{[k_s]}) \sum_{i=1}^{k_s} |f_s(i)| -  A^{\pi} (s,a_{[k_s+1]}) \sum_{i=1}^{k_s}|f_s(i)|\} \\
& = \mathbb{E}_{s \sim \rho_{\upsilon}^{\pi}} \sum_{j=1}^N \pi(a_j|s) (A^{\pi} (s,a_{[k_s]}) -  A^{\pi} (s,a_{[k_s+1]})) \sum_{i=1}^{k_s}|f_s(i)| \\
& \ge 0 \label{Adecrease2}.
\end{align}
\end{subequations}
where (\ref{Adecrease}) and (\ref{Adecrease2}) hold since $A^{\pi} (s,a_{[i]})$ is non-increasing as $i$ increases. Furthermore, at least one inequality of (\ref{Adecrease}) and (\ref{Adecrease2}) will not hold at \zcyan{equality} since $\sum_{i=1}^N \pi(a_i|s) A^\pi(s,a_i) = 0$, $\forall s \in  \mathcal{S}$, and for non-trivial cases, $Pr\{A^{\pi} (s,a) = 0, \forall s \in \mathcal{S}, \forall a \in \mathcal{A}\} < 1$, which means $Pr\{\exists s_1, s_2 \in \mathcal{S}, a_1, a_2 \in \mathcal{A}, \ s.t. \ A^{\pi} (s_1,a_1) \neq  A^{\pi} (s_2,a_2)\} > 0$.  Therefore, we have $G_{\lambda}(\beta_1) - G_{\lambda}(\beta_2) > 0$. 
\end{proof}

\subsection{Proof of Lemma \ref{lemma_sinkhorn_has_bounded_feasible_beta}}\label{appendix:lemma2}
%%%%%% Lemma 2 Proof %%%%%%
\lemmasinkhornhasboundedfeasiblebeta*

\begin{proof}[Proof of Lemma \ref{lemma_sinkhorn_has_bounded_feasible_beta}]
By substituting the optimal $Q_{ij}^{s*}$ in (\ref{sinkhorn_pij_optimum}) into (\ref{constraint1_sinkhorn_primal}), we can reformulate the left hand side of (\ref{constraint1_sinkhorn_primal}) as follows:
\begin{subequations}
\begin{align}
& \mathbb{E}_{s \sim \rho_{\upsilon}^{\pi}} [\sum_{i=1}^N \sum_{j=1}^N D_{ij}Q^{s*}_{ij} + \frac{1}{\lambda} Q^{s*}_{ij}\log Q^{s*}_{ij}] \\
& =  \mathbb{E}_{s \sim \rho_{\upsilon}^{\pi}} \{\sum_{i=1}^N \sum_{j=1}^N D_{ij}Q^{s*}_{ij} + \frac{1}{\lambda} Q^{s*}_{ij}[\frac{\lambda}{\beta}A^\pi(s,a_i) - \lambda D_{ij} + \log{\frac{\pi(a_j|s)}{\sum_{k=1}^N \exp{(\frac{\lambda}{\beta}A^\pi(s,a_k) - \lambda D_{kj})}}}] \} \\
& = \mathbb{E}_{s \sim \rho_{\upsilon}^{\pi}} \{ \sum_{i=1}^N \sum_{j=1}^N \frac{1}{\beta}Q^{s*}_{ij} A^\pi(s,a_i) + \frac{1}{\lambda}Q^{s*}_{ij}\log{\frac{\pi(a_j|s)}{\sum_{k=1}^N \exp{(\frac{\lambda}{\beta}A^\pi(s,a_k) - \lambda D_{kj})}}}
\}.
\end{align}
\end{subequations}
Now we prove that when $\beta = \frac{2A^{\mbox {\tiny max}}}{\delta}$, $\mathbb{E}_{s \sim \rho_{\upsilon}^{\pi}} \{ \sum_{i=1}^N \sum_{j=1}^N \frac{1}{\beta}Q^{s*}_{ij}(\beta) A^\pi(s,a_i)\} \le \frac{\delta}{2}$ and $\mathbb{E}_{s \sim \rho_{\upsilon}^{\pi}} \{\frac{1}{\lambda}Q^{s*}_{ij}(\beta)\log{\frac{\pi(a_j|s)}{\sum_{k=1}^N \exp{(\frac{\lambda}{\beta}A^\pi(s,a_k) - \lambda D_{kj})}}}\} \le \frac{\delta}{2}$ hold. For the first part, we have: 
\begin{subequations}
\begin{align}
&\mathbb{E}_{s \sim \rho_{\upsilon}^{\pi}} \{ \sum_{i=1}^N \sum_{j=1}^N \frac{1}{\beta}Q^{s*}_{ij} A^\pi(s,a_i) \} \\
& = \frac{1}{\beta} \mathbb{E}_{s \sim \rho_{\upsilon}^{\pi}} \{ \sum_{i=1}^N [\sum_{j=1}^N Q^{s*}_{ij}] A^\pi(s,a_i) \} \\
& = \frac{1}{\beta} \mathbb{E}_{s \sim \rho_{\upsilon}^{\pi}} \{ \sum_{i=1}^N \zcyan{\pi'(a_i|s)} A^\pi(s,a_i) \}  \\ 
& \le \frac{1}{\beta} \mathbb{E}_{s \sim \rho_{\upsilon}^{\pi}} \{ \sum_{i=1}^N \zcyan{\pi'(a_i|s)} |A^\pi(s,a_i)| \}  \\
& \le \frac{A^{\mbox \tiny max}}{\beta} = \frac{\delta}{2}.
\end{align}
\end{subequations}
For the second part, the followings hold:
\begin{subequations}
\begin{align}
& \mathbb{E}_{s \sim \rho_{\upsilon}^{\pi}} \{ \sum_{i=1}^N \sum_{j=1}^N  \frac{1}{\lambda} Q^{s*}_{ij}\log{\frac{\pi(a_j|s)}{\sum_{k=1}^N \exp{(\frac{\lambda}{\beta}A^\pi(s,a_k) - \lambda D_{kj})}}}
\} \\
& = \mathbb{E}_{s \sim \rho_{\upsilon}^{\pi}} \{
\sum_{j=1}^N  \frac{1}{\lambda} (\sum_{i=1}^N Q^{s*}_{ij}) \log{\frac{\pi(a_j|s)}{\sum_{k=1}^N \exp{(\frac{\lambda}{\beta}A^\pi(s,a_k) - \lambda D_{kj})}}}
\}  \\
& = \frac{1}{\lambda} \mathbb{E}_{s \sim \rho_{\upsilon}^{\pi}} \{
\sum_{j=1}^N  \pi(a_j|s) \log{\frac{\pi(a_j|s)}{\sum_{k=1}^N \exp{(\frac{\lambda}{\beta}A^\pi(s,a_k) - \lambda D_{kj})}}}
\}  \\
& \leq \frac{1}{\lambda} \mathbb{E}_{s \sim \rho_{\upsilon}^{\pi}} \{
\sum_{j=1}^N  \pi(a_j|s) \log{\frac{\pi(a_j|s)}{ \exp{(\frac{\lambda}{\beta}A^\pi(s,a_j))}}}
\}  \\
& \leq \frac{1}{\lambda} \mathbb{E}_{s \sim \rho_{\upsilon}^{\pi}} \{
\sum_{j=1}^N  \pi(a_j|s) \log{\frac{1}{ \exp{(\frac{\lambda}{\beta}A^\pi(s,a_j))}}}
\}  \\
& = \frac{1}{\lambda} \mathbb{E}_{s \sim \rho_{\upsilon}^{\pi}} \{ \sum_{j=1}^N  \pi(a_j|s) (-\frac{\lambda}{\beta} A^\pi(s,a_j)) \}  \\
& \leq \frac{1}{\beta} \mathbb{E}_{s \sim \rho_{\upsilon}^{\pi}} \{ \sum_{j=1}^N \pi(a_j|s) |A^\pi(s,a_j)| \} \\
& \leq \frac{A^{\mbox \tiny max}}{\beta} = \frac{\delta}{2}.
\end{align} 
\end{subequations}

Therefore, $Q_{ij}^{s*}(\frac{2A^{\mbox {\tiny max}}}{\delta})$ is feasible to (\ref{constraint1_sinkhorn_primal}) for any $\lambda$.
\end{proof}

\section{Gradient of the Objective in the Sinkhorn Dual Formulation}\label{appendix:gradient}
The closed-form gradient of the objective in the Sinkhorn dual formulation (\ref{sinkhorn_dual_formulation}) is as follows:
\small
\begin{equation*}
\begin{split}
    &\delta - \mathbb{E}_{s \sim \rho^{\pi}_\upsilon}\sum_{j=1}^N \pi(a_j|s) \Big\{\frac{1}{\lambda} + \frac{1}{\lambda}\ln(\pi(a_j|s)) - \frac{1}{\lambda} \ln [\sum_{i=1}^N \exp{(\frac{\lambda}{\beta}A^\pi(s,a_i) - \lambda D_{ij})}] \\
    &- \frac{\beta}{\lambda} \cdot \frac{1}{\sum_{i=1}^N \exp{(\frac{\lambda}{\beta}A^\pi(s,a_i) - \lambda D_{ij})}} \times \sum_{i=1}^N [\exp{(\frac{\lambda}{\beta}A^\pi(s,a_i) - \lambda D_{ij})} \times -\lambda A^\pi(s,a_i) \beta^{-2}] \Big\} \\
    &+ \mathbb{E}_{s \sim \rho^{\pi}_\upsilon}\sum_{i=1}^N\sum_{j=1}^N \Big\{\frac{\pi(a_j|s)}{\lambda}\frac{\exp{(\frac{\lambda}{\beta}A^\pi(s,a_i) - \lambda D_{ij})}}{\sum_{k=1}^N \exp{(\frac{\lambda}{\beta}A^\pi(s,a_k) - \lambda D_{kj})}} \\
    &+ \frac{\beta \pi(a_j|s)}{\lambda} \cdot \frac{\exp{(\frac{\lambda}{\beta}A^\pi(s,a_i) - \lambda D_{ij})} \times -\lambda A^\pi(s,a_i) \beta^{-2} \times \sum_{k=1}^N \exp{(\frac{\lambda}{\beta}A^\pi(s,a_k) - \lambda D_{kj})}}{(\sum_{k=1}^N \exp{(\frac{\lambda}{\beta}A^\pi(s,a_k) - \lambda D_{kj})})^2} \\
    &- \frac{\beta \pi(a_j|s)}{\lambda} \cdot \frac{\exp{(\frac{\lambda}{\beta}A^\pi(s,a_i) - \lambda D_{ij})} \times \sum_{k=1}^N [\exp{(\frac{\lambda}{\beta}A^\pi(s,a_k) - \lambda D_{kj})} \times -\lambda A^\pi(s,a_k) \beta^{-2}] }{(\sum_{k=1}^N \exp{(\frac{\lambda}{\beta}A^\pi(s,a_k) - \lambda D_{kj})})^2} \Big\}.
\end{split}
\end{equation*}
\normalsize

\section{Proof of Theorem \ref{thm_opt_beta_sinkhorn_wass_relationship}}
\label{appendix:optbetaconvergence}
%%%%%% Theorem 3 Proof %%%%%%
\xred{
Given the upper bound of Wassertein optimal $\beta$ in Theorem \ref{thm_opt_policy_wass} and the upper bound of Sinkhorn optimal $\beta$ in Proposition \ref{prop_opt_sinkhorn_beta_bounded}, we are able to derive the following theorem: }

\thmoptbetasinkhornwassrelationship*
\begin{proof}[Proof of Theorem \ref{thm_opt_beta_sinkhorn_wass_relationship}]
To show that $F_\lambda(\beta)$ converges to $F(\beta)$ uniformly on $[0, \beta_{\text{UB}}]$, it is equivalent to show that $\lim_{\lambda \xrightarrow{} \infty} \sup_{0 \le \beta \le \beta_{\text{UB}}} \Big| F_\lambda(\beta) - F(\beta) \Big| = 0$. Let $\zcyan{\epsilon_s^\pi(\beta,i,j)} = \max_{k = 1 \dots N} (A^{\pi} (s,a_k) -  \beta D_{kj}) - [A^{\pi} (s,a_i) -  \beta D_{ij}]$, and $\zcyan{\epsilon_s^\pi(\beta,i,j)} \geq 0$.
First, we have
\begin{subequations}
    \begin{align}
    & \Big| F_\lambda(\beta) - F(\beta) \Big| \notag \\
    & = \Big| \beta \delta - \mathbb{E}_{s \sim \rho^{\pi}_\upsilon} \sum_{j=1}^N \pi(a_j|s) \{\frac{\beta}{\lambda} + \frac{\beta}{\lambda} \ln (\pi(a_j|s)) - \frac{\beta}{\lambda} \ln [\sum_{i=1}^N \exp{(\frac{\lambda}{\beta}A^\pi(s,a_i) - \lambda D_{ij})}]\} \notag \\
    & + \mathbb{E}_{s \sim \rho^{\pi}_\upsilon} \sum_{i=1}^N \sum_{j=1}^N \frac{\beta}{\lambda} \frac{\exp{(\frac{\lambda}{\beta}A^\pi(s,a_i) - \lambda D_{ij})} \cdot \pi(a_j|s)}{\sum_{k=1}^N \exp{(\frac{\lambda}{\beta}A^\pi(s,a_k) - \lambda D_{kj})}} - \beta\delta \notag \\
    & - \mathbb{E}_{s \sim \rho^{\pi}_\upsilon} \sum_{j=1}^N \pi(a_j|s) \max_{i = 1 \dots N} (A^{\pi} (s,a_i) -  \beta D_{ij}) \Big| \\
    & \leq  \Big|\frac{\beta}{\lambda}\mathbb{E}_{s \sim \rho^{\pi}_\upsilon} \sum_{j=1}^N \pi(a_j|s)\Big| +  \Big|\frac{\beta}{\lambda}\mathbb{E}_{s \sim \rho^{\pi}_\upsilon} \sum_{j=1}^N \pi(a_j|s)\ln (\pi(a_j|s))\Big|  \notag \\ & + \Big|\mathbb{E}_{s \sim \rho^{\pi}_\upsilon} \sum_{i=1}^N \sum_{j=1}^N \frac{\beta}{\lambda} \frac{\exp{(\frac{\lambda}{\beta}A^\pi(s,a_i) - \lambda D_{ij})} \cdot \pi(a_j|s)}{\sum_{k=1}^N \exp{(\frac{\lambda}{\beta}A^\pi(s,a_k) - \lambda D_{kj})}}\Big| \notag \\
    & +  \Big|\mathbb{E}_{s \sim \rho^{\pi}_\upsilon} \sum_{j=1}^N \pi(a_j|s)  \frac{\beta}{\lambda} \ln [\sum_{i=1}^N \exp{(\frac{\lambda}{\beta}A^\pi(s,a_i) - \lambda D_{ij})}]  \notag \\
    & - \mathbb{E}_{s \sim \rho^{\pi}_\upsilon} \sum_{j=1}^N \pi(a_j|s) \max_{i = 1 \dots N} (A^{\pi} (s,a_i) -  \beta D_{ij}) \Big| \\
    & \le 2\Big|\frac{\beta}{\lambda}\mathbb{E}_{s \sim \rho^{\pi}_\upsilon} \sum_{j=1}^N \pi(a_j|s)\Big| +  \Big|\frac{\beta}{\lambda}\mathbb{E}_{s \sim \rho^{\pi}_\upsilon} \sum_{j=1}^N \pi(a_j|s)\ln (\pi(a_j|s))\Big| \notag \\
    & + \Big|\mathbb{E}_{s \sim \rho^{\pi}_\upsilon} \sum_{j=1}^N \pi(a_j|s) \frac{\beta}{\lambda} \ln [\sum_{i=1}^N \exp{(\frac{\lambda}{\beta}A^\pi(s,a_i) - \lambda D_{ij})}]  \notag \\ 
    & - \mathbb{E}_{s \sim \rho^{\pi}_\upsilon} \sum_{j=1}^N \pi(a_j|s) \max_{i = 1 \dots N} (A^{\pi} (s,a_i) -  \beta D_{ij}) \Big| \label{abslimit}.
    \end{align}
\end{subequations}
In addition, 
\begin{subequations}
    \begin{align}
    &\Big|\mathbb{E}_{s \sim \rho^{\pi}_\upsilon} \sum_{j=1}^N \pi(a_j|s) \frac{\beta}{\lambda} \ln [\sum_{i=1}^N \exp{(\frac{\lambda}{\beta}A^\pi(s,a_i) - \lambda D_{ij})}] \notag \\ 
    & - \mathbb{E}_{s \sim \rho^{\pi}_\upsilon} \sum_{j=1}^N \pi(a_j|s) \max_{i = 1 \dots N} (A^{\pi} (s,a_i) -  \beta D_{ij}) \Big| \\
    = &  \Big|\mathbb{E}_{s \sim \rho^{\pi}_\upsilon} \sum_{j=1}^N \pi(a_j|s) \frac{\beta}{\lambda} \ln [\exp{(\frac{\lambda}{\beta}\max_{k = 1 \dots N} (A^{\pi} (s,a_k) -  \beta D_{kj}))} \sum_{i=1}^N \exp{(-\frac{\lambda}{\beta}\zcyan{\epsilon_s^\pi(\beta,i,j)})}] \notag \\&  - \mathbb{E}_{s \sim \rho^{\pi}_\upsilon} \sum_{j=1}^N \pi(a_j|s) \max_{i = 1 \dots N} (A^{\pi} (s,a_i) -  \beta D_{ij}) \Big| \\
    = & \Big|\mathbb{E}_{s \sim \rho^{\pi}_\upsilon} \sum_{j=1}^N \pi(a_j|s) \frac{\beta}{\lambda} \{ \ln [\exp{(\frac{\lambda}{\beta}\max_{k = 1 \dots N} (A^{\pi} (s,a_k) -  \beta D_{kj}))}] + \ln[\sum_{i=1}^N \exp{(-\frac{\lambda}{\beta}\zcyan{\epsilon_s^\pi(\beta,i,j)})}]\} \notag \\&  - \mathbb{E}_{s \sim \rho^{\pi}_\upsilon} \sum_{j=1}^N \pi(a_j|s) \max_{i = 1 \dots N} (A^{\pi} (s,a_i) -  \beta D_{ij}) \Big| \\
    = &  \Big|\mathbb{E}_{s \sim \rho^{\pi}_\upsilon} \sum_{j=1}^N \pi(a_j|s)   \frac{\beta}{\lambda} \ln[\sum_{i=1}^N \exp{(-\frac{\lambda}{\beta}\zcyan{\epsilon_s^\pi(\beta,i,j)})}] \Big|.
    \end{align}
\end{subequations}
Therefore, 
\begin{subequations}
\begin{align}
& \lim_{\lambda \xrightarrow{} \infty} \sup_{0 \le \beta \le \beta_{\text{UB}}} \Big| F_\lambda(\beta) - F(\beta) \Big|  \\
   & \leq \lim_{\lambda \xrightarrow{} \infty}  \frac{2\beta_{\text{UB}}}{\lambda}\Big|\mathbb{E}_{s \sim \rho^{\pi}_\upsilon} \sum_{j=1}^N \pi(a_j|s)\Big| + \lim_{\lambda \xrightarrow{} \infty} \frac{\beta_{\text{UB}}}{\lambda} \Big|\mathbb{E}_{s \sim \rho^{\pi}_\upsilon} \sum_{j=1}^N \pi(a_j|s)\ln (\pi(a_j|s))\Big|  \nonumber \\
   & + \lim_{\lambda \xrightarrow{} \infty} \sup_{0 \le \beta \le \beta_{\text{UB}}} \frac{\beta}{\lambda}\Big|\mathbb{E}_{s \sim \rho^{\pi}_\upsilon} \sum_{j=1}^N \pi(a_j|s)   \ln[\sum_{i=1}^N \exp{(-\frac{\lambda}{\beta}\zcyan{\epsilon_s^\pi(\beta,i,j)})}] \Big| \\
   &= \lim_{\lambda \xrightarrow{} \infty} \sup_{0 \le \beta \le \beta_{\text{UB}}} \frac{\beta}{\lambda}\Big|\mathbb{E}_{s \sim \rho^{\pi}_\upsilon} \sum_{j=1}^N \pi(a_j|s)   \ln[\sum_{i=1}^N \exp{(-\frac{\lambda}{\beta}\zcyan{\epsilon_s^\pi(\beta,i,j)})}] \Big|. \label{limsup}
\end{align}
\end{subequations}
In addition, $\forall \beta \in [0, \beta_{\text{UB}}]$ and $\forall \lambda$, $\zcyan{\epsilon_s^\pi(\beta,i,j)}$ is bounded since
\begin{align}
 \Big|\zcyan{\epsilon_s^\pi(\beta,i,j)}\Big| & = 
 \Big|\max_{k = 1 \dots N} (A^{\pi} (s,a_k) -  \beta D_{kj})  - [A^{\pi} (s,a_i) -  \beta D_{ij}]\Big|
 \\
%  & = 
%  \Big|(A^{\pi} (s,a_{\zcyan{k_s^\pi(\beta,j)}}) - A^{\pi} (s,a_i)) \zcyan{-} (\beta M_{\zcyan{k_s^\pi(\beta,j)}j} - \beta D_{ij}) \Big|\nonumber\\
%  & \leq \Big|A^{\pi} (s,a_{\zcyan{k_s^\pi(\beta,j)}}) - A^{\pi} (s,a_i)\Big| + \Big|\beta M_{\zcyan{k_s^\pi(\beta,j)}j} - \beta D_{ij} \Big|\nonumber\\
 & \leq 2\max_{s,a}A^{\pi}(s,a) + \beta_{\text{UB}} \max_{i,j}D_{ij} \nonumber \\ 
 & \zcyan{\leq 2 A^{\mbox {\tiny max}} + \beta_{\text{UB}} \max_{i,j}D_{ij}} < \infty \label{boundfore}.
\end{align}
Then, $\Big|\mathbb{E}_{s \sim \rho^{\pi}_\upsilon} \sum_{j=1}^N \pi(a_j|s)   \ln[\sum_{i=1}^N \exp{(-\frac{\lambda}{\beta}\zcyan{\epsilon_s^\pi(\beta,i,j)})}] \Big|$ is bounded. Therefore in (\ref{limsup}), the optimal $\beta$ can be achieved. Let $\beta^\lambda = \text{argmax}_{0 \le \beta \le \beta_{\text{UB}}} \frac{\beta}{\lambda}\Big|\mathbb{E}_{s \sim \rho^{\pi}_\upsilon} \sum_{j=1}^N \pi(a_j|s)   \ln[\sum_{i=1}^N \exp{(-\frac{\lambda}{\beta}\zcyan{\epsilon_s^\pi(\beta,i,j)})}] \Big|$, and then we have: 
\begin{subequations}
\begin{align}
    & \lim_{\lambda \xrightarrow{} \infty} \sup_{0 \le \beta \le \beta_{\text{UB}}} \frac{\beta}{\lambda}\Big|\mathbb{E}_{s \sim \rho^{\pi}_\upsilon} \sum_{j=1}^N \pi(a_j|s)   \ln[\sum_{i=1}^N \exp{(-\frac{\lambda}{\beta}\zcyan{\epsilon_s^\pi(\beta,i,j)})}] \Big| \\
    & = \lim_{\lambda \xrightarrow{} \infty} \frac{\beta^\lambda}{\lambda} \Big|\mathbb{E}_{s \sim \rho^{\pi}_\upsilon} \sum_{j=1}^N \pi(a_j|s)   \ln[\sum_{i=1}^N \exp{(-\frac{\lambda}{\beta^\lambda}\zcyan{\epsilon_s^\pi(\beta^\lambda,i,j)})}] \Big|.
\end{align}
\end{subequations}

Let $\mathcal{K}^\pi_s(\beta,j) = \text{argmax}_{k = 1 \dots N} A^{\pi} (s,a_k) -  \beta D_{kj}$. Define $\zcyan{\sigma_s(j)} = \min_{0\le \beta \le \beta_{\text{UB}}} \min_{i = 1 \dots N, i \notin \zcyan{\mathcal{K}^\pi_s(\beta,j)}} \zcyan{\epsilon_s^\pi(\beta,i,j)}$. Then since $\zcyan{\epsilon_s^\pi(\beta,i,j)}>0$ \zcyan{for $i \notin \mathcal{K}^\pi_s(\beta,j)$} based on its definition, we have $\zcyan{\sigma_s(j)} > 0$. On one hand, we have
\begin{subequations}
\begin{align}
    &\lim_{\lambda \xrightarrow{} \infty}\ln[\sum_{i=1}^N \exp{ (-\frac{\lambda}{\beta^\lambda}\zcyan{\epsilon_s^\pi(\beta^\lambda,i,j)})}] \\
    &= \lim_{\lambda \xrightarrow{} \infty}\ln[\sum_{i=1| i \notin \zcyan{\mathcal{K}^\pi_s(\beta_\lambda,j)}}^N \exp{ (-\frac{\lambda}{\beta^\lambda}\zcyan{\epsilon_s^\pi(\beta^\lambda,i,j)})} + \sum_{i=1| i \in \zcyan{\mathcal{K}^\pi_s(\beta_\lambda,j)}}^N \exp{ ( -\frac{\lambda}{\beta^\lambda}\zcyan{\epsilon_s^\pi(\beta^\lambda,i,j)})}] \\
    &\le \lim_{\lambda \xrightarrow{} \infty}\ln[\sum_{i=1| i \notin \zcyan{\mathcal{K}^\pi_s(\beta_\lambda,j)}}^N \exp{ ( -\frac{\lambda}{\beta_{\text{UB}}} \zcyan{\sigma_s(j)})} + \sum_{i=1| i \in \zcyan{\mathcal{K}^\pi_s(\beta_\lambda,j)}}^N \exp{ (0)}] \\
    & = \lim_{\lambda \xrightarrow{} \infty}\ln[\sum_{i=1| i \notin \zcyan{\mathcal{K}^\pi_s(\beta_\lambda,j)}}^N \exp{ ( -\frac{\lambda}{\beta_{\text{UB}}} \zcyan{\sigma_s(j)})} + |\zcyan{\mathcal{K}^\pi_s(\beta_\lambda,j)}| ] \\
    & = \lim_{\lambda \xrightarrow{} \infty} \ln[|\zcyan{\mathcal{K}^\pi_s(\beta_\lambda,j)}|]\label{upper_bound_ln}.
\end{align}
\end{subequations}
On the other hand, we have
\begin{subequations}
\begin{align}
    &\lim_{\lambda \xrightarrow{} \infty} \ln[\sum_{i=1}^N \exp{ (-\frac{\lambda}{\beta^\lambda}\zcyan{\epsilon_s^\pi(\beta^\lambda,i,j)})}] \\
    &= \lim_{\lambda \xrightarrow{} \infty}\ln[\sum_{i=1| i \notin \zcyan{\mathcal{K}^\pi_s(\beta_\lambda,j)}}^N \exp{ (-\frac{\lambda}{\beta^\lambda}\zcyan{\epsilon_s^\pi(\beta^\lambda,i,j)})} + \sum_{i=1| i \in \zcyan{\mathcal{K}^\pi_s(\beta_\lambda,j)}}^N \exp{ ( -\frac{\lambda}{\beta^\lambda}\zcyan{\epsilon_s^\pi(\beta^\lambda,i,j)})}] \\
    &\ge \lim_{\lambda \xrightarrow{} \infty}\ln[\sum_{i=1| i \in \zcyan{\mathcal{K}^\pi_s(\beta_\lambda,j)}}^N \exp{ ( -\frac{\lambda}{\beta^\lambda}\zcyan{\epsilon_s^\pi(\beta^\lambda,i,j)})}] \\
    &= \lim_{\lambda \xrightarrow{} \infty}\ln[\sum_{i=1| i \in \zcyan{\mathcal{K}^\pi_s(\beta_\lambda,j)}}^N \exp{ (0)}] \\
    & = \lim_{\lambda \xrightarrow{} \infty}\ln[ |\zcyan{\mathcal{K}^\pi_s(\beta_\lambda,j)}| ] \label{lower_bound_ln}.
\end{align}
\end{subequations}
Therefore, $\lim_{\lambda \xrightarrow{} \infty} \Big|\ln[\sum_{i=1}^N \exp{ (-\frac{\lambda}{\beta^\lambda}\zcyan{\epsilon_s^\pi(\beta^\lambda,i,j)})}]\Big| = \lim_{\lambda \xrightarrow{} \infty} \ln[|\zcyan{\mathcal{K}^\pi_s(\beta_\lambda,j)}|]$. Based on that, we have   
\begin{subequations}
\begin{align}
&\lim_{\lambda \xrightarrow{} \infty} \frac{\beta^\lambda}{\lambda} \Big|\mathbb{E}_{s \sim \rho^{\pi}_\upsilon} \sum_{j=1}^N \pi(a_j|s)   \ln[\sum_{i=1}^N \exp{(-\frac{\lambda}{\beta^\lambda}\zcyan{\epsilon_s^\pi(\beta^\lambda,i,j)})}] \Big|  \\
&\le \lim_{\lambda \xrightarrow{} \infty} \frac{\beta^\lambda}{\lambda} \Big|\sum_{j=1}^N \ln[\sum_{i=1}^N \exp{(-\frac{\lambda}{\beta^\lambda}\zcyan{\epsilon_s^\pi(\beta^\lambda,i,j)})}] \Big| \\
&\le \lim_{\lambda \xrightarrow{} \infty} \frac{\beta^\lambda}{\lambda} \sum_{j=1}^N \Big|\ln[\sum_{i=1}^N \exp{(-\frac{\lambda}{\beta^\lambda}\zcyan{\epsilon_s^\pi(\beta^\lambda,i,j)})}] \Big| \\
& = \lim_{\lambda \xrightarrow{} \infty} \frac{\beta^\lambda}{\lambda} \sum_{j=1}^N  \ln[|\zcyan{\mathcal{K}^\pi_s(\beta_\lambda,j)}|] \\
& \le \lim_{\lambda \xrightarrow{} \infty} \frac{\beta_{\text{UB}}}{\lambda} N \ln N = 0,
\end{align}
\end{subequations}
which means $\lim_{\lambda \xrightarrow{} \infty} \sup_{0 \le \beta \le \beta_{\text{UB}}} \Big| F_\lambda(\beta) - F(\beta) \Big| \le 0$. Furthermore, since $\lim_{\lambda \xrightarrow{} \infty} \sup_{0 \le \beta \le \beta_{\text{UB}}} |F_\lambda(\beta) - F(\beta)| \ge 0$ holds naturally, we have $\lim_{\lambda \xrightarrow{} \infty} \sup_{0 \le \beta \le \beta_{\text{UB}}} |F_\lambda(\beta) - F(\beta)| = 0$. Therefore, $F_\lambda(\beta)$ converges to $F(\beta)$ uniformly on $[0, \beta_{\text{UB}}]$, which also indicates $F_\lambda(\beta)$ epi-converges to $F(\beta)$ on $[0, \beta_{\text{UB}}]$ \citep{royset2016epiconvergence, rockafellar1998variationalanalysis}. By properties of epi-convergence, we have that $\begin{aligned} \lim_{\lambda \xrightarrow{} \infty} \text{argmin}_{0 \le \beta \le \beta_{\text{UB}}} F_{\lambda} (\beta) \subseteq \text{argmin}_{0 \le \beta \le \beta_{\text{UB}}} F(\beta) \end{aligned}$ \citep{rockafellar1998variationalanalysis}.
\end{proof}

\section{Proof of Lemma \ref{lemma_spo_converge_to_wpo}}
\label{appendix:spo_converge_to_wpo}
\lemmaspoconvergetowpo*

\begin{proof}[Proof of Lemma \ref{lemma_spo_converge_to_wpo}]
Let $\xi_s^k(i,j) = \frac{\lambda}{\beta_k} \{
\max_{l=1,\dots,N} (\hat{A}^{\pi_k} (s,a_{l}) -  \beta_k D_{lj}) - [\hat{A}^{\pi_k} (s,a_i) -  \beta_k D_{ij}]\}$. The SPO update with $\lambda \xrightarrow{} \infty$ equals to: 

\begin{subequations}
\begin{align}
    \pi_{k+1}(a_i|s) &= \lim_{\lambda \xrightarrow{} \infty} \mathbb{F}^{\textrm{SPO}}(\pi_k)=  \lim_{\lambda \xrightarrow{} \infty} \sum_{j=1}^N \frac{\exp{(\frac{\lambda}{\beta_k}\hat{A}^{\pi_k}(s,a_i) - \lambda D_{ij})}}{{\sum}_{l=1}^N \exp{(\frac{\lambda}{\beta_k}\hat{A}^{\pi_k}(s,a_l) - \lambda D_{lj})}} \pi_k(a_j|s) \\
    & = \lim_{\lambda \xrightarrow{} \infty} \sum_{j=1}^N  \frac{\exp{(\hat{A}^{\pi_k} (s,a_{\hat{k}^{\pi_k}_s(\beta_k, j)}) -  \beta_k D_{\hat{k}^{\pi_k}_s(\beta_k, j)j}) \cdot \exp{(- \xi_s^k(i,j))}}}{ \exp{(\hat{A}^{\pi_t} (s,a_{\hat{k}^{\pi_k}_s(\beta_k, j)}) -  \beta_k D_{\hat{k}^{\pi_k}_s(\beta_k, j)j}) \cdot {\sum}_{l=1}^N \exp(-\xi_s^k(l,j))}} \pi_k(a_j|s) \\
    & = \lim_{\lambda \xrightarrow{} \infty} \sum_{j=1}^N  \frac{\exp{(- \xi_s^k(i,j))}}{{\sum}_{l=1}^N \exp(-\xi_s^k(l,j))} \pi_k(a_j|s) \\
    & = \lim_{\lambda \xrightarrow{} \infty} \sum_{j=1}^N  \frac{\exp{(- \xi_s^k(i,j))}\cdot\pi_k(a_j|s)}{{\sum}_{l \in \hat{\mathcal{K}}^{\pi_k}_s(\beta_k, j)} \exp(-\xi_s^k(l,j)) + {\sum}_{l \notin \hat{\mathcal{K}}^{\pi_k}_s(\beta_k, j)} \exp(-\xi_s^k(l,j))}  \\
    & = \sum_{j=1}^N  \frac{\lim_{\lambda \xrightarrow{} \infty}\exp{(- \xi_s^k(i,j))}\cdot\pi_k(a_j|s)}{{\sum}_{l \in \hat{\mathcal{K}}^{\pi_k}_s(\beta_k, j)} \lim_{\lambda \xrightarrow{} \infty}\exp(-\xi_s^k(l,j)) + {\sum}_{l \notin \hat{\mathcal{K}}^{\pi_k}_s(\beta_k, j)} \lim_{\lambda \xrightarrow{} \infty}\exp(-\xi_s^k(l,j))} \\
    & = \sum_{j=1}^N  \frac{I_{\hat{\mathcal{K}}^{\pi_k}_s(\beta_k, j)}(i)}{|\hat{\mathcal{K}}^{\pi_k}_s(\beta_k, j)|}\pi_k(a_j|s), \label{spo_infty_lambda_proof1}
\end{align}
\label{spo_infty}
\end{subequations}
where $I$ denotes the indicator function; (\ref{spo_infty_lambda_proof1}) holds because as $\lambda \xrightarrow{} \infty$, $\xi_s^k(i,j) = \infty$ for $i \notin \hat{\mathcal{K}}^{\pi_k}_s(\beta_k, j)$ and $0$ otherwise, thus $\lim_{\lambda \xrightarrow{} \infty}\exp{(- \xi_s^k(i,j))} = 0$ for $i \notin \hat{\mathcal{K}}^{\pi_k}_s(\beta_k, j)$ and $1$ otherwise. 

Let $f_s^k(i,j) = \frac{1}{|\hat{\mathcal{K}}^{\pi_k}_s(\beta_k, j)|}$ if $i \in \hat{\mathcal{K}}^{\pi_k}_s(\beta_k, j)$, and  $f_s^k(i,j) = 0$ otherwise. Therefore, SPO update with $\lambda \xrightarrow{} \infty$ equals to the following WPO update, $\mathbb{F}^{\text{WPO}}(\pi_k) = \sum_{j=1}^N \pi_k(a_j|s)f_s^k(i,j)$.  
\end{proof}

\section{Proof of Theorem \ref{thm_monotonic_improvement_wpo_and_spo}}
\label{appendix:monotonic}
\thmmonotonicimprovementwpoandspo*

\begin{proof}[Proof of Theorem \ref{thm_monotonic_improvement_wpo_and_spo}]
\begin{subequations}
\begin{align}
    {J}(\pi_{k+1}) - {J}(\pi_k) &= \mathbb{E}_{s \sim \rho_\upsilon^{\pi_{k+1}}} \mathbb{E}_{a \sim \pi_{k+1}}  [A^{\pi_k} (s,a)] \label{monotonic_improvement_inexact_proof_1} \\
    & = \mathbb{E}_{s \sim \rho_\upsilon^{\pi_{k+1}}} \sum_{i=1}^N \pi_{k+1}(a_i|s) A^{\pi_k} (s,a_i) \label{monotonic_improvement_inexact_proof_2} \\
    & = \mathbb{E}_{s \sim \rho_\upsilon^{\pi_{k+1}}} \sum_{i=1}^N \sum_{j=1}^N \pi_k(a_j|s) f_s^k(i, j) A^{\pi_k} (s,a_i) \label{monotonic_improvement_inexact_proof_3} \\
    & = \mathbb{E}_{s \sim \rho_\upsilon^{\pi_{k+1}}} \sum_{j=1}^N  \pi_k(a_j|s) \sum_{i=1}^N  f_s^k(i, j) A^{\pi_k} (s,a_i)  \label{monotonic_improvement_inexact_proof_4} \\
    & = \mathbb{E}_{s \sim \rho_\upsilon^{\pi_{k+1}}} \sum_{j=1}^N \pi_{k}(a_j|s) \sum_{i \in \hat{\mathcal{K}}^{\pi_k}_s(\beta_k, j)} f_s^k(i, j) A^{\pi_k} (s,a_i)   \label{monotonic_improvement_inexact_proof_5} \\
    & \ge \mathbb{E}_{s \sim \rho_\upsilon^{\pi_{k+1}}} \sum_{j=1}^N \pi_{k}(a_j|s) \sum_{i \in \hat{\mathcal{K}}^{\pi_k}_s(\beta_k, j)} f_s^k(i, j) [A^{\pi_k} (s,a_j) + \beta_k D_{ij} - 2\epsilon] \label{monotonic_improvement_inexact_proof_6} \\
    & = \beta_k \mathbb{E}_{s \sim \rho_\upsilon^{\pi_{k+1}}} \sum_{j=1}^N \pi_{k}(a_j|s) \sum_{i \in \hat{\mathcal{K}}^{\pi_k}_s(\beta_k, j)} f_s^k(i, j) D_{ij} - \frac{2\epsilon}{1-\gamma}, \label{monotonic_improvement_inexact_proof_7}
\end{align}
\label{monotonic_improvement_inexact_proof}
\end{subequations}

where (\ref{monotonic_improvement_inexact_proof_1}) holds due to the performance difference lemma in \citet{kakade2002_approximatelyoptimal}; (\ref{monotonic_improvement_inexact_proof_6}) follows from the definition of $\hat{\mathcal{K}}^{\pi_k}_s(\beta_k, j)$ and the fact that $||\hat{A}^{\pi_k} - A^{\pi_k} ||_\infty \le \epsilon$, therefore for $i \in \hat{\mathcal{K}}^{\pi_k}_s(\beta_k, j)$, $[A^{\pi_k} (s, a_i) + \epsilon] - \beta_k D_{ij} \ge \hat{A}^{\pi_k} (s, a_i) - \beta_k D_{ij} \ge \hat{A}^{\pi_k} (s, a_j) - \beta_k D_{jj} = \hat{A}^{\pi_k} (s, a_j) \ge A^{\pi_k} (s, a_j) - \epsilon$; (\ref{monotonic_improvement_inexact_proof_7}) holds since $\mathbb{E}_{a \sim \pi}[A^\pi(s,a)] = 0$. 
\end{proof}

\section{Proof of Theorem \ref{thm: GlobalConvergence}}
\label{appendix:globalconvergenceproof}
\thmglobalconvergence*

\begin{proof}[Proof of Theorem \ref{thm: GlobalConvergence}] Our proof is inspired by the 
work \cite{bhandari2021linear}. 

We use the shorthand $\pi_s$ for the probability distribution $\pi(\cdot\mid s)$ on the actions and denote the probability distribution on the action space $\mathcal{A}$ as $\Delta$. To save notations, we rewrite $\pi_{k+1},\pi_k$ and $\beta_k$ as $\pi^+,\pi$ and $\beta$ respectively. We use $d$ for either $d_{\text{W}}$ or $d_{\text{S}}$ in the following derivation. Note $d \leq \|D\|_{\infty} =\,:D $ for both cases\footnote{For Sinkhorn divergence, note that the entropy function is always nonnegative.}, and $d_{\text{S}}\geq -2\frac{\log N}{\lambda}$.
\footnote{This lower bound is obtained via 
$d_{\text{S}}(\pi',\pi|\lambda) \geq\min_{Q\geq 0, \sum_{i,j}Q_{ij}=1}\left\{\langle Q,D\rangle -\frac{1}{\lambda}h(Q)\right\}\overset{(a)}{=}
\langle Q,D\rangle -\frac{1}{\lambda}h(Q)|_{Q_{ij} =\frac{\exp(-\lambda D_{ij})}{\sum_{i,j} \exp(-\lambda D_{ij})}} = -\frac{1}{\lambda} \log \left(\sum_{i,j} \exp(-\lambda D_{ij})\right) \overset{(b)}{\geq} -\frac{2\log N}{\lambda}$. Here in the step $(a)$, we use the Lagrangian multiplier method to derive the optimal $Q_{ij} =\frac{\exp(-\lambda D_{ij})}{\sum_{i,j} \exp(-\lambda D_{ij})}$. In the step $(b)$, we use the fact that $\log (\sum_{i=1}^n \exp(x_i))\leq \max\{x_1,\dots,x_n\}+\log n$ for any $x_1,\dots, x_n\in \mathbb{R}$ and $D_{ii} = 0$ for any $i$.}

Since a policy $\pi$ is indeed just a member of $\prod_{i=1}^{|S|} \Delta$, we find that the problem (\ref{odrpo_WA}) can be split into $|\mathcal{S}|$ many optimization problems. 
For each $s\in \mathcal{S}$, we need to solve
\begin{equation}
\begin{split}
& \max_{\pi'_s \in \Delta} \hspace{3mm} \rho^{\pi}(s)\mathbb{E}_{a \sim \pi'(\cdot|s)}[A^{\pi} (s,a)] - \beta 
\rho^{\pi}(s)d (\pi'_s, \pi_s) .
\end{split}
\label{odrpo_WAsplitState}
\end{equation}
Denote the quality function of $\pi$ as $Q^{\pi}(s,a) = \mathbb{E}[R_t|s_t = s, a_t = a; \pi]$, and the value function of $\pi$ as $V^{\pi}(s) = \mathbb{E}[R_t|s_t = s; \pi]$, we find that 
$A^{\pi}(s,a) =Q^{\pi}(s,a) - V^{\pi}(s)$. Since the second term is only a function of the current policy $\pi$ and the state $s$, we find that Problem (\ref{odrpo_WAsplitState}) is further equivalent to (in the sense of the same solution set):
\begin{equation}
\begin{split}
& \max_{\pi'_s\in \Delta} \hspace{3mm} \mathbb{E}_{a\sim \pi'_s}[Q^{\pi}(s,a)] - \beta d (\pi'_s, \pi_s)   .
\end{split}
\label{odrpo_WAsplitState2}
\end{equation}
Here we use $\rho_0(s)>0$ for all $s$.
Let $\bar{\pi}$ be a solution of the policy iteration: 
\begin{equation}
    \bar{\pi}_s \in \arg\max_{\pi'_s}\; \mathbb{E}_{a\sim \pi'_s}[Q^{\pi}(s,a)] .\label{eq: policyiteration}
\end{equation}
Also define the bellman operator $T:\mathbb{R}^{|\mathcal{S}|} \rightarrow \mathbb{R}^{|\mathcal{S}|}$ and the operator 
$T^{\pi'}:\mathbb{R}^{|\mathcal{S}|} \rightarrow \mathbb{R}^{|\mathcal{S}|}$: for any $V\in \mathbb{R}^{|\mathcal{S}|} $,
\begin{align}
    (TV)_s = \max_{a\in \mathcal{A}}r(s,a) + \gamma \mathbb{E}_{s'\sim P(\cdot|s,a)}[V(s')],\\
     (T^{\pi'}V)_s = \mathbb{E}_{a\sim \pi'_s}[r(s,a) + \gamma \mathbb{E}_{s'\sim P(\cdot|s,a)} V(s')].
\end{align}
Using the relation between the quality function and the value function, $Q^\pi (s,a) = r(s,a) + \mathbb{E}_{s'\sim P(\cdot|s,a)}[V^\pi(s')]$, we can rewrite the above equations in terms of the quality function for $V=V^{\pi}$:
\begin{align}
    (TV^{\pi})_s & =  \max_{a\in \mathcal{A}}r(s,a) + \gamma \mathbb{E}_{s'\sim P(\cdot|s,a)}[V^{\pi}(s')] = \max_{a\in \mathcal{A}} Q(s,a) = T^{\bar{\pi}}V^{\pi}, \label{eq: bellmaniteration}\\ 
    (T^{\pi'}V^{\pi})_s & = \mathbb{E}_{a\sim \pi'_s}[Q^{\pi}(s,a)]. \label{eq: bellmanparticularpi}
\end{align}

Let us consider $d= d_{\text{W}}$ first. Using the optimality of $\pi^+$ for the problem (\ref{odrpo_WAsplitState}), we know that 
\begin{equation}
\begin{aligned}\label{eq: wacomparedtopolicyiteration}
    & \mathbb{E}_{a\in \pi^+_s}[Q^{\pi}(s,a)] - \beta d (\pi^+_s, \pi_s) \geq \mathbb{E}_{a\in \bar{\pi}_s}[Q^{\pi}(s,a)]  - \beta d (\bar{\pi}_s, \pi_s)\\
    \implies & 
    \mathbb{E}_{a\in \pi^+_s}[Q^{\pi}(s,a)] \geq \mathbb{E}_{a\in \bar{\pi}_s}[Q^{\pi}(s,a)] - \beta D. 
\end{aligned} 
\end{equation}
and 
\begin{equation}
\begin{aligned}\label{eq: wacomparedtooldpolicy}
    & \mathbb{E}_{a\in \pi^+_s}[Q^{\pi}(s,a)] - \beta d (\pi^+_s, \pi_s) \geq \mathbb{E}_{a\in \pi_s}[Q^{\pi}(s,a)]  - \beta d (\pi_s, \pi_s)\\
    \implies & 
    \mathbb{E}_{a\in \pi^+_s}[Q^{\pi}(s,a)] \geq \mathbb{E}_{a\in \pi_s}[Q^{\pi}(s,a)] = V^\pi (s).
\end{aligned} 
\end{equation}
Using the notation in (\ref{eq: bellmaniteration}) and (\ref{eq: bellmanparticularpi}), (\ref{eq: wacomparedtopolicyiteration}) and (\ref{eq: wacomparedtooldpolicy}) become
\begin{align} 
    T^{\pi^+} V^{\pi} \geq T V^{\pi} - \beta D \mathbf{1}_{|\mathcal{S}|}, \label{eq: piplusapprox}\\
    T^{\pi^+}V^{\pi} \geq V^{\pi} \label{eq: piplus}.
\end{align}
Here $\mathbf{1}_{|\mathcal{S}|}$ is a vector of all one entries and the inequality $\geq$ means entrywisely larger than or equal to. 
By iteratively applying $T^{\pi^+}$ to (\ref{eq: piplus}) and use the fact that $T^{\pi^+}$ is a monotone and contraction map with $V^{\pi^+}$ as the unique fixed point, we have  
\begin{equation}
    V^{\pi^+} \geq \dots \geq (T^{\pi^+})^2 V^{\pi}\geq T^{\pi^+} V^{\pi}\geq V^{\pi}. \label{eq: Vpiplus}
\end{equation}
Hence we have 
\begin{equation}\label{eq: iterative bound}
   0\overset{(a)}{\leq}  V^\star - V^{\pi^+} \overset{(b)}{\leq} V^\star - T^{\pi^+}V^{\pi} \overset{(c)}{\leq}
   V^\star - TV^{\pi} + \beta D \mathbf{1}_{|\mathcal{S}|}.
\end{equation}
Here the inequality $(a)$ is due to the optimality of $V^\star$. The inequality $(b)$ is due to (\ref{eq: Vpiplus}), and the inequality $(c)$ is due to (\ref{eq: piplusapprox}). Now using the fact $V^\star$ is the unique fixed point of $T$, and 
$T$ is a monotone and contraction map, we have from (\ref{eq: iterative bound}) that 
\begin{equation}\label{eq: Wfinal}
    \|V^\star - V^{\pi^+}\|_{\infty} \leq 
    \|TV^\star - TV^{\pi}\|_{\infty} + \beta D \leq \gamma   \|V^\star - V^{\pi}\|_{\infty} +\beta D.
\end{equation}

Next consider $d= d_{\text{S}}$. The optimality of $\pi^+$ reveals that for $\tilde{\pi} = \bar{\pi}$ or $\pi$:
\begin{equation}
\begin{aligned}\label{eq: Scomparedtopolicyiteration}
    & \mathbb{E}_{a\in \pi^+_s}[Q^{\pi}(s,a)] - \beta d (\pi^+_s, \pi_s) \geq \mathbb{E}_{a\in \tilde{\pi}_s}[Q^{\pi}(s,a)]  - \beta d (\tilde{\pi}_s, \pi_s)\\
    \implies & 
    \mathbb{E}_{a\in \pi^+_s}[Q^{\pi}(s,a)] \geq \mathbb{E}_{a\in \tilde{\pi}_s}[Q^{\pi}(s,a)] - \beta (D+2\frac{\log N}{\lambda}). 
\end{aligned} 
\end{equation}
Thus we have the following 
\begin{align} 
    T^{\pi^+} V^{\pi} & \geq T V^{\pi} - \beta (D +\frac{2\log N}{\lambda})\mathbf{1}_{|\mathcal{S}|}, \label{eq: Spiplusapprox}\\
    T^{\pi^+}V^{\pi} & \geq V^{\pi} -\beta (D +\frac{2\log N}{\lambda})\mathbf{1}_{|\mathcal{S}|} \label{eq: Spiplus}.
\end{align}
By iteratively applying $T^{\pi^+}$ to (\ref{eq: Spiplus}) and use the fact that $T^{\pi^+}$ is a monotone and contraction map with $V^{\pi^+}$ as the unique fixed point, we have  
\begin{equation}
    V^{\pi^+}\geq V^{\pi}-\frac{\beta}{1-\gamma} (D +2\frac{\log N}{\lambda})\mathbf{1}_{|\mathcal{S}|}. \label{eq: SVpiplus}
\end{equation}
Hence we have 
\begin{equation}\label{eq: Siterative bound}
\begin{aligned} 
   0\overset{(a)}{\leq}  V^\star - V^{\pi^+} \overset{(b)}{\leq} & V^\star - T^{\pi^+}V^{\pi} +\frac{\beta}{1-\gamma} (D +2\frac{\log N}{\lambda})\mathbf{1}_{|\mathcal{S}|}\\ \overset{(c)}{\leq} &
   V^\star - TV^{\pi} + 2\frac{\beta}{1-\gamma} (D +2\frac{\log N}{\lambda}) \mathbf{1}_{|\mathcal{S}|}.
 \end{aligned} 
\end{equation}
Here the inequality $(a)$ is due to the optimality of $V^\star$. The inequality $(b)$ is due to (\ref{eq: SVpiplus}), and the inequality $(c)$ is due to (\ref{eq: Spiplusapprox}). A similar derivation as (\ref{eq: Wfinal}) shows the inequality in the theorem. 
Hence the theorem is established. 
\end{proof}

\section{Computational Complexity of the Algorithm \ref{odrpo_algorithm}}\label{appendix:complexity}
Our overall algorithm applies a general actor-critic framework: the actor follows the proposed WPO or SPO update while the critic follows TD methods. The computational complexity depends on (i) the per-iteration computation cost of the policy and critic update and (ii) the iteration complexity of the actor-critic method. Here we mainly discuss the per-iteration computation cost of the policy update, as studies on the iteration complexity of actor-critic framework for constrained policy optimization are limited. 

The computation cost of WPO and SPO updates at each iteration depends on the selection of $\beta_k$. If $\beta_k$ is chosen time dependently, the computation cost of WPO/SPO policy update is $O(n_a^2 n_s)$, where $n_a$ and $n_s$ are the number of actions and states to perform policy update. If we set $\beta_k$ as the dual optimizer, there will be additional cost to run gradient descent to solve the one-dimensional dual formulation. As discussed in our experiments, we can set $\beta_k$ to be the dual optimizer only in the first a few iterations and use a decaying afterward. Therefore, the average computational complexity of a policy update step can be $O(n_a^2 n_s)$. 

\section{Difference between SPO/WPO and Other Exponential Style Updates}

\hlight{Sinkhorn divergence smooths the original Wasserstein by adding an entropy term, which causes the SPO update to contain exponential components similar to standard exponential style updates such as NPG }\citep{kakade2001natural, peng2019advantage}\hlight{. Thus, SPO can be viewed as a smoother version of WPO update. Nonetheless, it's important to note that SPO/WPO updates differ fundamentally from standard exponential style updates that are based solely on entropy or KL divergence. In both SPO and WPO, the probability mass at action $a$ is redistributed to neighboring actions with high value (i.e., those $a'$ with high $A^\pi(s,a') - \beta d(a',a)$). In contrast,  in these standard exponential style updates, probability mass at action $a$ is reweighted according to its exponential advantage or Q value.}

\section{Exploration Properties of WPO/SPO}

\hlight{Compared to the Wasserstein metric, the KL divergence between policies is often larger, especially when considering the policy shifts of closely related actions, as shown in Figure {\ref{fig:wd_geometric_advantage}}. In practice, when employing the same trust region size $\delta$, Wasserstein metric allows for more admissible policies within the trust region compared to KL, thereby leading to better exploration. This advantage is demonstrated in our motivating example in Figure {\ref{fig:policy_update_example_kl_wass}}.}

\hlight{Furthermore, Sinkhorn divergence has even more exploration advantages than using Wasserstein. As Sinkhorn smooths the original Wasserstein with an entropy term, it includes additional smoother (more uniform) policies in the trust region, leading to even faster exploration.}

\hlight{Our numerical results in Section {\ref{section_experiments}} also support that WPO/SPO explores better than KL; and SPO achieves faster exploration than WPO.}   

\section{Policy Parametrization, Prior Work on Nonparametric Policy}

\hlight{As noted in }\citep{tessler2019_dpo}\hlight{, the suboptimality of policy gradient is not due to parametrization (e.g., neural network), but is a result of the parametric distribution assumption imposed on policy, which constrains policies to a predefined set. In our work, we strive to avoid suboptimality by circumventing the parametric distribution assumption imposed on policy, while still allowing for parametrization of policy in our empirical studies. }

\hlight{Previous research, such as }\citep{abdolmaleki2018maximum, peng2019advantage}\hlight{, has investigated theoretical policy update rules based on KL divergence without making explicit parametric assumptions about the policy being used. However, to our best knowledge, no prior work has explored theoretical policy update rules based on Wasserstein metric or Sinkhorn divergence.}

\section{T-tests to Compare the Performance of WPO, SPO with BGPG and WNPG}

We conduct independent two-sample one-tailed t-tests \citep{student1908ttest} to compare the mean performance of our proposed methods (WPO and SPO) with two other Wasserstein-based policy optimization approaches: BGPG \citep{pacchiano2019_bgpg} and WNPG \citep{moskovitz2021wnpg}. Specifically, we formulate four alternative hypotheses for each task: $J_{\text{WPO}} > J_{\text{BGPG}}$, $J_{\text{WPO}} > J_{\text{WNPG}}$, $J_{\text{SPO}} > J_{\text{BGPG}}$, and $J_{\text{SPO}} > J_{\text{WNPG}}$.

MuJuCo continuous control tasks are considered for the t-tests, with a sample size of $10$ for each algorithm. All t-tests are conducted at a confidence level of $90\%$. The results of the t-tests are presented in Table \ref{ttest_performance_table}, where a checkmark ($\checkmark$) indicates that the alternative hypothesis is supported with $90\%$ confidence, and a dash ($-$) indicates a failure to support the alternative hypothesis.

Based on the results presented in Table \ref{ttest_performance_table}, we can conclude the following:

\begin{itemize}
    \item The mean performance of WPO is higher than BGPG with $90\%$ confidence for all tasks.
    \item The mean performance of WPO is higher than WNPG with $90\%$ confidence for all tasks.
    \item The mean performance of SPO is higher than BGPG with $90\%$ confidence for all tasks except Ant-v2.
    \item The mean performance of SPO is higher than WNPG with $90\%$ confidence for all tasks except HalfCheetah-v2.
\end{itemize}

We note that though SPO's performance is not statistically significantly higher than BGPG or WNPG in Ant-v2 and HalfCheetah-v2 tasks, SPO demonstrates a faster convergence speed than WNPG and BGPG in these two tasks. 

\begin{table}[H]
\caption{T-tests results on the performance of WPO, SPO, BGPG and WNPG}
\renewcommand{\arraystretch}{1.5}
\centering
\begin{adjustbox}{width=0.8\linewidth,center}
\begin{tabular}{lllll}
 \hline\hline
Environment  & $J_{\text{WPO}} > J_{\text{BGPG}}$ & $J_{\text{WPO}} > J_{\text{WNPG}}$ & $J_{\text{SPO}} > J_{\text{BGPG}}$ & $J_{\text{SPO}} > J_{\text{WNPG}}$  \\
\hline
HalfCheetah-v2 & $\checkmark$ & $\checkmark$ & $\checkmark$ & $-$ \\
\hline
Hopper-v2 & $\checkmark$ & $\checkmark$ & $\checkmark$ & $\checkmark$ \\
\hline
Walker2d-v2 & $\checkmark$ & $\checkmark$ & $\checkmark$ & $\checkmark$ \\
\hline
Ant-v2 & $\checkmark$ & $\checkmark$ & $-$ & $\checkmark$ \\
\hline
Humanoid-v2 & $\checkmark$ & $\checkmark$ & $\checkmark$ & $\checkmark$ \\
\hline
\end{tabular}
\label{ttest_performance_table}
\end{adjustbox}
\end{table}

\end{document}